\documentclass{article}

\usepackage{microtype}
\usepackage{booktabs} 

\usepackage{hyperref}

\usepackage[accepted]{icml2024}

\usepackage{amsmath}
\usepackage{amssymb}
\usepackage{mathtools}
\usepackage{amsthm}

\usepackage[capitalize,noabbrev]{cleveref}

\usepackage{algorithm}
\usepackage{algorithmic}

\usepackage{amsmath,amsfonts,bm}

\renewcommand{\Re}{{\mathbb R}}

\def\eqref#1{equation~\ref{#1}}

\def\1{\bm{1}}

\DeclareMathAlphabet{\mathsfit}{\encodingdefault}{\sfdefault}{m}{sl}
\SetMathAlphabet{\mathsfit}{bold}{\encodingdefault}{\sfdefault}{bx}{n}

\newcommand{\R}{\mathbb{R}}

\DeclareMathOperator*{\argmin}{arg\,min}

\usepackage{url}
\usepackage{romannum}
\newenvironment{sproof}{%
  \proof}{\endproof}
\usepackage{wrapfig}
\usepackage{amsthm,thmtools}
\usepackage{tikz}
\usetikzlibrary{shapes,arrows,positioning,calc}
\usetikzlibrary{shapes.geometric,arrows}
\usepackage{caption}
\usepackage{subcaption}

\theoremstyle{plain}
\newtheorem{theorem}{Theorem}[section]
\newtheorem{proposition}[theorem]{Proposition}

\theoremstyle{definition}

\theoremstyle{remark}
\newtheorem{remark}[theorem]{Remark}

\usepackage[textsize=tiny]{todonotes}

\declaretheorem[sibling=theorem]{example}

\newcommand{\SIR}{{(\text{SIR})}}

\tikzstyle{input} = [coordinate]
\tikzstyle{output} = [coordinate]
\tikzstyle{sum} = [draw, circle]

\tikzstyle{startstop} = [rectangle, rounded corners, text centered, draw=black, fill=blue!10]
\tikzstyle{process} = [rectangle, minimum width=2cm, minimum height=1cm, text centered, draw=black, fill=orange!10]
\tikzstyle{decision} = [diamond, aspect=3, minimum width=3cm, minimum height=1cm, text centered, draw=black, fill=red!10]
\tikzstyle{compute} = [rectangle, minimum width=2cm, minimum height=1cm, text centered, draw=black, fill=green!10]
\tikzstyle{estimate} = [rectangle, rounded corners, minimum width=2cm, minimum height=1cm, text centered, draw=black, fill=yellow!10]
\tikzstyle{arrow} = [thick,->,>=stealth]

\begin{document}

\twocolumn[
 \icmltitle{Nonlinear Filtering with Brenier Optimal Transport Maps}



\begin{icmlauthorlist}
\icmlauthor{Mohammad Al-Jarrah}{yyy}
\icmlauthor{Niyizhen Jin}{xxx}
\icmlauthor{Bamdad Hosseini}{xxx}
\icmlauthor{Amirhossein Taghvaei}{yyy}
\end{icmlauthorlist}

\icmlaffiliation{yyy}{Department of Aeronautics \& Astronautics, University of Washington, Seattle.}
\icmlaffiliation{xxx}{Department of Applied Mathematics, University of Washington, Seattle.}

\icmlcorrespondingauthor{Mohammad Al-Jarrah}{mohd9485@uw.edu}
\icmlcorrespondingauthor{Niyizhen Jin}{njin2@uw.edu}


\icmlkeywords{Machine Learning, ICML}

\vskip 0.3in
]



\printAffiliationsAndNotice{}  

\begin{abstract}
This paper is concerned with the problem of nonlinear filtering, i.e., computing the conditional distribution of the state of a stochastic dynamical system given a history of noisy partial observations. Conventional sequential importance resampling (SIR) particle filters suffer from fundamental limitations, in scenarios involving degenerate likelihoods or high-dimensional states, due to the weight degeneracy issue. In this paper, we explore an alternative method, which is based on estimating the Brenier optimal transport (OT) map from the current prior distribution of the state to the posterior distribution at the next time step. Unlike SIR particle filters, the OT formulation does not require the analytical form of the likelihood. Moreover, it allows us to harness the approximation power of neural networks to model complex and multi-modal distributions and employ stochastic optimization algorithms to enhance scalability. Extensive numerical experiments are presented that compare the OT method to the SIR particle filter and the ensemble Kalman filter, evaluating the performance in terms of sample efficiency, high-dimensional scalability, and the ability to capture complex and multi-modal distributions.
\end{abstract}

\section{Introduction}
    This paper is concerned with numerical methods for the solution of nonlinear filtering problems in discrete-time 
with a particular focus towards the settings where the 
distribution of interest is highly non-Gaussian and multi-modal.
To formulate the filtering problem consider
two stochastic processes: (i) a hidden Markov process, denoted by $\{X_t\}_{t=0}^\infty$, that represents the state of a dynamical system; (ii) an observed random process, denoted by $\{Y_t\}_{t=1}^\infty$, that represents sensory data. 
We assume the state and observation processes follow the probabilistic relationship
\begin{subequations}
    \begin{align}\label{eq:model-dyn}
        X_{t} &\sim  a(\cdot \mid X_{t-1}),\quad X_0 \sim \pi_0,\\\label{eq:model-obs}
        Y_t &\sim h(\cdot \mid X_t),
    \end{align}
where 
$a(\cdot \mid \cdot)$ and $ h(\cdot \mid \cdot)$ denote the conditional probability kernels of $X_{t+1}$ and $Y_t$, given $X_t$, respectively, and $\pi_0$ denotes the probability distribution of the initial state $X_0$. The filtering problem is to compute the conditional probability distribution of the hidden state $X_t$, given the history of observations $\mathcal Y_t=\{Y_1,Y_2,\ldots,Y_t\}$, denoted by
\begin{equation}\label{eq:posterior}
    \pi_t(\cdot) := \mathbb P(X_t \in \cdot \mid \mathcal Y_t),\quad \text{for}\quad t=1,2,\ldots
\end{equation}
    \end{subequations}
The conditional distribution $\pi_t$ is also referred to as the {\it posterior} or the {\it belief}.

In a general nonlinear and non-Gaussian setup, the posterior does not admit an explicit analytical solution. Therefore, it is necessary to design numerical methods to approximate it. Standard sequential importance re-sampling (SIR) particle filters (PF) approximate the posterior $\pi_t$ with a weighted empirical distribution of particles $\sum_{i=1}^N w^i_t\delta_{X^i_t}$~\cite{gordon1993novel,doucet09} with the particle locations and weights 
updated according to 
\begin{subequations}
    \begin{align}
        \text{(propagation)}\quad 
        X^i_t &\overset{\text{i.i.d}}{\sim} \sum_{j=1}^N w^j_{t-1} a(\cdot \mid X^j_{t-1}), \\
        \text{(conditioning)}\quad 
        w^i_t &= \frac{h(Y_t \mid X^i_{t})}{\sum_{j=1}^N h(Y_t \mid X^j_{t})}. 
    \end{align}
\end{subequations}
The propagation step involves simulation according to the dynamic model~(\ref{eq:model-dyn}), followed by a resampling procedure\footnote{In practice, resampling  is usually done only when the variance of weights becomes large.}. The conditioning step implements the Bayes's rule according to the observation model~(\ref{eq:model-obs}).

SIR PF provides an exact solution to the filtering problem in the limit  $N\to \infty$ 
~\cite{del2001stability}. 
However, it suffers from the {\it curse of dimensionality (COD)}. Specifically, as the dimension of the state and observation  grows, the likelihood becomes degenerate and concentrated on a small support, causing the majority of weights to become zero. This issue is known as {\it particle degeneracy} and can only be avoided when the number of particles grows exponentially with the dimension~\cite{bickel2008sharp,bengtsson08,rebeschini2015can}; This is an active research area; see \cite{van2019particle} and references therein. 

This curse of dimensionality has motivated the development of alternative algorithms that replace the conditioning step of SIR  with a coupling or transport step; see~\cite{daum10,crisan10,reich11,yang2011mean,el2012bayesian,reich2013nonparametric,de2015stochastic,yang2016,marzouk2016introduction,mesa2019distributed,reich2019data,pathiraja2021mckean,taghvaei2021optimal,calvello2022ensemble} as well as~\cite{spantini2022coupling} and \cite{taghvaei2023survey} for a recent survey of these topics. A more detailed comparison is provided in Sec.~\ref{sec:comparison}. Here the posterior is approximated with a uniformly weighted distribution of particles $\frac{1}{N}\sum_{i=1}^N \delta_{X^i_t}$, with a general particle update law of the form  
\begin{subequations}\label{eq:CIPS}
    \begin{align}\label{eq:Xi-propagation}
         \text{(propagation)}\quad   X^i_{t|t-1} &\sim a(\cdot \mid X^i_{t-1}),\\
         \text{(conditioning)}\quad   X^i_{t} &= T_t(X^i_{t\mid t-1}, Y_t).\label{eq:Xi-conditioning}
    \end{align}
    \end{subequations}
Here, $T_t(\cdot, Y_t)$ represents a (possibly stochastic) transport map from the (prior) distribution $\mathbb P(X_t \in \cdot \mid \mathcal Y_{t-1})$ to the posterior distribution $\mathbb P(X_t \in \cdot \mid \mathcal Y_{t})$.

This paper is concerned with the development of algorithms for approximating the above transport map.
In particular, we utilize the recently introduced optimal transport (OT) approach in~\cite{taghvaei2022optimal,al2023optimal} to formulate the problem of finding the map $T_t$ as a max-min stochastic optimization problem (see equations~(\ref{eq:empirical-optimization}) and~(\ref{eq:emperical_loss})) that only involves the particles $X^i_{t|t-1}$ and the associated observation $Y^i_{t\mid t-1} \sim h(\cdot \mid X^i_{t \mid t-1})$, for $i=1,\ldots,N$. The map $T$ is represented by residual neural networks, as in Fig.~\ref{tikz:static_struc}, and 
trained with stochastic optimization algorithms (see Algorithm~\ref{alg:otpf} for details). 
The OT approach is based on the combination  of
block-triangular transport, in the context of conditional simulation~\cite{kovachki2020conditional,ray2022efficacy,shi2022conditional,siahkoohi2021preconditioned}, with min-max 
formulations used to estimate OT maps as in~\cite{makkuva2020optimal,fan2020scalable,rout2022generative}. The OT methodology has two distinct features: (i) it can be implemented in a completely likelihood-free/simulation-based setting, because it only requires samples from the likelihood model, not its analytical form; (ii) and it allows for the application of neural networks to represent transport maps, that enables the capture of  non-Gaussian and multi-modal distributions, as well as stochastic optimization algorithms that  are scalable with the number of the particles. An illustrative example comparing the OT method with SIR particle filter is presented in Fig.~\ref{fig:squared_high_SNR_likelihood}. 

\vspace{-1ex}
\begin{figure}[htp]
    \centering
    \includegraphics[width=\columnwidth,trim={100 25 95 45},clip]{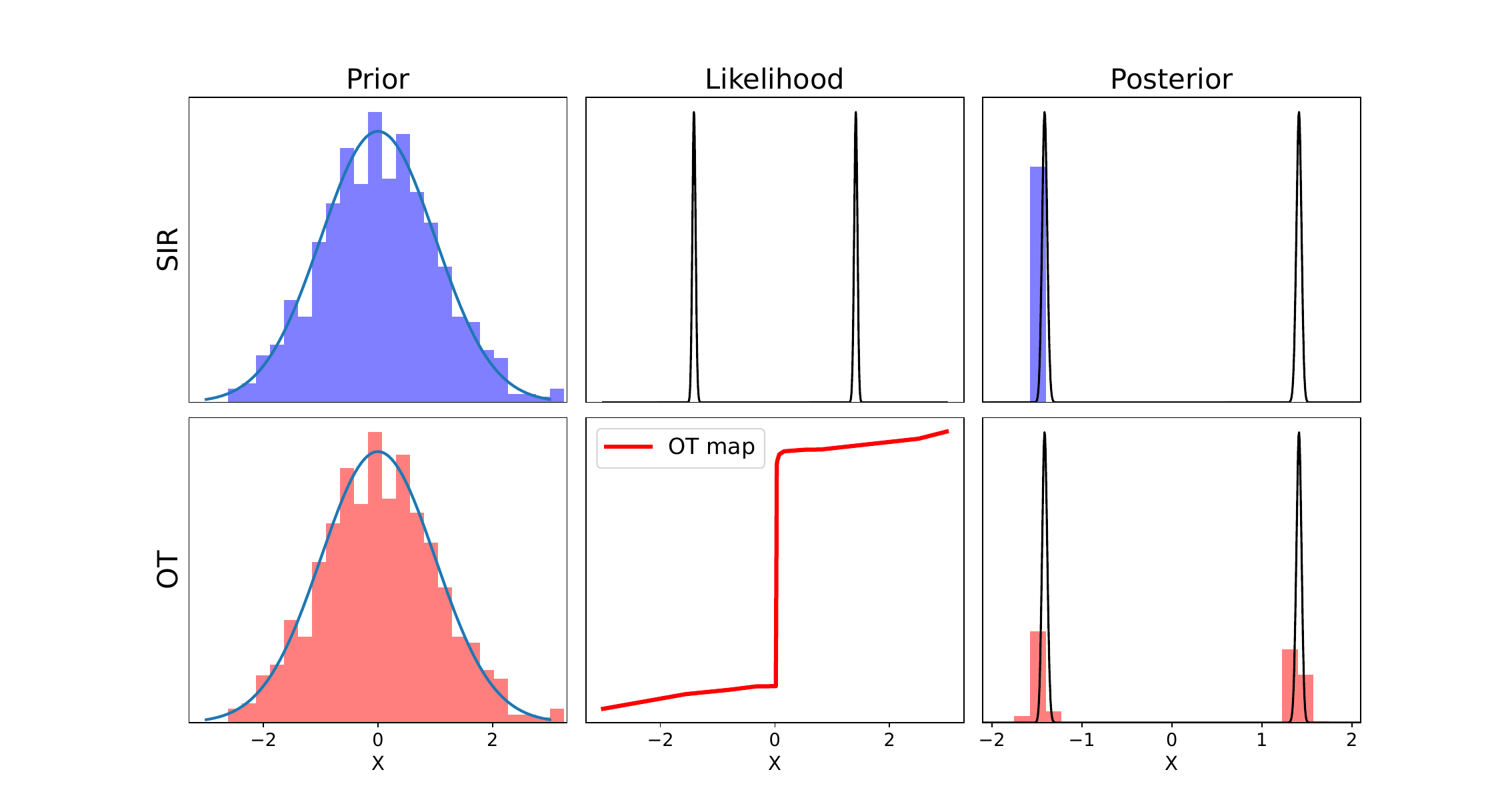}
    \caption{
A comparison between the OT method and SIR for 
the static model of Sec.~\ref{sec:Static_Example}. 
The OT approach  captures the bimodal posterior 
by pushing the prior through an OT map computed by solving equation~(\ref{eq:empirical-optimization}). On the other hand, SIR 
only captures one mode 
due to the degeneracy  of the likelihood leading to only a few 
 weights that are order  $1$. 
}
\label{fig:squared_high_SNR_likelihood}
\end{figure}

{The contributions of the paper are twofold: (1) 
We present a theoretical analysis of the OT method 
focusing on consistency and stability results; these are outlined 
in Prop~\ref{prop:consistency} and \ref{prop:map estimation error bound}. We also provide a preliminary error analysis in Prop.~\ref{prop:error-analysis} that compares the OT  approach to  SIR for a single conditioning step of the filter; (2) 
We present an extensive suite of   numerical experiments in Sec.~\ref{sec:numerics},
containing toy examples, benchmark Lorenz models, and a high-dimensional filtering problem  where image in-painting  is performed in a dynamic manner. We evaluate the OT method, in comparison to the SIR and ensemble Kalman filter (EnKF) algorithms,  in terms of sample efficiency and
high-dimensional scalability in conjunction with the ability to capture complex and multi-modal distributions.
}


\section{OT methodology} 

\subsection{Exact filter}
Let us
 begin by recalling the exact filter before moving on to 
transport maps and the OT formulation. 
Henceforth we assume
$X_t \in \mathbb R^n$ and $Y_t \in \mathbb R^m$. 
The posterior distribution~(\ref{eq:posterior}) admits a recursive update
law that is essential for the design of filtering algorithms. To present this recursive update, we introduce the
following operators: 
\begin{subequations}
\begin{align}\label{eq:propagation}
    \text{(prop.)}\quad \pi \mapsto \mathcal A \pi &:= \int_{\mathbb R^n} a(\cdot \mid x) \pi(x)d x\\\label{eq:conditioning}
  \text{(cond.)}\quad \pi \mapsto \mathcal B_y (\pi) \!&:=\! \frac{h(y \mid \cdot)\pi(\cdot)}{\int_{\mathbb R^n} h(y \mid x) \pi(x)d x}
\end{align}
The first operator represents the update for the distribution of the state according to the dynamic model~(\ref{eq:model-dyn}). The second operator represents Bayes' rule that
carries out the conditioning  according to 
the observation model~(\ref{eq:model-obs}).  
The exact posterior distribution $\pi_t:=\mathbb P(X_t \in 
\cdot \mid \mathcal Y_t)$ then follows the sequential update law \cite{cappe2009inference}
\begin{align}\label{eq:exact-posterior}
 \pi_{t \mid t-1} &= \mathcal{A} \pi_{t-1},\quad 
 \pi_{t} =\mathcal{B}_{Y_t} (\pi_{t \mid t-1}),
\end{align} 
where the notation $\pi_{t \mid t-1}(\cdot):=\mathbb P(X_t \in \cdot \mid \mathcal Y_{t-1})$ is used for the distribution of $X_t$ before applying the conditioning with respect to the current observation $Y_t$. 

\end{subequations}

\subsection{Conditioning with transport maps}

Transport or coupling-based approaches approximate the posterior $\pi_t$ with an empirical distribution of a set of particles $\{ X^i_t\}_{i=1}^N$, where the particles are simulated according to the equations~(\ref{eq:Xi-propagation}) and~(\ref{eq:Xi-conditioning}). Although not explicitly stated, the transport map $T_t$, that is used in the conditioning step~(\ref{eq:Xi-conditioning}), depends on the empirical distribution of the particles. As a result, the equation~(\ref{eq:CIPS}) represents an {\it interacting particle system}.  

The problem of designing the transport map $T_t$ is studied in the {\it mean-field} limit ($N=\infty$), which assumes that all particles are independent and identically distributed according to a single mean-field distribution, denoted by $\overline \pi_t$. Under this assumption, the update law for $\overline \pi_t$ may be expressed as
\begin{equation}\label{eq:bar-pi}
     \overline \pi_{t|t-1} =\mathcal A \overline \pi_{t-1},\quad \overline \pi_t = T_t(\cdot,Y_t)_\# \overline \pi_{t|t-1},
\end{equation}
where $\#$ denotes the push-forward operator. The transport map $T_t$ should be designed such that $\overline \pi_t = \pi_t$ for all $t\geq 0$. 
Comparing equation~(\ref{eq:bar-pi}) with~(\ref{eq:exact-posterior}) concludes the following general design problem: 

{\bf Problem:} For all probability distributions $\pi$, find a map $T$ such that 
\begin{equation}\label{eq:consistency-condition}
    T(\cdot,y)_{\#} \pi = \mathcal B_y(\pi),\quad \forall y.
\end{equation}
We call this procedure {\it conditioning with transport maps}. 
The following examples are illustrative. 

\begin{example}[Noiseless observation]
    Assume $X\sim \pi$ and $Y=h(X)$ where $h$ is an invertible map. Then, the conditional distribution $\mathcal B_y(\pi) = \delta_{h^{-1}(y)}$ can simply be represented via a transport map $T(x,y) = h^{-1}(y)$. 
\end{example}
\begin{example}[Gaussian]
Assume $X\sim \pi$ is Gaussian and the observation $Y \sim h(\cdot \mid X)$  corresponds to a linear observation model with additive Gaussian noise, meaning $X$ and $Y$ are jointly Gaussian. Then, the conditional distribution $\mathcal B_y(\pi)$ is Gaussian $N(\mu,\Sigma)$ 
\begin{align*}
    \mu&=\mathbb E[X] + K(y-\mathbb E[Y]),\\
    \Sigma&=\text{Cov}(X) - \text{Cov}(X,Y)\text{Cov}(Y)^{-1}\text{Cov}(Y,X),
\end{align*}
and $K=\text{Cov}(X,Y)\text{Cov}(Y)^{-1}$. The stochastic map
\begin{equation}\label{eq:affine}
    X \mapsto X + K(y-Y),
\end{equation}
transports $\pi$ to $\mathcal B_y (\pi)$ for any value of the observation $y$. This can be checked by noting that $X + K(y-Y)$ is a Gaussian random variable with the same mean and covariance as $N(\mu,\Sigma)$.  The
map in equation~(\ref{eq:affine}) is the basis for the ensemble Kalman filter (EnKF) algorithm; see \cite{evensen2006,bergemann2012ensemble,calvello2022ensemble}. 

\end{example}
The consistency condition~(\ref{eq:consistency-condition}) does not specify a unique map $T$. 
In the following subsection, we formulate the problem of finding a valid map $T$ as an OT problem alongside an 
stochastic optimization approach on which our numerical procedure is based on. 
\subsection{OT formulation}\label{sec:OT-formulation}
In order to present the OT formulation, consider  $X\sim \pi$, $Y\sim h(\cdot \mid X)$, and let $\overline X \sim \pi$ be an independent copy of $X$. Also, let $P_{X,Y}$ denote the joint distribution of $(X,Y)$, with marginals $P_X$ and $P_Y$. 
First observe that the condition (\ref{eq:consistency-condition}), which may be expressed as $T(\cdot,y)_{\#}P_X = P_{X|Y}(\cdot|y)$ a.e.,  is equivalent 
to 
\begin{subequations}
\begin{equation}\label{eq:T-constraint-joint}
       (T(\overline X,Y),Y) \sim P_{X,Y}.
\end{equation}
A justification for this result appears in Appendix~\ref{apdx:consistency} and~\citep[Thm. 2.4] {kovachki2020conditional}.
In order to select a unique map that satisfies the condition~(\ref{eq:T-constraint-joint}), we formulate the (conditional) Monge problem:
\begin{equation} \label{eq:Monge OT}
\begin{aligned}
    \min_{T \in \mathcal{M}(P_X \otimes P_Y)}\, &\mathbb E\left[c(T(\overline X,Y),\overline X)\right],\quad \text{s.t.~ (\ref{eq:T-constraint-joint}) holds.} 
\end{aligned}
\end{equation} 
where $c(x,x')=\frac{1}{2}\|x-x'\|^2$  and  $\mathcal{M}(P_X \otimes P_Y)$ is the set of  maps $\R^n \times \R^m \mapsto \R^n$ that are 
$P_X \otimes P_Y$-measurable. 
The optimization~(\ref{eq:Monge OT}) is viewed as the Monge problem between the independent coupling $(\overline X,Y) \sim P_X \otimes P_Y$ and the joint distribution $(X,Y)\sim P_{X,Y}$ with transport maps that are constrained to be block-triangular $(x,y) \mapsto (T(x,y),y)$. 
Upon using the Kantorovich duality and the definition of $c$-concave functions\footnote{with a squared loss $c$, $f$ is $c$-concave iff $\frac{1}{2}\|\cdot\|^2-f$ is convex.} 
the Monge problem~(\ref{eq:Monge OT}) becomes 
\begin{align}\label{eq:new_loss}
    \max_{ f \in c\text{-Concave}_x}\,
    \min_{T \in \mathcal{M}(P_X \otimes P_Y)}\, J(f,T),
\end{align}
where the objective function $J(f,T)$ is equal to
\begin{equation*}
    \mathbb E\left[f(X,Y) - f(T(\overline X,Y),Y)+  c(T(\overline X,Y),\overline X)\right],
\end{equation*}
\end{subequations}
and the set $c\text{-Concave}_x$ denotes the set of functions on $\R^n \times \R^m \mapsto \R$
that are $c$-concave in their first variable everywhere. A rigorous justification of the max-min formulation, in the standard OT setting, appears in Appendix~\ref{apdx:max-min} and~\cite{makkuva2020optimal,rout2022generative}.  
The following result is the extension to the conditional setting. 

\begin{proposition} \label{prop:consistency}
Assume $\pi$ is absolutely continuous with respect to the Lebesgue measure with a convex support set $\mathcal{X}$, $\mathcal B_y(\pi)$ admits a density with respect to the Lebesgue measure 
$\forall y$, and $\mathbb E[\|X\|^2]<\infty$. Then, there exists a unique pair $(\overline f,\overline T)$, modulo an additive constant for $\overline{f}$, that solves the optimization problem~(\ref{eq:new_loss}) and  the map $\overline T(\cdot,y)$ is the OT map from $\pi$ to $\mathcal B_y(\pi)$ for a.e. $y$. 
\end{proposition}
\begin{sproof}
The proof is a direct consequence of applying 
 Theorem 2.3 (\romannum{3}) in \cite{carlier2016vector} and Porp.~\ref{prop:max-min-OT} for the max-min formulation in the standard OT setting;  
 see Appendix \ref{proof:consistency} for details.
\end{sproof}
The next proposition is concerned with the case that the max-min optimization problem~(\ref{eq:new_loss}) is not solved exactly
and provides an error bound for the computed OT maps 
in terms of the pertinent optimality gap. 
More precisely, for a pair $(f,T)$, let 
\begin{equation}\label{eq:opt-gaps}
\begin{aligned}
        \epsilon(f,T) := &J(f,T) - \min_{S} J(f,S) \\+& \max_{g}\min_{S}J(g,S) - \min_{S} J(f,S) 
\end{aligned}
\end{equation}
be the total optimality gap for the max-min problem. 
We then have the following result { which can be viewed as an extension of \citep[Thm. 4.3]{rout2022generative} and \citep[Thm. 3.6]{makkuva2020optimal} by adding the observation variable $Y$}. The proof appears in Appendix~\ref{proof:map estimation error bound}.

\begin{proposition} \label{prop:map estimation error bound}
Consider the setting of Prop.~\ref{prop:consistency} with the optimal pair $(\overline f, \overline T)$. Let $(f,T)$ be a possibly non-optimal pair with an optimality gap  $\epsilon(f, T)$.  Assume $x \mapsto \frac{1}{2}\|x\|^2-f(x,y)$ is $\alpha$-strongly 
convex in $x$ for all $y$. Then,  
\begin{align*}
        \mathbb{E}\,&[\|T(\overline{X},Y) - \overline T(\overline{X},Y)\|^2]
        \leq \frac{4}{\alpha}\epsilon(f,T).
    \end{align*}
\end{proposition}
 
\subsection{Empirical approximation}
In order to numerically solve the optimization problem~(\ref{eq:new_loss}), the objective function $J(f,T)$ is approximated empirically in terms of samples according to   
\begin{equation}\label{eq:emperical_loss}
    \begin{aligned}
        &J^{(N)}(f,T):= \frac{1}{N} \sum_{i=1}^N \big[f(X^i,Y^i) \quad+ \\
    &\frac{1}{2}\|T( \overline X^i,Y^i)- X^i\|^2 - f(T(\overline X^i,Y^i),Y^i)\big]
    \end{aligned}
\end{equation}
where $X^i,\overline X^i\overset{\text{i.i.d}}{\sim} \pi$ and $Y^i \sim h(\cdot|X^i)$. In the context of filtering problem, $\{X^i\}_{i=1}^N$ is constructed from the particles $\{X^i_{t|t-1}\}_{i=1}^N$ and $\{\overline X^i\}_{i=1}^N$ may be constructed by an independent   random shuffling of the same set, resulting in the filtering algorithm presented in Algorithm~\ref{alg:otpf} in Appendix~\ref{apdx:numerics}.  The function $f$ and the map $T$ are represented with a parametric class of functions, denoted by $\mathcal F$ and $\mathcal T$ respectively.
This concludes the optimization problem
\begin{equation}\label{eq:empirical-optimization}
    \max_{f\in \mathcal F}\,\min_{T \in \mathcal T}\, J^{(N)}(f,T). 
\end{equation}
Here we propose to take both of these spaces to be neural network classes with architectures 
that are summarized in Fig.~\ref{tikz:static_struc}; further details 
about these architectures can be found in 
the Appendix~\ref{apdx:numerics}.
\begin{remark}\label{remark:TH}
    Note that our choice of $\mathcal F$ does not impose the constraint that $f(x,y)$ 
    is $c\text{-Concave}_x$. We make this choice due to practical limitations of imposing convexity constraints on 
    neural nets using, for example, input-convex networks \cite{amos2016input,bunne2022supervised}. However, note that if the computed 
    $f$ happens to be $c\text{-Concave}_x$ (which one can check a posteriori) then Prop.~\ref{prop:map estimation error bound} remains applicable. {We provide qualitative and quantitative results, in Sec.~\ref{sec:Static_Example}},  that measure the convexity of $\frac{1}{2}\|x\|^2-f(x,y)$ and monotonicity of the map $T$ for a particular example.
\end{remark}

\begin{remark}
The proposed computational procedure may be extended to the Riemannian manifold setting by using the square of the geodesic distance as the cost function $c$ and modelling the map $T$ as exponential of a parameterized vector-field; see~\cite{grange2023computational}.
\end{remark}

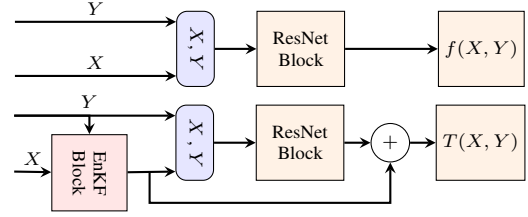
\begin{figure}[t]
\scriptsize
\centering 
	\begin{tikzpicture}[node distance=1cm, auto, scale = 0.8]
        \node (input) [startstop,minimum width=1cm, minimum height=0.5cm,rotate=270] at (0,0) {$X$, $Y$ \par};
        
	\node (x) [input] at (-3,-0.45) {};
         \node (y) [input] at (-3,0.45) {};
         
         \node (resnet) [process,minimum width=1cm, minimum height=1cm, text width=1cm] at (1.75,0) {ResNet Block \par};

	\node (f) [process,minimum width=1cm] at (4.75,0){$f(X,Y)$ \par };
    \draw[arrow] (x) -- node [above]{$X$}($(input.270) + (0,-0.45)$);
    
    \draw[arrow] (y) -- node [above]{$Y$}($(input.270) + (0,0.45)$);
    
        \draw[arrow] (input) -- (resnet);
	\draw[arrow] (resnet) -- (f);
	
	\end{tikzpicture}
        
        \begin{tikzpicture}[node distance=1cm, auto, scale=0.8]
        \node (input) [startstop,minimum width=1cm, minimum height=0.5cm,rotate=270] at (0,0) {$X$ , $Y$ \par};
        \node (enkf) [process,fill=red!10, rotate=270, text width=0.85cm, minimum width=1cm] at (-1.75,-0.5){EnKF Block\par};
        
	\node (x) [input] at (-3,-0.48) {};
         \node (y) [input] at (-3,0.45) {};        
    
	\node (resnet) [process,minimum width=1cm, minimum height=1cm, text width=1cm] at (1.75,0) {ResNet Block \par};

        \node (sum) [sum] at (3.25,0){$+$ \par };
	\node (T) [process,minimum width=1cm] at (4.75,0){$T(X,Y)$ \par };
    \draw[arrow] (x) -- node {$X$} (enkf);
    \draw[arrow] (y) -| node [above]{$Y$} (enkf);
    \draw[arrow] (y) -- ($(input.270) + (0,0.45)$);
	\draw[arrow] (enkf) -- ($(input.270) + (0,-0.48)$);
        \draw[arrow] (input) -- (resnet);
	\draw[arrow] (resnet) -- (sum);
	\draw[arrow] (sum) -- (T);

    \draw[arrow] (-0.75,-0.48) |- (3.25,-1) -| ($(sum.270)$);
	
	\end{tikzpicture}
 	\caption{Neural net architectures for the function classes  $\mathcal F$ 
  and $\mathcal T$ within our proposed algorithm.}
	\label{tikz:static_struc}
 \end{figure}
\subsection{Comparison with SIR}\label{sec:comparison-SIR}
In order to compare the OT and the SIR approaches, we use the {\it dual bounded-Lipschitz metric} $d(\mu,\nu) := \sup_{g \in \mathcal G} \sqrt{\mathbb{E}\left|\int g d \mu - \int g d \nu\right|^2}$ on (possibly random) probability measures $\mu,\nu$, where 
$\mathcal G$
is the space of functions that are  bounded by one and  Lipschitz with {a constant} smaller than one.

\begin{figure*}[t]
    \centering
    \begin{subfigure}{0.48\textwidth}
    \centering
    \includegraphics[width=0.95\textwidth, ,trim={70 20 90 40},clip]{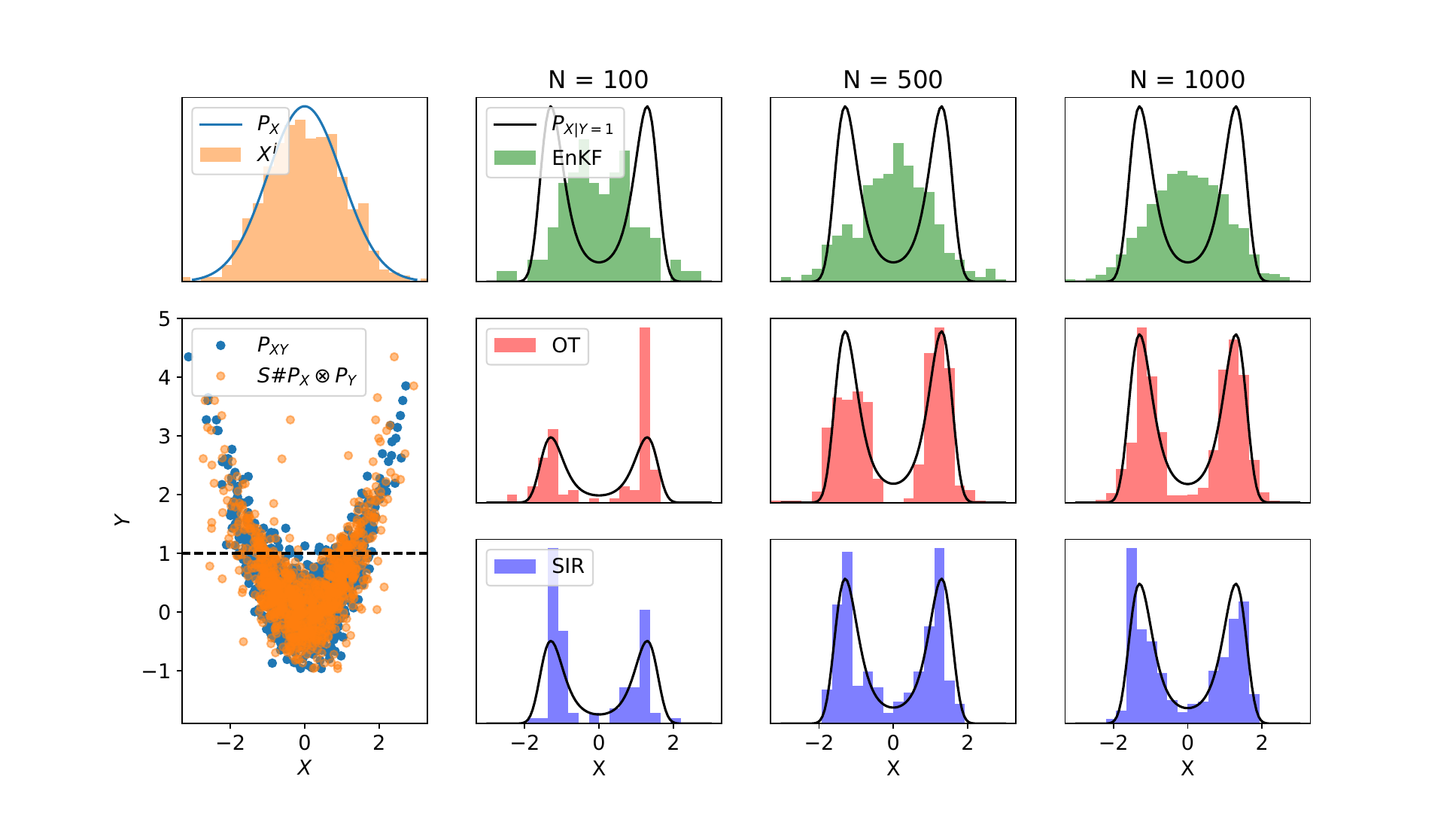}
    \caption{$\lambda_w=0.4$.}
    \label{fig:squared_not_SNR}
\end{subfigure}
\hfill
\begin{subfigure}{0.5\textwidth}
    \centering
    \includegraphics[width=0.95\textwidth,trim={70 20 90 40},clip]{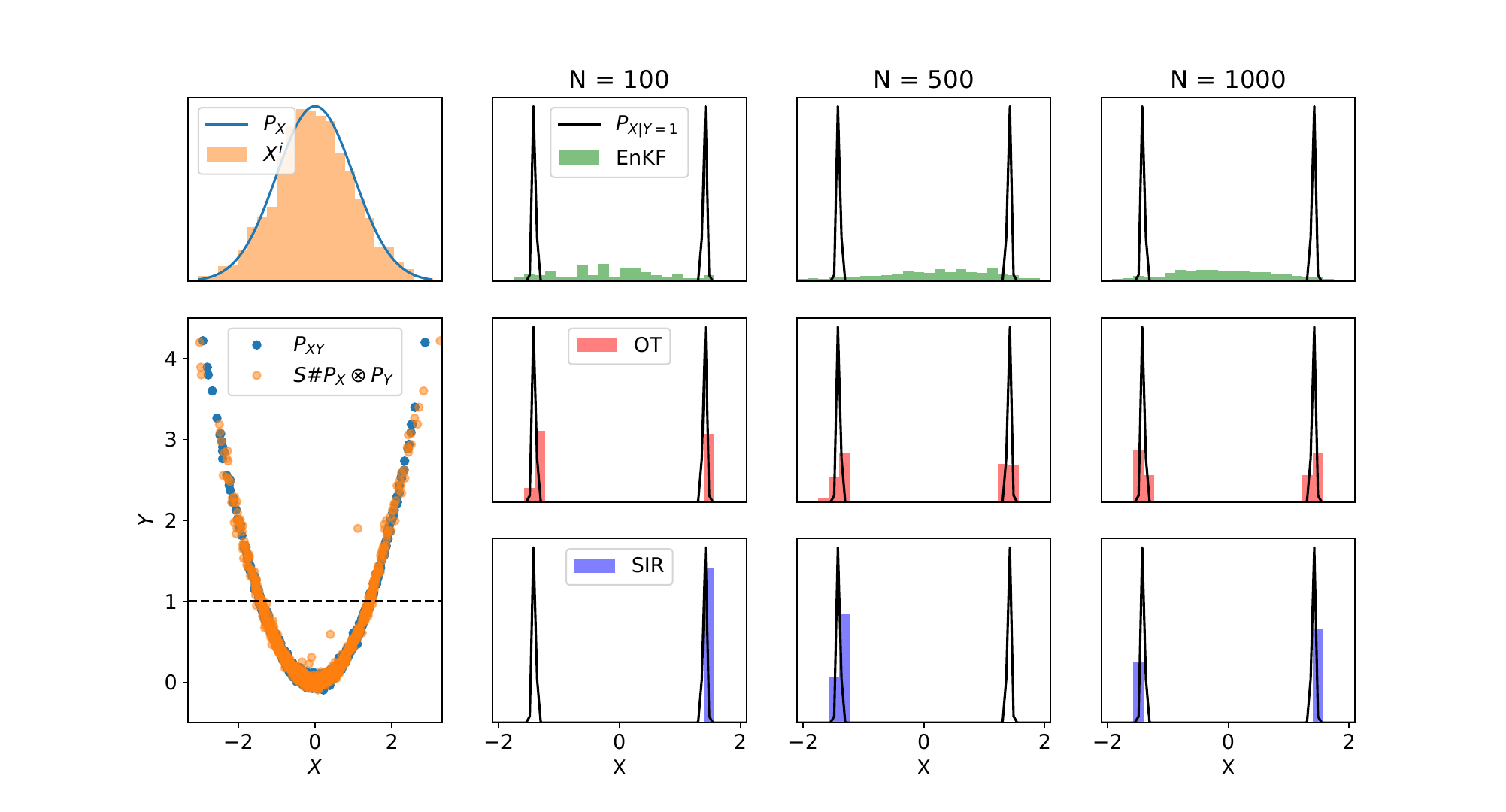} 
    \caption{$\lambda_w=0.04$.}
    \label{fig:squared_SNR}
\end{subfigure}
    \caption{Numerical results for the static example in Sec.~\ref{sec:Static_Example}. (a) top-left: Samples $\{X^i\}_{i=1}^N$ from the prior $P_X$; bottom-left: samples $\{(X^i,Y^i)\}_{i=1}^N$ from the joint distribution $P_{XY}$ in comparison with the transported samples $\{(T(X^{\sigma_i},Y^i),Y^i)\}_{i=1}^N$; rest of the panels: transported samples for $Y=1$ for different values of $N$ and three different algorithms. (b) Similar results to panel (a) but  for a smaller $\lambda_w$.}
    \label{fig:squared}
\end{figure*}

For any probability distribution $\pi$, the SIR and OT approaches approximate the conditional distribution $\mathcal B_y(\pi)$ in two different ways. The SIR approach approximates the posterior with 
\begin{subequations}
\begin{equation}
    \pi^\SIR := \frac{\sum_{i=1}^Nh(y|X^i)\delta_{X^i}}{\sum_{i=1}^N h(y|X^i)}=\mathcal B_y(S^{(N)}(\pi))
    \end{equation}
where $X^i\overset{\text{i.i.d.}}{\sim} \pi$ and $\pi \mapsto S^{(N)} \pi := \frac{1}{N}\sum_{i=1}^N \delta_{X^i}$ is the sampling operator. 
On the other hand, the OT approach approximates the posterior by numerically solving the optimization problem~(\ref{eq:empirical-optimization}). Assuming the pair $(\widehat f,\widehat T)$ is a solution of this problem, the approximated posterior is 
    \begin{equation}\label{eq:pi-OT}
        \pi^{(OT)} := \widehat T (\cdot,y)_\# \pi.
    \end{equation}
    \end{subequations}  
The following proposition provides an analysis of the approximation errors.
\begin{proposition}\label{prop:error-analysis}
    Consider $X\sim \pi$, $Y \sim h(\cdot|X)$, and the SIR and OT approaches explained above in order to approximate the posterior $\mathcal B_Y(\pi)$. Then, 
    \begin{subequations}
    \begin{align}\label{eq:SIR-bound}
    \liminf_{N \to \infty}\sqrt{N} d(\pi^\SIR,\mathcal B_Y(\pi)) &\geq  \sup_{g \in \mathcal G}\sqrt{V_h(g)},\\
    d(\pi^{(OT)},\mathcal B_Y(\pi))&\leq \sqrt{\frac{4}{\alpha}\epsilon(\widehat f,\widehat T)},\label{eq:OT-bound}
\end{align}
    \end{subequations} 
where $V_h(g):=\mathbb E[\overline h(Y|\overline X)^2(g(\overline X)-\mathbb E[g(X)|Y])^2]$, $\overline h(x|y):=h(x|y)/\int h(x'|y)d\pi (x')$, $\overline X$ is an independent copy of $X$, $\alpha$ is the strong convexity constant for $x\mapsto\frac{1}{2}\|x\|^2 - \widehat{f}(x,y)$, and $\epsilon(\widehat f,\widehat T)$ is the optimality gap~(\ref{eq:opt-gaps}) for the pair $(\widehat f,\widehat T)$.
\end{proposition}
\begin{sproof}
    The lower-bound~(\ref{eq:SIR-bound}) follows from an application of the central limit results available for importance sampling in~\citep[Thm. 9.1.8]{cappe2009inference}. The upper-bound (\ref{eq:SIR-bound}) follows from Prop.~\ref{prop:map estimation error bound} and the definition of the metric
    $d$; see Appendix~\ref{apdx:error-analysis} for details.
\end{sproof}
The COD in SIR PF is demonstrated by considering the special case where $X=[X(1),\ldots,X(n)]$ and $Y=[Y(1),\ldots,Y(n)]$ are $n$-dimensional with independent and identically distributed components. Then, for $N$ large enough, the SIR error satisfies the bound 
\[d(\pi^\SIR,\mathcal B_Y(\pi)) \geq  \frac{C\gamma^n}{\sqrt{N}}\]
where $C$ is a constant and $\gamma:=\mathbb E[\overline h(Y(1)|\overline X(1))^2]>1$. The COD is usually avoided by exploiting problem specific properties, such as spatial correlations or  decay; see ~\cite{rebeschini2015can,van2019particle}.   

The OT upper-bound~(\ref{eq:OT-bound}) depends on the optimality gap 
$\epsilon( \widehat f, \widehat T)$ which, in principle, decomposes to a bias and variance term. The bias term corresponds to the representation power of the function classes $\mathcal{F}$ and $\mathcal T$, in comparison with the complexity of the problem. The variance term 
corresponds to the statistical generalization errors due to  the empirical approximation of the objective function. The variance term is expected to grow as $O(\frac{1}{\sqrt{N}})$ with a proportionality constant that depends on the complexity of the function classes, but independent of the dimension.  In principle, the OT approach may also suffer from the COD under no additional assumptions on the problem. However, in comparison to SIR, it provides a more flexible design methodology that  can exploit problem specific structure and regularity.

    
\section{Numerical results}\label{sec:numerics}
    \begin{figure*}[t]
    \centering
\begin{subfigure}{0.24\textwidth}
    \centering
    \includegraphics[width=1.0\textwidth,trim={30 0 70 60},clip]{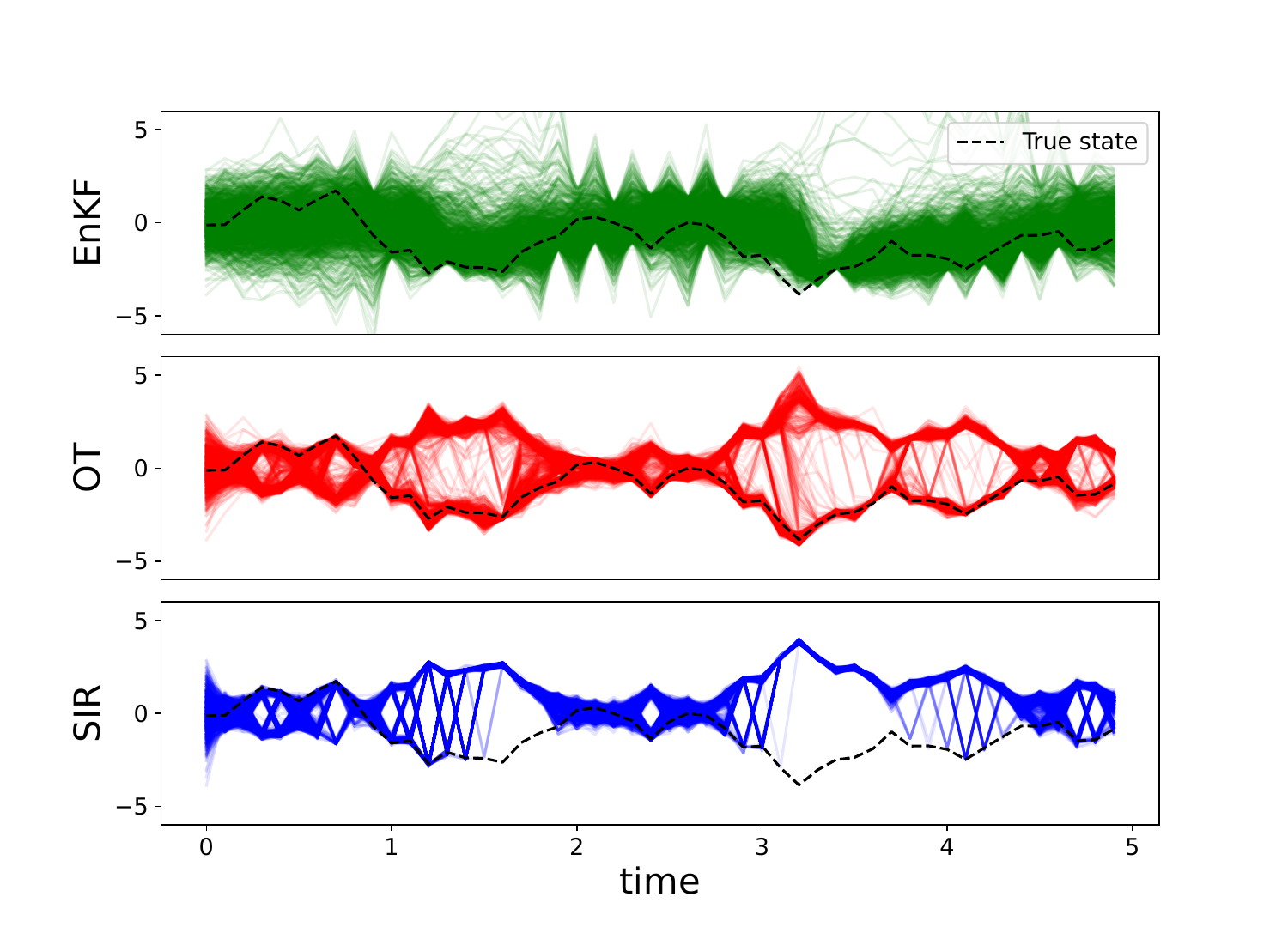}
    \caption{Particles trajectory.}
\end{subfigure}
\hfill
\begin{subfigure}{0.24\textwidth}
    \centering
     \includegraphics[width=1.0\textwidth,trim={30 0 70 60},clip]{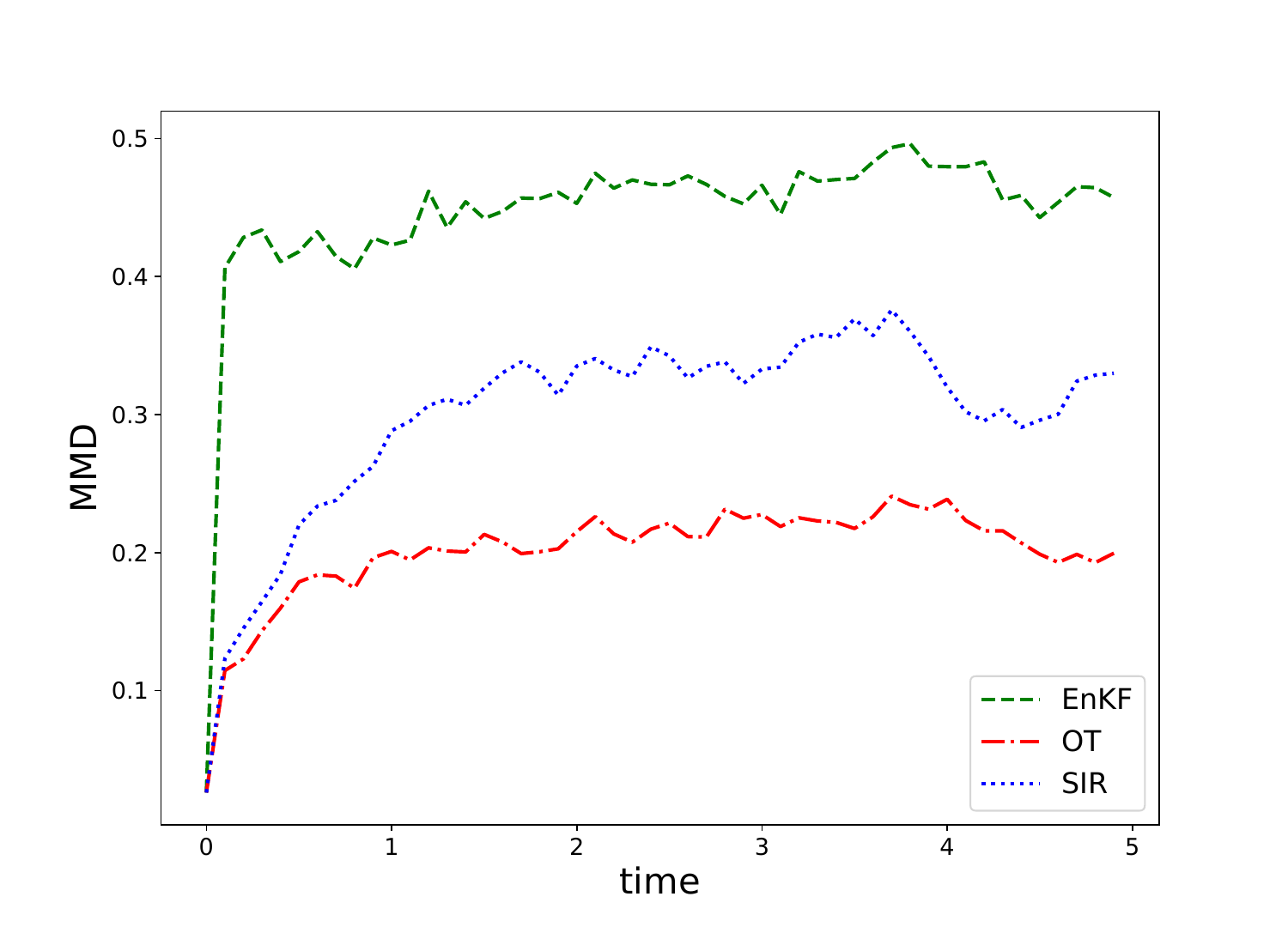} 
    \caption{MMD vs time.}
\end{subfigure}
\hfill
\begin{subfigure}{0.24\textwidth}
    \centering
    \includegraphics[width=1.0\textwidth,trim={30 0 70 60},clip]{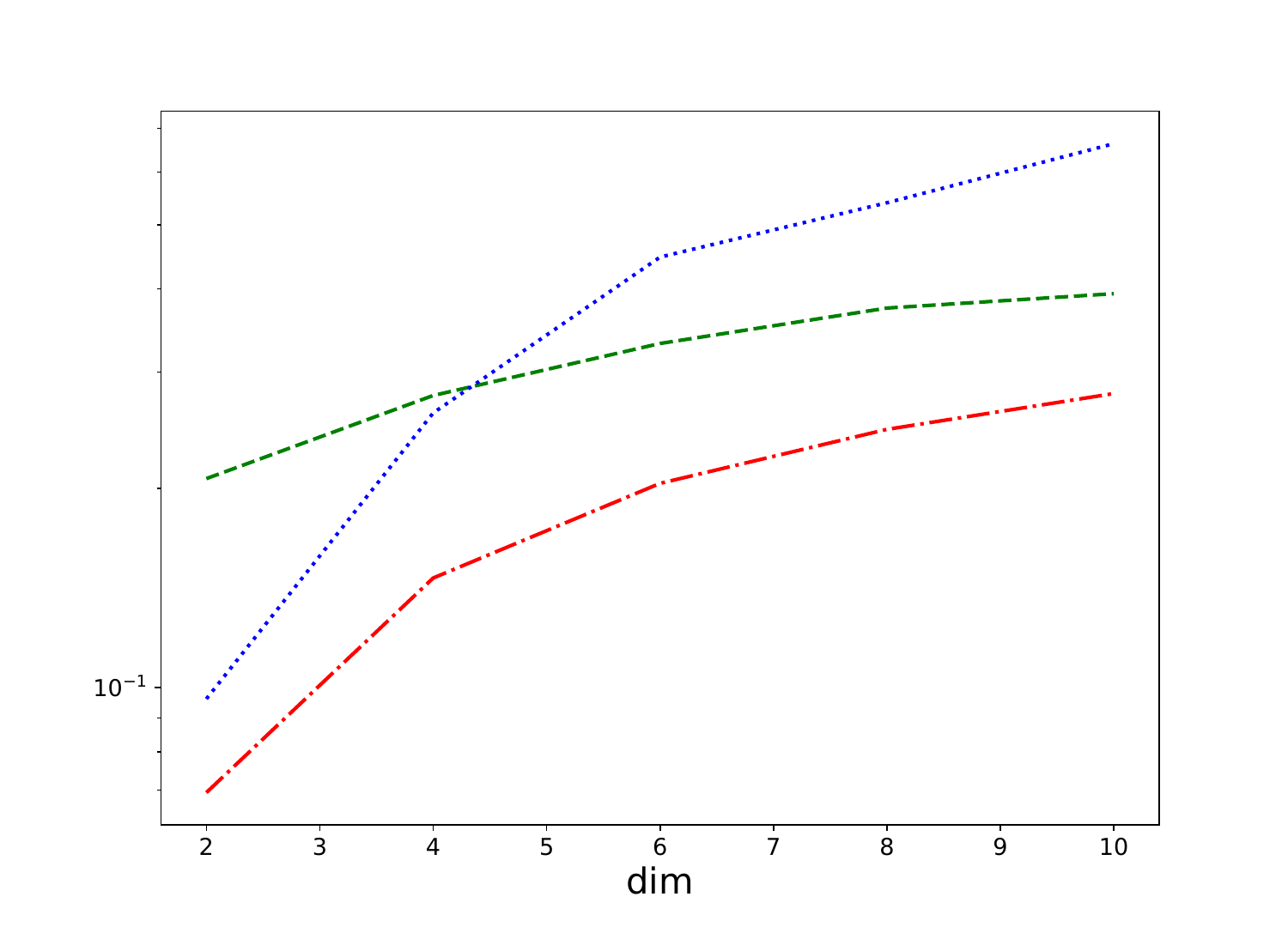}
    \caption{MMD vs the dimension.}
\end{subfigure}
\hfill
\begin{subfigure}{0.24\textwidth}
    \centering
     \includegraphics[width=1.0\textwidth,trim={30 0 70 60},clip]{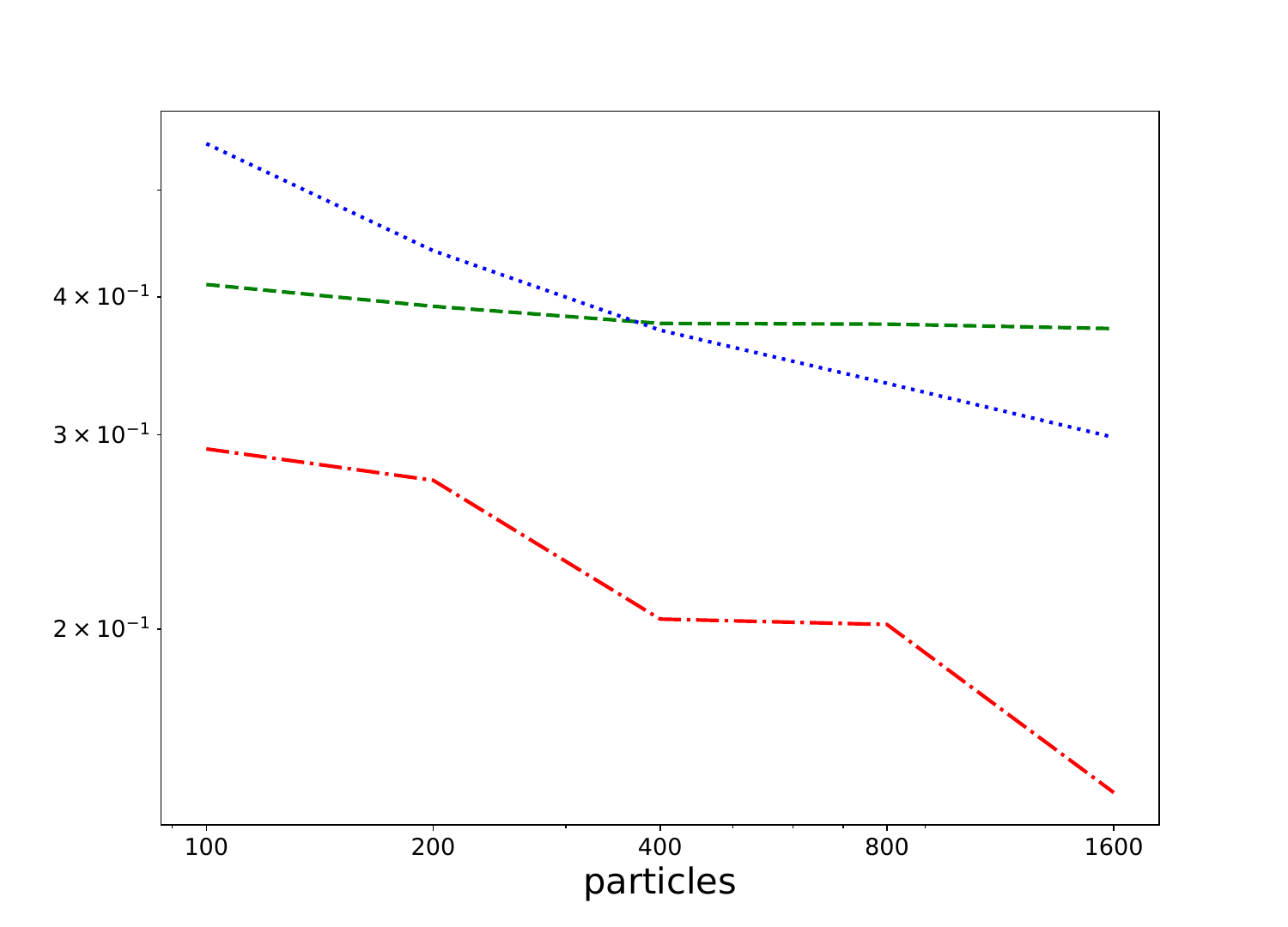} 
    \caption{MMD vs $\#$ of particles.}
\end{subfigure}

     \caption{Numerical results for the dynamic example~\ref{eq:model-example}. The left panel shows the trajectory of the particles $\{X^1_t,\ldots,X^N_t\}$ along with the trajectory of the true state $X_t$ for EnKF, OT, and SIR algorithms, respectively. The second panel shows the MMD distance with respect to the exact conditional distribution. The last two panels show MMD variation with dimension and the number of particles.}
    \label{fig:dynamic_example_states}
\end{figure*}

We perform several numerical experiments to study the performance of the OT approach in comparison with the EnKF and SIR algorithms. 
The OT algorithm \ref{alg:otpf}
consists of solving (\ref{eq:CIPS}) and (\ref{eq:empirical-optimization}) with $f, T$ 
taken to be neural nets whose architectures were outlined in Fig.~\ref{tikz:static_struc}. The network weights are learned with a gradient ascent-descent procedure using the Adam optimization algorithm. To reduce the computational cost, the optimization iteration number decreases as the time grows because the OT map is not expected to change significantly from a time step to the next one. The computational cost and details of all three algorithms and additional numerical results appear in the Appendix~\ref{apdx:numerics}. The Python code for reproducing the numerical results is available online\footnote{\url{https://github.com/Mohd9485/Filtering-with-Optimal-Transport}}.

\subsection{A bimodal static example} \label{sec:Static_Example}
Our first numerical experiment serves as a proof of concept and demonstrates the sample efficiency of the OT approach when the likelihood becomes degenerate. Specifically, we consider the task of computing the conditional distribution of a Gaussian hidden random variable $X \sim N(0,I_n)$ given the observation
\begin{align} \label{example:squared}
    Y=\frac{1}{2}X\odot X + \lambda_w W, \qquad W \sim N(0,I_n)
\end{align}
where $\odot$ denotes the element-wise (i.e., Hadamard) product. This model is specifically selected to produce a bimodal posterior.  We only present the  $n=2$ case since the difference between OT and SIR was not significant when $n=1$.

The first numerical results for this model are presented with a noise 
standard deviation of $\lambda_w=0.4$
in Fig.~\ref{fig:squared_not_SNR}. The top left panel shows the initial particles as samples from the Gaussian prior distribution. 
The bottom left panel shows the pushforward of samples from 
$P_X \otimes P_Y$ via 
the block triangular map $S(x,y)=(T(x,y),y)$, in comparison to samples from $P_{XY}$, verifying the consistency condition (\ref{eq:T-constraint-joint}) 
for the map. Then, we pick a particular value for the observation $Y=1$ (as shown by the dashed line) and present the histogram of (transported) particles in comparison with the exact conditional density. It is observed that both OT and SIR capture the bimodal posterior, while EnKF falls short since it always approximates the posterior with a Gaussian. 

We repeat the procedure in
Fig.~\ref{fig:squared_SNR} but for a smaller noise standard deviation $\lambda_w=0.04$ which leads to a more degenerate posterior. Our results 
clearly demonstrate the  weight degeneracy of SIR even in this low-dimensional setting, as all particles collapse into a single mode, while the OT approach still captures the bimodal posterior. Additional numerical results for this model are available in Appendix~\ref{sec:app_static}.

 Additionally, we present the learned function $\frac{1}{2}\|x\|_2^2-\widehat{f}(x,y)$ and the learned map $\widehat{T}(x,y)$, as a function of $x$ and for a $1D$ case with fixed value $y=1$, in Fig.~\ref{fig:bimodal_convex}. It is observed that the function  $\frac{1}{2}\|x\|_2^2-\widehat{f}(x,1)$ is not entirely convex, specially around the region with small posterior probability. Meanwhile, the transport map $\widehat{T}(x,1)$ appears to be entirely monotone. In order to quantify the degree of monotonicity of  $x \mapsto \widehat{T}(x,y)$, for all values of $y$, we employ a Monte-Carlo method with $10^5$ samples to approximate
\begin{align*}
\mathbb{E}_{(X,X',Y)\sim P_X \otimes P_X \otimes P_Y} [\mathbf{1}_{\varrho(\widehat{T}(\cdot,Y),X,X')\leq 0}] &\approx 0.0012, \\\mathbb{E}_{(X,X',Y)\sim P_X \otimes P_X \otimes P_Y}  [\varrho(\widehat{T}(\cdot,Y),X,X')]&\approx 0.88
\end{align*}
where 
\(\varrho(S,x,x'):=\frac{(S(x)-S(x')^\top (x-x')}{\|x-x'\|^2}\)
for any map $S:\mathbb R^{n} \to \mathbb R^n$ and two points $x,x'\in \mathbb R^n$. If the map $S$ is monotone, then $\varrho(S,x,x')\geq 0$ for all $x,x'\in \mathbb R^n$. If $S$ is the gradient of a $\alpha$-strongly convex function, then, $\rho(S,x,x')\geq \alpha$. The first number shows that the map is not monotone over a set of small probability, while the second number shows that the ``degree'' of monotinicity is positive on average.  
We perform a similar  procedure to quantify the convexity of the function 
$\frac{1}{2}\|x\|_2^2-\widehat{f}(x,y)$ by evaluating the monotonicity of its gradient: $F(x,y):=x-\nabla_x \widehat{f}(x,y)$\footnote{Checking the positive-definiteness of the Hessian matrix is not applicable in this case because $\widehat{f}(x,y)$ is not differentiable when we use ReLU activation functions.}.
With a similar Monte-Carlo method  we obtained the respective quantities $0.033$ and $1.57$, 
indicating that $\frac{1}{2}\|x\|_2^2-\widehat{f}(x,y)$  is 
non-convex over a set of small probability measure under the posterior. This empirical observation motivates future direction of theoretical research concerning the relaxation of the convexity assumption in Prop.~\ref{prop:map estimation error bound} to 
convexity over sets of large measure under the posterior.

\begin{figure}[h]
    \centering
    \includegraphics[width=0.48\textwidth,trim={80 0 80 10},clip]{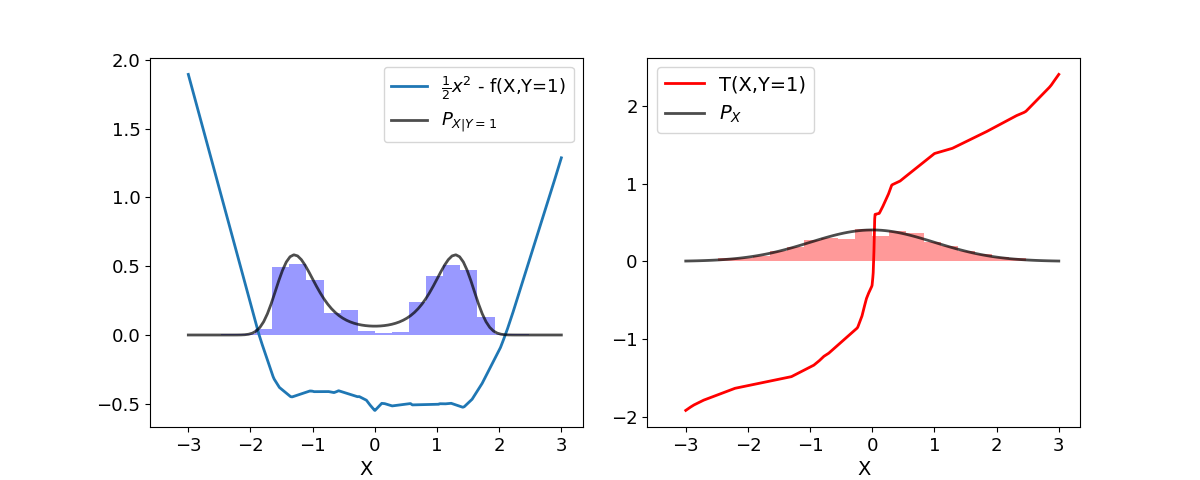}
    \caption{Numerical results for the bimodal static example  in sec.~(\ref{sec:Static_Example}). The left panel shows the  function $\frac{1}{2}x^2- \widehat{f}(x,1)$ and the conditional distribution $P_{X|Y=1}$. The right panel shows the map $\widehat{T}(x,1)$ and the prior distribution $P_X$.}
    \label{fig:bimodal_convex}
\end{figure}

\subsection{A bimodal dynamic example}\label{sec:dynamic-bimodal}
We consider a dynamic version of the previous example according to the following model:
\begin{subequations}\label{eq:model-example}
\begin{align}
    X_{t} &= (1-\alpha) X_{t-1} + 2\lambda V_t,\quad X_0 \sim \mathcal{N}(0,I_n)\\
    Y_t &= X_t\odot X_t + \lambda W_t,
\end{align}
\end{subequations}
where $\{V_t,W_t\}_{t=1}^\infty$ are i.i.d sequences of standard Gaussian random variables, $\alpha=0.1$ and $\lambda=\sqrt{0.1}$. The choice of $Y_t$ 
will once again lead to a bimodal posterior $\pi_t$ at every time step.

The numerical results are depicted in Fig.~\ref{fig:dynamic_example_states}: Panel (a) shows the trajectory of the particles for the three algorithms, along with the true state $X_t$ denoted with a dashed black line. The OT approach produces a bimodal distribution of particles, while the EnKF gives a Gaussian approximation and the SIR approach exhibits the weight collapse and misses a mode for the time duration $t\in[1,2] \cup [3,4.5]$. Panel (b) presents a quantitative error analysis comparing the maximum-mean-discrepancy (MMD) between the particle distribution of each algorithm  and the exact posterior; 
details such as the choice of the MMD kernel are presented in Appendix~\ref{sec:app_dynamic}.
Since the exact posterior is not explicitly available it is approximated by simulating the SIR algorithm with $N=10^5$ particles.  This quantitative result affirms the qualitative observations of panel (a) that the OT posterior better captures the true posterior in time.

We also performed a numerical experiment to study the effect of the dimension $n$ and the number of particles $N$ on the performance of the three algorithms. The results are depicted in panels (c) and (d), respectively. It is observed that both EnKF and OT scale better with dimension compared to SIR. However, as the number of particles increases, the EnKF error remains constant, due to its Gaussian bias, while the approximation error for SIR and OT decreases.

\subsection{The Lorenz 63 model}
\begin{figure}[t]
    \centering
    \includegraphics[width=0.43\textwidth,trim={75 0 100 20},clip]{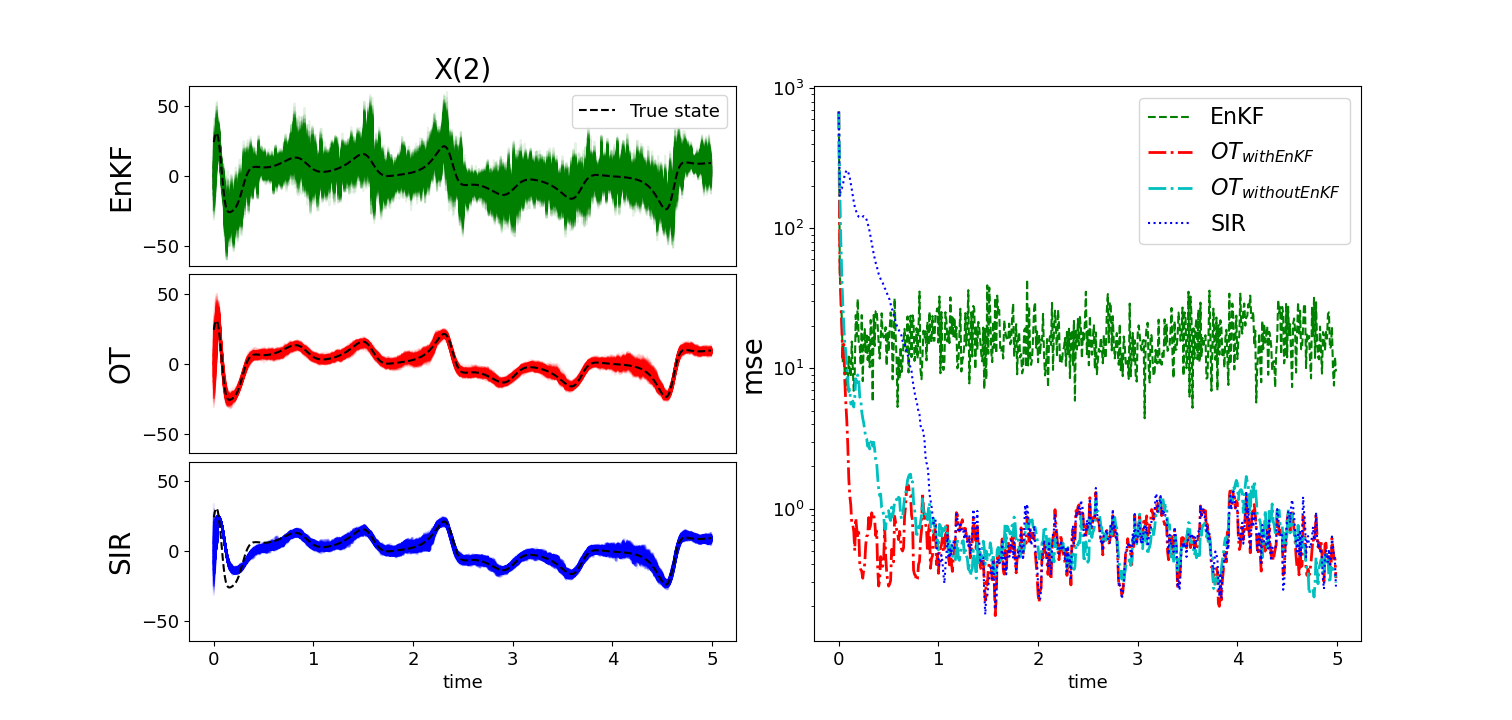}
    \caption{Numerical results for the Lorenz 63 example. The left panel shows the trajectory of the unobserved component of the true state and the particles. The right panel shows the MSE comparison.}
    \label{fig:state2_and_mse_L63}
\end{figure}
We present numerical results on the three-dimensional  Lorenz 63 model which often serves as a benchmark for nonlinear filtering algorithms. The model is presented in Appendix~\ref{sec:app_L63}. The
state $X_t$ is $3$-dimensional while the observation $Y_t$ is $2$-dimensional and consists of noisy measurements of the first and third components of the state. 

The numerical results are presented in Fig.~\ref{fig:state2_and_mse_L63}. The left panel shows the trajectory of the second component of the true state and the particles.  The OT and EnKF are quicker in converging to the true state, with EnKF admitting larger variance. The right panel shows the mean-squared-error (MSE) in estimating the state confirming the qualitative observations. We present the MSE result for two variations of the OT method: either the EnKF layer in the architecture of Fig.~\ref{tikz:static_struc} is implemented or not.  The results show that the addition of the EnKF layer helps with the performance of the filter, while computationally, we observed more numerical stability when the EnKF layer is removed. 

We have also performed numerical experiments on the Lorentz 96 model, another 
benchmark from the filtering literature, that appear in Appendix~\ref{sec:app_L96}.
The results for that example are in line with Lorentz 63, i.e., 
EnKF and OT outperform SIR but are somewhat similar to each other.

\subsection{Static image in-painting on MNIST}

\begin{figure}[t]
    \centering
    \includegraphics[width=\columnwidth,trim={105 50 80 20},clip]{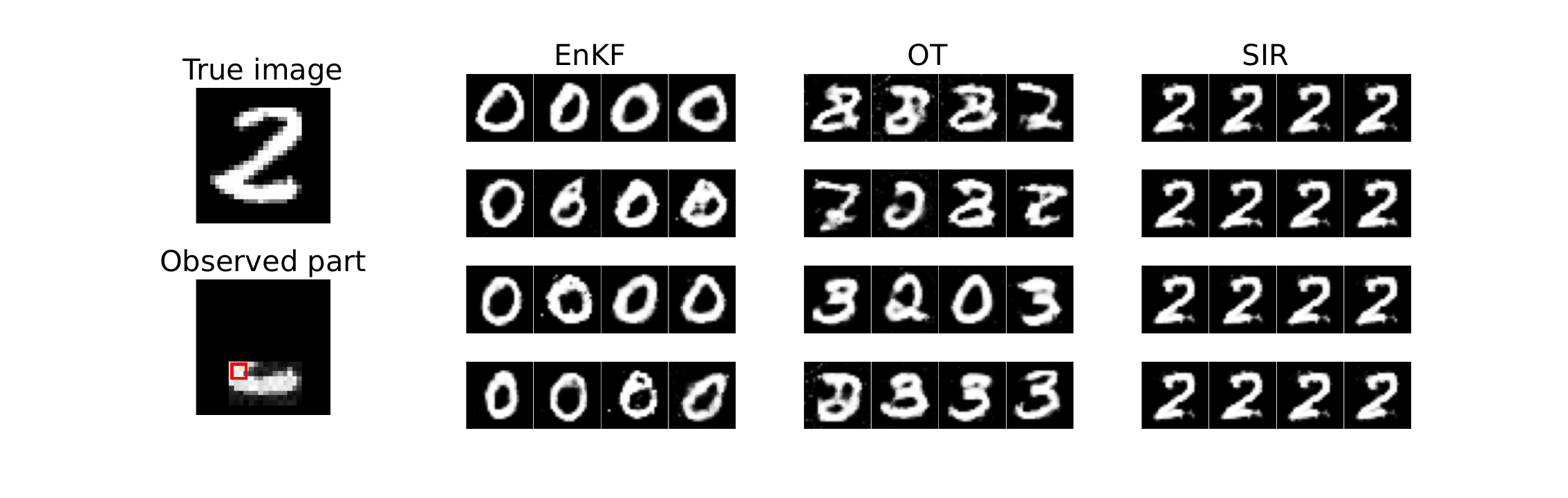}
    \caption{Numerical result for the static MNIST example.  The left column shows the true image, the observed part, and the observation window (red square). The next three columns show  $16$ random particles from the final time-step of each algorithm.}
    \label{fig:mnist_static_final_particles_example2}
\end{figure}

\begin{figure*}[t!]
    \centering
        \begin{subfigure}[b]{0.49\textwidth}
    \centering
    \includegraphics[width=0.95\textwidth,trim={105 60 105 50},clip]{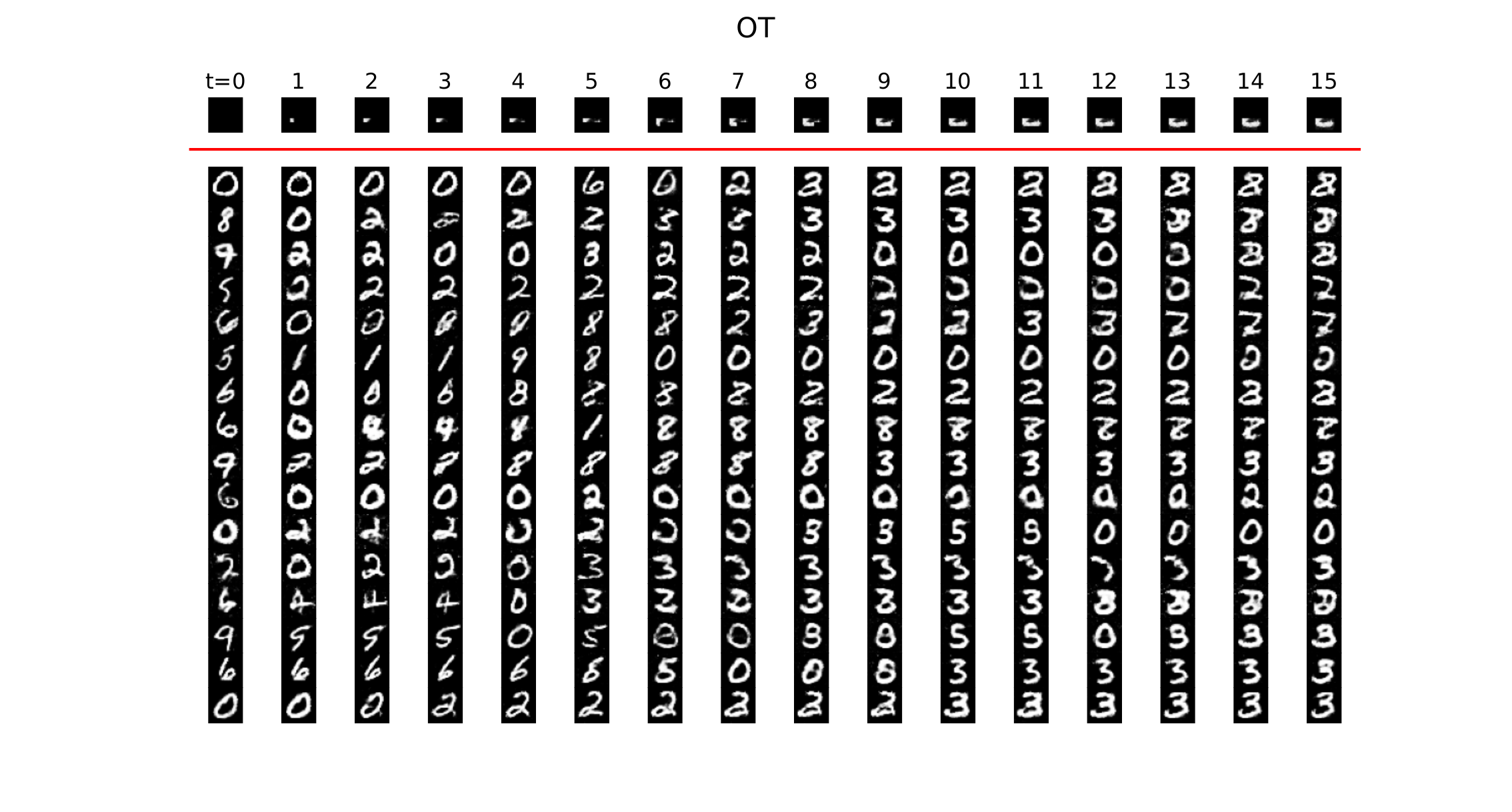}
    \caption{16 particles from the OT method.}
    \label{fig:mnist_static_OT_particles_example2}
    \end{subfigure}
\hfill
    \begin{subfigure}[b]{0.49\textwidth}
    \centering
    \includegraphics[width=0.95\textwidth,trim={95 60 105 50},clip]{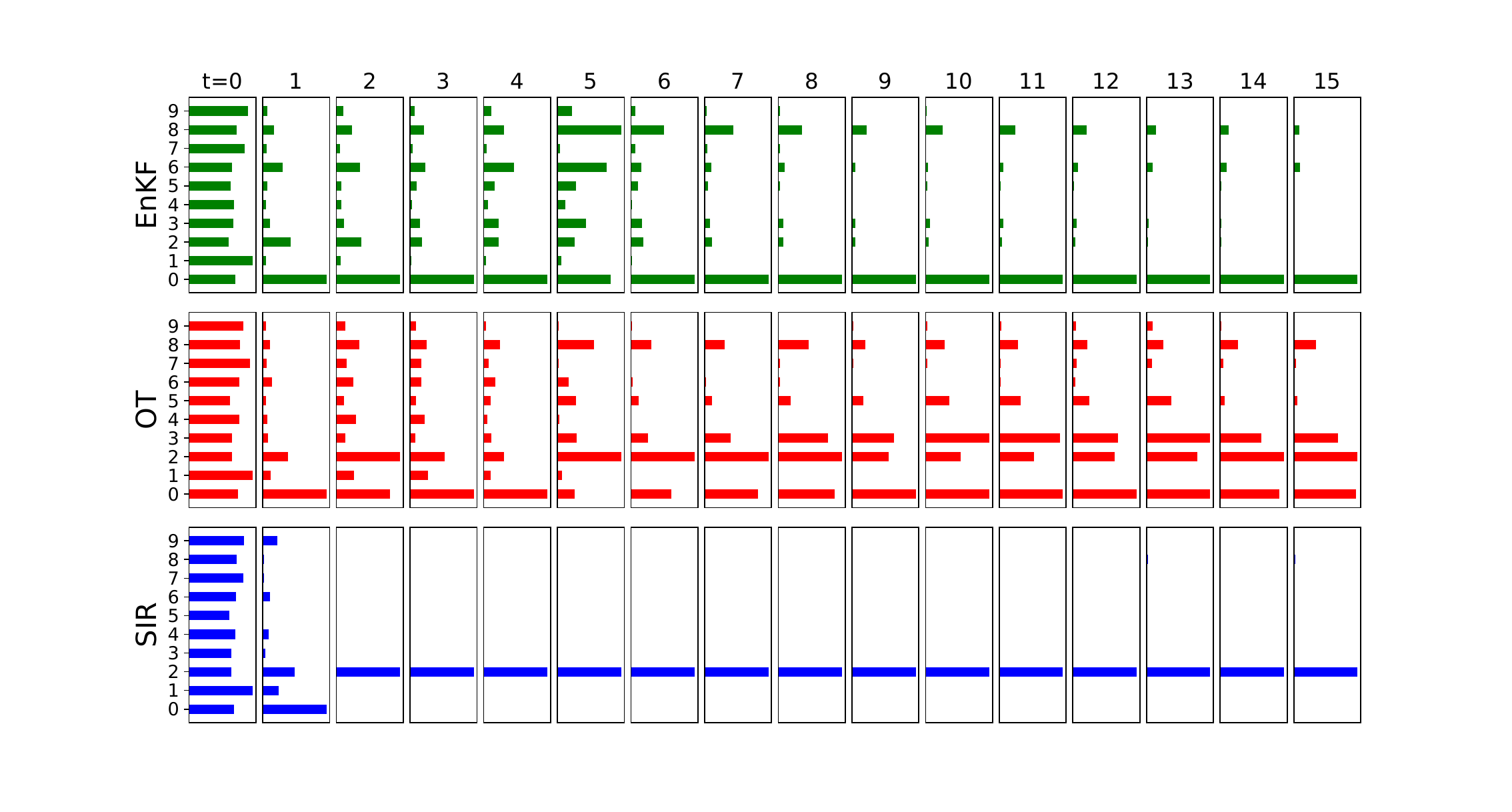}
    \caption{Digits distribution.}
    \label{fig:mnist_static_dist_particles_example2}
\end{subfigure}
\caption{Additional numerical results for the static MNIST example. (a) The first row shows the cumulative total observations up to each time step. The subsequent rows under the red  line show 16 particles from the OT method (that correspond to the same 16 particles in Fig.~\ref{fig:mnist_static_final_particles_example2}; (b) The histogram of the the digits generated by the particles, from the three algorithms as a function of time, evaluated using an accurate  MNIST classifier.}
\end{figure*}
In order to evaluate the high-dimensional scalability of the OT approach, we consider the problem of computing conditional distributions on the $100$-dimensional latent space of generative adversarial network (GAN) trained to represent the MNIST digits~\cite{goodfellow2014generative}. In particular, denoting the generator by $G : \mathbb{R}^{100}\rightarrow \mathbb{R}^{28\times 28}$, we consider the model:
\[Y_t = h(G(X),c_t) + \sigma W_t,\quad X\sim N(0,I_{100}),\]
where the observation function $(z,c) \in \mathbb{R}^{28\times 28} \times \mathbb R^2 \mapsto h(z,c)\in \mathbb{R}^{r\times r}$ is defined as the $r\times r$ window of pixels $z[c(1):c(1)+r,c(2):c(2)+r]$. The coordinates of the corner $c_t$ moves from left to right and top to bottom scanning a partial portion of the image called  the {\it observed part}. In order to make the problem more challenging, we fed a noisy measurement of the corner location to the algorithms. While the true image does not change, we included artificial Gaussian noise to the particles to avoid particle collapse. For for details, see Appendix~\ref{sec:app_static_mnist}.

The numerical results are presented in Fig.~\ref{fig:mnist_static_final_particles_example2}. The left column presents the true image, the observed part of the image, and an instance of the $3\times 3$ observation window with a red square. The filtering algorithms are implemented for $15$ time steps, as the observation window moves, and  $16$ random particles from the final time-step of each algorithm are depicted in the rest of panels in Fig.~\ref{fig:mnist_static_final_particles_example2}. The results show that although the SIR particles are similar to the true image, they do not capture the true uncertainty corresponding to the observation, in contrast to the OT approach (the lower part of the digits $0,3,8$ are similar to the lower part of the digit $2$).

In order to further demonstrate the performance of the OT approach, we present the trajectory of the particles in Fig.~\ref{fig:mnist_static_OT_particles_example2}. The top row shows the total observed part of the image up to that time-step, and the following 16 rows show the images generated from the particles that approximate the conditional distribution.  Moreover, we used an accurate MNIST classifier to represent the histogram of the digits generated from the particles of each algorithm in Fig.~(\ref{fig:mnist_static_dist_particles_example2}). According to these results, the SIR and EnKF methods are {\it overconfident}  that the true digit is $2$ and $0$, respectively, while the OT method produces a 
more realistic and multi-modal posterior concentrated around the digits $\{0,2,3,8\}$, 
all of which are natural in-paintings for the data.

\subsection{Dynamic image in-painting on MNIST}
We extend the previous example to the dynamic setting by introducing an update law for the latent variable of the GAN according to $X_{t+1} = (1-\alpha)X_{t} + \lambda V_t$, where $V_t$ is a standard Gaussian random variable. We consider the same observation model but we select a larger window size $r=12$ and constrain the the vertical coordinate of the window to be at the bottom, while the horizontal coordinate is randomly selected. 

A visual representation of the results is provided in Fig.~\ref{fig:mnist_dynamic_particles}. The top two rows display the true images and their observed portion at each time instant $t=1,2,\ldots,20$. Subsequent rows show the images corresponding to four randomly selected  particles from the
three algorithms; a larger figure depicting more particles is provided in Appendix~\ref{sec:app_dynamic_mnist}. The results show the particle collapse of the SIR algorithm at all time steps, convergence of the EnKF algorithm to an incorrect distribution, and the capability of the OT approach in approximating the uncertainty associated with the observation at each time step. A figure similar to the Fig.~\ref{fig:mnist_static_dist_particles_example2}  that shows the histogram of digits as a function of time is also provided in Appendix~\ref{sec:app_dynamic_mnist} which further 
highlights these phenomenon.

\begin{figure}[ht]
    \centering
    \includegraphics[width=\columnwidth,trim={115 60 90 40},clip]{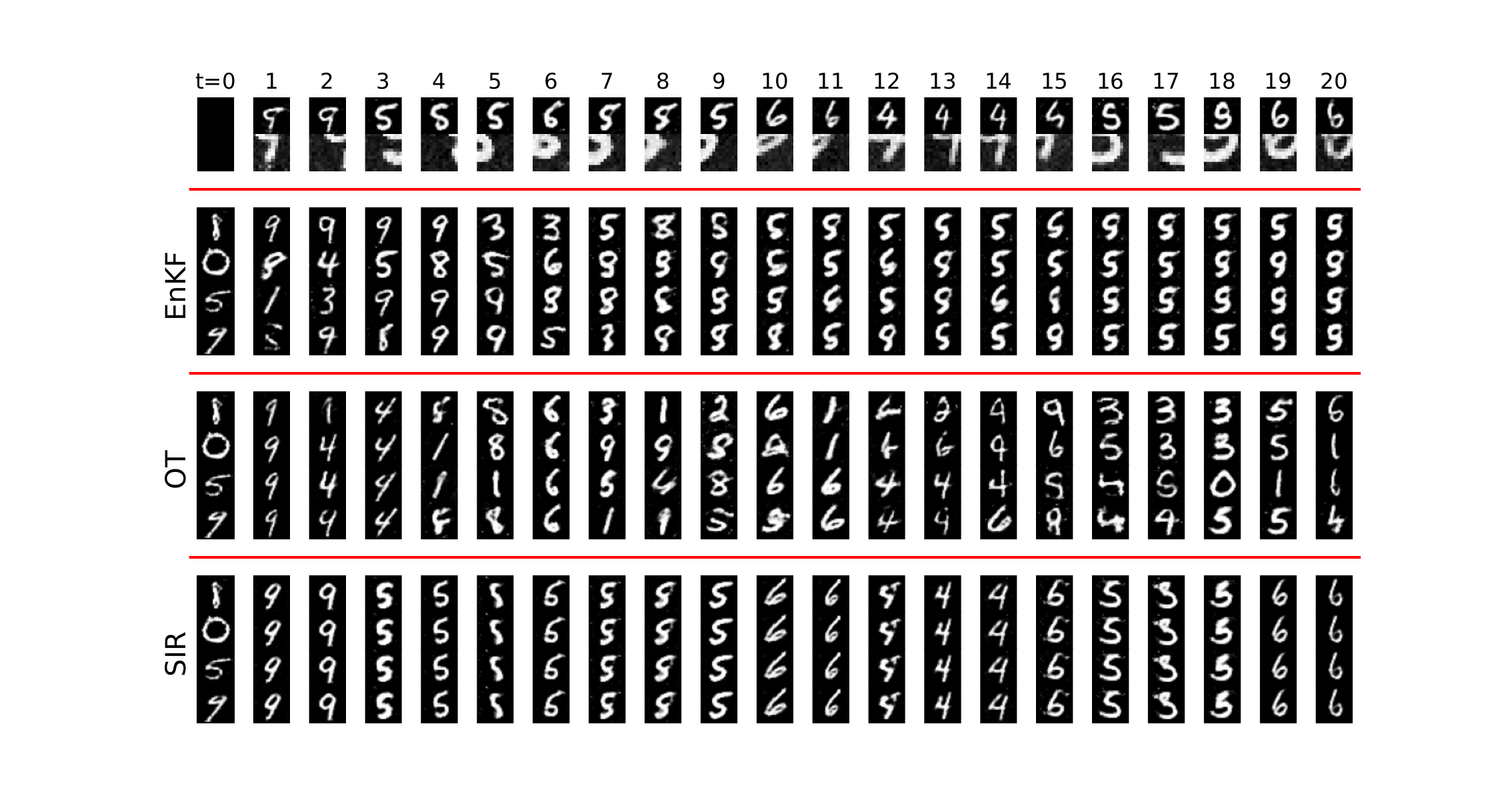}
    \caption{Numerical results for the dynamic MNIST example. The top two rows show the true image  and the observations at each time step. The rest of the rows show the first 4 particles for each method, representing the posterior at that time step. }
    \label{fig:mnist_dynamic_particles}
\end{figure}

\section{Comparison with coupling-based methods }\label{sec:comparison}
In this section, we provide a comprehensive comparison to other coupling-based methods for filtering. The particle flow method~\cite{daum2012particle,de2015stochastic} and feedback particle filter (FPF)~\cite{yang2013feedback,yang2016} involve either 
an ordinary differential equation or stochastic differential equation that updates the locations of the particles so that the probability density of the particles follows a given PDE. The particles' equation involves an unknown vector-field that needs to be approximated by solving a certain partial differential equation (PDE). 
The main challenge in this type of algorithms is to approximate the aforementioned vector-field at each time-step. The time discretization
for this type of equation often becomes unstable, especially for multi-modal posteriors or degenerate likelihoods. Our OT approach can be viewed as an exact time-discretization of the FPF algorithm, as shown in Prop. 2 in~\cite{taghvaei2022optimal}, which resolves the time-discretization issues discussed above. 

The ensemble transform particle filter~\cite{reich11}
involves solving a linear program for the discrete OT problem from a uniform prior distribution to the weighted posterior distribution for the particular value of the observation.   Solving the linear program becomes challenging as the number of particles $N$ increases. Moreover, approximating the marginal with a weighted empirical distribution suffers from the same fundamental issue that importance sampling particle filter suffers from.

The coupling method proposed in~\cite{spantini2022coupling} is the closest method to our approach. It is also likelihood free and amenable to the neural net parameterizations.  The main difference is in the form of the transport map. While in this paper we aim at finding the OT map from prior to the posterior, the approach in~\cite{spantini2022coupling} aims at finding the Knothe–Rosenblatt rearrangement. Thus, our approach is more closely related to the semi-dual solutions to the OT problem and can utilize the existing theoretical results and computational methodologies.

Our work is based on the variational formulation of the Bayes' law presented in~\cite{taghvaei2022optimal} (TH). TH proposes  $\min_f \max_g \mathbb E[f(X,Y) - f(\nabla_x g(\bar X,Y),Y) + \bar X^\top\nabla_x g(\bar X,Y)]$ where $f$ and $g$ are ICNNs. Our initial numerical study using this method did not provide compelling results due to the limitations of the ICNN architecture (please see Remark \ref{remark:TH}). To resolve this issue, we made two changes: we represent $f(x,y)$ as $\frac{1}{2}\|x\|^2 -\tilde f(x,y)$ and $\nabla_x g(x,y)$ as $T(x,y)$, where $T$ and $\tilde f$ are ResNets. We also include an EnKF block in the architecture for $T$ so that our method performs as well as EnKF if no training is performed. With these improvements, we were able to produce numerical results on toy, benchmark, and high-dimensional problems that are consistent and perform better than conventional baselines. The clear computational improvement can be observed by comparing Fig.~1 in the TH paper with 
Fig.~\ref{fig:squared} and Fig.~\ref{fig:bimodal} in this paper.

\section{Concluding remarks and limitations}\label{sec:discussion}
    We presented theoretical and numerical results demonstrating the performance of an OT-based nonlinear filtering algorithm. A complete characterization of the bias-variance terms in the optimality gap of~(\ref{eq:empirical-optimization}) is subject of ongoing work using the existing tools for convergence of empirical measures in Wasserstein metric~\cite{fournier2015rate}, OT map estimation~\cite{hutter2019minimax}, and Bayesian posterior 
perturbation analysis~\cite{tapio}. 

Numerical experiments show the better performance of the OT approach for examples designed to have multi-modal distributions, while it is noted that the raw 
computational time of the OT approach is higher, and 
for nonlinear filtering examples that admit unimodal posterior, such as Lorentz-96, the EnKF provides a fast and reasonable approximation. 
A computational feature of the OT method is that it provides the user with the flexibility to set the computational budget: without any training, OT algorithm implements EnKF; with additional budget (increasing training iterations and complexity of the neural net), the accuracy is increased, see Appendix~\ref{apdx:comp-time}. The computational efficiency of the OT approach can be improved by fine-tuning the neural network architectures, optimizing the hyper-parameters, and including an offline training stage for the first time step, which will be used as a warm-start for training at future time steps in the online implementation. This line of research is pursued in our recent work~\cite{al2024data}, where a new data-driven nonlinear filtering algorithm was introduced aimed at ergodic state and observation dynamics. The algorithm consists of offline and online stages: The offline stage is expensive to train and learns a static conditioning transport map; The online stage is computationally cheap and uses the learned conditioning map without any further training, providing a competitive computational time compared with traditional methods during online inference.

\section*{Acknowledgements}
The authors thank the reviewers for their constructive comments and feedback, which helped improve the paper. Mohammad Al-Jarrah and Amirhossein Taghvaei are supported by the National Science Foundation (NSF) award EPCN-2318977, Niyizhen Jin and  Bamdad Hosseini are supported by the NSF award DMS-2208535. Both supports are gratefully acknowledged.

\section*{Impact Statement}
This paper impacts the field of nonlinear filtering, data assimilation, and uncertainty quantification. It facilitates effective quantification and evaluation of uncertainty, allowing safe and reliable operation of autonomous systems.

\bibliography{TAC-OPT-FPF,references}
\bibliographystyle{icml2024}


\newpage
\appendix
\onecolumn

\section{Preliminaries}
We start by reviewing preliminaries on convex analysis and OT theory. 
\subsection{Preliminaries on convex analysis}
We recall definitions from convex analysis that will be useful in the subsequent sections. These definitions appear in most expositions on convex analysis, e.g. see~\cite{rockafellar1997convex}.
\begin{itemize}
    \item \textbf{Convex conjugate:}
    For a  function $\phi: \mathbb{R}^n \mapsto \cup \{\pm \infty\}$, its  convex conjugate, denoted as $\phi^*$, is defined according to
    \begin{equation}\label{eq:convex-conjugate}
        \phi^*(y) = \sup_{x \in \mathbb R^n}(x^Ty - \phi(x)).
    \end{equation}
    Moreover, $\phi$ is convex iff there exists a function $\psi$ such that $\phi=\psi^*$. 
    \item \textbf{Strong convexity:} 
    A function $\phi: \mathbb{R}^n \mapsto \mathbb{R} \cup \{\pm \infty\}$ is $\alpha$-strongly convex 
    if 
    for all $x,y,z \in \Re^n$ and $t \in [0,1]$:
    \[\phi(tx+(1-t)y) \leq t\phi(x)+(1-t)\phi(y)-\frac{\alpha}{2}t(1-t)\|x-y\|^2.\]
    \item \textbf{Smoothness:} A function $\phi: \mathbb{R}^n \mapsto \mathbb{R} \cup \{\pm \infty\}$
    is $\beta$-smooth if it is differentiable and for all $x,y \in \Re^n$:
    \[ \phi(y) \leq \phi(x) + \nabla \phi(x)^\top (y-x) + \frac{\beta}{2} \|x-y\|^2.\]
\end{itemize}

For a lower semi-continuous cost function $c:\Re^n \times \Re^n \to \Re$, we recall the following definition.  

\begin{itemize}
    \item 	{\bf inf-$c$ convolution:} For a function $f:\mathbb R^n \to \mathbb R \cup \{\pm \infty\}$, its inf-$c$ convolution, denoted by $f^c$, is defined according to 
\begin{equation}\label{eq:inf-c}
		f^c(y) = \inf_{x\in \mathbb R^n}\left[c(x,y)-f(x)\right]. 
	\end{equation}
	Moreover, $f$ is said to be $c$-concave if there exists a function $g$ such that $f=g^c$. 
\end{itemize}

\begin{remark}
    With a squared cost function $c(x,y)=\frac{1}{2}\|x-y\|_2^2$, the $c$-concave property is related to convexity. In particular, 
    \begin{center}
        $f$ is $c$-concave if and only if $\phi= \frac{1}{2}\|\cdot\|^2 - f$ is convex. 
    \end{center}Moreover, the inf-$c$ convolution and convex conjugate are related according to: \[\frac{1}{2}\|\cdot\|^2 - f^c =  (\frac{1}{2}\|\cdot\|^2 - f)^*.\]
\end{remark}

We also recall the definition of sub-differential and gradient of a convex function. 
\begin{itemize}
    \item 	{\bf Sub-differential and gradient:} The sub-differential of a convex function $\phi:\mathbb R^n \to \mathbb R \cup \{\pm \infty\}$ at $x\in \Re^n$, denoted by $\partial \phi(x)$, is a set containing all vectors $y \in \Re^n$ such that 
\begin{equation*}
   \phi(z) \geq \phi(x) + y^\top(z-x),\quad \forall z \in \Re^n.
\end{equation*}
Moreover, $\phi$ is differentiable at $x$, with derivative $\nabla \phi(x)$, iff $\partial \phi(x)=\{\nabla \phi(x)\}$. 
\end{itemize}
The definition of the convex conjugate and sub-differential imply the relationship~\citep[Thm. 23.5]{rockafellar1997convex} 
\begin{equation*}
    \phi(x) + \phi^*(y) = x^\top y\quad \Leftrightarrow \quad x \in \partial \phi^*(y) \quad \Leftrightarrow \quad y \in \partial \phi(x).
\end{equation*} 
In particular, when $\phi^*$ is differentiable at $y \in \Re^n$, we have the identity
\begin{equation*}
   \phi^*(y)  =  \nabla \phi^*(y)^\top y - \phi(\nabla \phi^*(y)). 
\end{equation*}
The relationship between convexity and $c$-concave, when $c$ is the squared cost, implies the similar identity when $f^c$ is differentiable at $y \in \Re^n$:
\begin{equation*}
   f^c(y)  =  c(y,y-\nabla f^c(y)) - f(y-\nabla f^c(y)). 
\end{equation*}

Finally, we recall the duality relationship between strong convexity and smoothness (e.g. see~\citep[Thm. 6]{kakade2009duality}): 
\begin{center}
    $\phi$ is $\alpha$-strongly convex iff $\phi^*$ is $\frac{1}{\alpha}$-smooth. 
\end{center}

\subsection{Preliminary on OT theory}
In this subsection, we review the problem formulation and key results  from OT theory. For a complete treatment of the subject, see~\cite{villani2009optimal}. 

Given two 
probability distributions $P$ and $Q$ on $\mathbb R^n$, the Monge optimal transportation problem aims to find a map $T\colon \mathbb R^n\to \mathbb R^n$ that solves the optimization problem
\begin{align}\label{eq:Monge}
	\inf_{T\in \mathcal T(P,Q)}\, \mathbb E_{Z\sim P}[c(Z,T(Z))],
\end{align}
where $\mathcal T(P,Q):=\{T\colon \Re^n \to \Re^n;\,T_{\#}P=Q\}$ is the set of all transport maps pushing forward $P$ to $Q$, and $c\colon \mathbb R^n \times \mathbb R^n \to \mathbb R$ is a lower semi-continuous cost function that is bounded from below.  The Monge problem is relaxed by replacing deterministic transport maps with stochastic couplings according to
\begin{align}\label{eq:Kantorovich-coupling}
	\inf_{\pi \in \Pi(Q,P)}\, \mathbb E_{(Z',Z)\sim \pi}[c(Z',Z)],
\end{align}
where $\Pi(Q,P)$ denotes the set of all joint distributions on $\mathbb R^n \times \mathbb R^n$ with marginals  $Q$ and $P$. This relaxation, due to Kantorovich, turns the Monge problem into a linear program, whose dual becomes
\begin{align}\label{eq:dual-Kantorovich}
	\sup_{(f,g)\in \text{Lip}_c}\,\{\mathcal J_\text{dual}(f,g):= \mathbb E_{Z'\sim Q}[f(Z')] + E_{Z\sim P}[g(Z)]\}
\end{align}
where $\text{Lip}_c$ is the set of pairs  of functions $(f,g)$, from $\mathbb R^n\to\mathbb R$, that 
satisfy the constraint 
\begin{align*}
f(z)+g(z')\leq c(z,z'),\quad\forall z,z'\in \mathcal \mathbb R^n.
\end{align*}

The well-known Brenier's result~\cite{brenier1991polar}  establishes the existence and uniqueness of the solution to the Monge problem by making a direct connection to the solution of the dual Kantorovich problem. 
\begin{theorem}[\cite{brenier1991polar}]Consider the Monge problem~(\ref{eq:Monge}) with cost $c(z,z')=\|z-z'\|^2/2$. Assume $P$ and $Q$ have finite second-order moments and $P$ is absolutely continuous with respect to the Lebesgue measure. Then, the Monge problem admits a unique minimizer $\bar T(z) =z-\nabla \bar f^c(z)$ where $\bar f$ is $c$-concave and the pair $(\bar f,\bar f^c)$ are the unique (up to an additive constant) maximizer of the dual Kantorovich problem~(\ref{eq:dual-Kantorovich}). 
 \label{thm:Brenier}
\end{theorem}

\subsection{Max-min formulation}\label{apdx:max-min}
In this subsection, we present the max-min formulation 
of the dual Kantorovich problem~(\ref{eq:dual-Kantorovich}), which is useful for numerical approximation of OT maps~\cite{makkuva2020optimal,rout2022generative}.  The presentation starts with a formal derivation followed by a proposition supported by rigorous arguments. 

The first step is to use the result of Theorem~\ref{thm:Brenier} to replace $g$ with $f^c$, in the dual problem~(\ref{eq:dual-Kantorovich}), resulting the following optimization problem
\begin{equation*} 
	\sup_{f \in c\text{-Concave}}\,\mathcal J_\text{dual}(f,f^c),
\end{equation*}
where $c$-Concave is the set of functions on $\mathbb R^n$ that are $c$-concave. Next, we use the definition of $f^c$ in~(\ref{eq:inf-c}) to conclude
\begin{align*}
 &\mathbb E_{Z\sim P}[f^c(Z)] = \mathbb E_{Z\sim P}[\inf_{z\in \mathbb R^n} c(z,Z)-f(z)]
 = \inf_{T\in \mathcal M(P)} \,\mathbb E_{Z\sim P}[ c(T(Z),Z)-f(T(Z))] 
\end{align*}
where  we assumed that  
\begin{equation*}
T(Z) = \argmin_{z \in \mathbb R^n} c(Z,z) - f(z),\quad \text{a.e.}
\end{equation*}
for a measurable map $T:\Re^n \to \Re^n$. This assumption is valid for $T = \text{Id} - \nabla f^c$ whenever $f^c$ is differentiable. 
Combining these two steps, the Kantorovich dual problem is expressed as the following max-min problem
\begin{align}
\sup_{f \in c\text{-concave}} \,\inf_{T\in \mathcal M(P)}\, \left\{\mathcal J_\text{max-min}(f,T):=\mathbb E_{Z'\sim Q}[f(Z')] + \mathbb E_{Z\sim P}[ c(T(Z),Z)-f(T(Z))] \right\}\label{eq:max-min-OT}
\end{align}
The following proposition, which is an extension of the existing result in~\citep[Thm. 3.3]{makkuva2020optimal}, connects the max-min problem to the original Kantorovich problem. 
\begin{proposition}\label{prop:max-min-OT}
Consider the max-min problem~(\ref{eq:max-min-OT}) under the assumptions of Thm.~\ref{thm:Brenier}. Let $(\bar f,\bar f^c)$ be the optimal pair for the dual Kantorovich problem~(\ref{eq:dual-Kantorovich}), and $\bar T = \text{Id} - \nabla f^c$ the OT map from $P$ to $Q$.  Then, $(\bar f, \bar T)$ is the unique pair, up to an additive constant for $\bar f$, that solves the problem~(\ref{eq:max-min-OT}), i.e. 
\begin{equation}\label{eq:OT-max-min-result}
     \sup_{f \in c\text{-concave}} \,\inf_{T\in \mathcal M(P)}\, \mathcal J_\text{max-min}(f,T) = \mathcal J_\text{max-min}(\bar f, \bar T).  
\end{equation}
and  if $\mathcal J_\text{max-min}(\bar f,\bar T) = \mathcal J_\text{max-min}(g,S)$, then $g= \bar f + \text{constant}$ and $S=\bar T$.  
\end{proposition}
\begin{proof} We present the proof in four steps.

{\bf Step 1:} According to the definition of inf-$c$ convolution~(\ref{eq:inf-c})
\begin{align*}
   f^c(Z) 
    \leq  c(T(Z),Z)  - f(T(Z))
   ,\quad \text{a.e.} 
\end{align*}
for all maps $T \in\mathcal M(P)$. Upon taking the expectation with respect to $Z\sim P$, adding $\mathbb E_{Z'\sim Q}[f(Z')]$ , and taking the inf over $T\in \mathcal M(P)$,
\begin{align*}
    \mathcal J_{\text{dual}}(f,f^c) \leq \inf_{T \in \mathcal M(P)}  \mathcal J_\text{max-min}(f,T). 
\end{align*}
{\bf Step 2:} For any $c$-concave function $f$,  introduce $f_\epsilon:=f-\frac{\epsilon}{2}\|\cdot\|^2$. This transformation ensures that its inf-$c$ convolution $f_\epsilon^c$ is differentiable a.e. ($f_\epsilon^c= \frac{1}{2}\|\cdot\|^2 - \phi_\epsilon^*$, with $\phi_\epsilon:=\frac{1}{2}\|\cdot\|^2 - f_\epsilon$, is $\frac{1}{\epsilon}$-smooth because  $\phi_\epsilon$ is $\epsilon$-strongly convex. This regularization process is known as  
{\it Moreau-Yosida
approximation} \cite{moreau1965proximite}). Because $f^c_\epsilon$ is differentiable, the minimum in the definition of $f^c_\epsilon(z)$ is achieved at a unique point $z-\nabla f^c_\epsilon(z)$, concluding the identity 
\begin{align*}
   f^c_\epsilon(Z) = c(Z-\nabla f^c_\epsilon(Z),Z)  - f_\epsilon(Z-\nabla f^c_\epsilon(Z)),\quad \text{a.e.}  
\end{align*}
Using the fact that $f_\epsilon(z)\leq f(z)$ on the right-hand-side of this identity, we arrive at the inequality
\begin{align*}
   f^c_\epsilon(Z)  \geq c(Z-\nabla f^c_\epsilon(Z),Z)  - f(Z-\nabla f^c_\epsilon(Z)) ,\quad \text{a.e.}  
\end{align*}
Taking the expectation over  $Z\sim P$ and adding $\mathbb E_{Z'\sim Q}[f(Z')]$ yields 
\begin{align*}
    \mathcal J_{\text{dual}}(f,f^c_\epsilon) \geq \mathcal J_\text{max-min}(f,\text{Id}-\nabla f^c_\epsilon)  \geq \inf_{T \in \mathcal M(P)}  \mathcal J_\text{max-min}(f,T), 
\end{align*}
where the last inequality follows  because $\text{Id}- \nabla f^c_\epsilon \in \mathcal M(P)$. Finally, using $\mathbb E_{Z'\sim Q}[f(Z')] = \mathbb E_{Z'\sim Q}[f_\epsilon(Z')] +\epsilon M  $, where $M := \mathbb E_{Z'\sim Q}[\frac{1}{2}\|Z'\|^2]<\infty$, concludes 
\begin{align*}
    \mathcal J_{\text{dual}}(f_\epsilon,f^c_\epsilon) + \epsilon M   \geq \inf_{T \in \mathcal M(P)}  \mathcal J_\text{max-min}(f,T). 
\end{align*}
{\bf Step 3:} Combining the results of step 1 and 2 gives 
\begin{align*}
     \mathcal J_{\text{dual}}(f,f^c) \leq \inf_{T \in \mathcal M(P)}  \mathcal J_\text{max-min}(f,T) \leq  \mathcal J_{\text{dual}}(f_\epsilon,f^c_\epsilon) + \epsilon M
\end{align*}
Taking the sup over $f \in c$-Concave and the fact that $(\bar f,\bar f^c)$ are the maximizer of the dual problem, concludes 
\begin{align*}
     \mathcal J_{\text{dual}}(\bar f,\bar f^c) \leq  \sup_{f \in c\text{-concave}} \,\inf_{T \in \mathcal M(P)}  \mathcal J_\text{max-min}(f,T) \leq  \mathcal J_{\text{dual}}(\bar f,\bar f^c) + \epsilon M
\end{align*}
Taking the limit as $\epsilon \to 0$ concludes the identity
\begin{align*}
    \sup_{f \in c\text{-concave}} \,\inf_{T \in \mathcal M(P)}  \mathcal J_\text{max-min}(f,T) =   \mathcal J_{\text{dual}}(\bar f,\bar f^c)
\end{align*}
The proof of~(\ref{eq:OT-max-min-result}) is concluded by noting that $\mathcal J_{\text{dual}}(\bar f,\bar f^c) = J_\text{max-min}(\bar f,\bar T)$ for $\bar T = \text{Id}-\nabla \bar f^c$ because $\bar f^c$ is differentiable a.e.. 

 {\bf Step 4:} To prove uniqueness, let $(g,S)$ be another optimal pair for $J_\text{max-min}$. Therefore, 
 \begin{equation}\label{eq:opt-condition-uniqueness-proof}
 \mathcal J_\text{max-min}(g,S) = \sup_{f \in c\text{-concave}} \,\inf_{T \in \mathcal M(P)}  \mathcal J_\text{max-min}(f,T) = \mathcal J_\text{dual}(\bar f,\bar f^c).    
 \end{equation}
 On the other hand, from the inequalities obtained in the previous steps
\begin{align*}
    \mathcal J_\text{dual}(g,g^c) \leq    \mathcal J_\text{max-min}(g,S)  \leq  \mathcal J_\text{dual}(g,g_\epsilon^c)
\end{align*}
Taking the limit as $\epsilon \to 0$ and the application of monotone convergence theorem implies $\mathcal J_\text{dual}(g,g^c) =\mathcal J_\text{max-min}(g,S)$, which with the optimality condition~(\ref{eq:opt-condition-uniqueness-proof}), concludes $\mathcal J_\text{dual}(g,g^c) =\mathcal J_\text{dual}(\bar f,\bar f^c)$. The uniqueness of the solution $(\bar f,\bar f^c)$ to the dual problem~(\ref{eq:dual-Kantorovich}) and the differentiability of $\bar f^c$ concludes that $g=\bar f+\text{const.}$ and $S=\text{Id}-\nabla \bar f^c=\bar T$. 
\end{proof}

 \section{Proofs of the theoretical results}

\subsection{Justification for the consistency condition~(\ref{eq:T-constraint-joint})}\label{apdx:consistency}
The consistency condition~(\ref{eq:T-constraint-joint}) implies that
\begin{equation*}
    \mathbb E[F(T(\overline X,Y),Y)] =   \mathbb E[F(X,Y)]
\end{equation*}
for all measurable and bounded functions $F:\mathbb R^n \times \mathbb R^m \to \mathbb R$. In particular, with $F(x,y)=f(x)g(y)$, we have 
\begin{equation*}
    \mathbb E[f(T(\overline X,Y))g(Y)] =   \mathbb E[f(X)g(Y)]
\end{equation*}
for all measurable and bounded functions $f:\Re^n \to \Re$ and $g:\Re^m \to \Re$. The definition of the conditional expectation implies 
\begin{equation*}
    \mathbb E[f(T(\overline X,Y))|Y] =   \mathbb E[f(X)|Y].
\end{equation*}
Finally, the independence of $\bar X$ and $Y$ yields 
\begin{equation*}
    \mathbb E[f(T(\overline X,y)] =   \mathbb E[f(X)|Y=y],\quad \text{a.e. } y
\end{equation*}
for all measurable and bounded functions $f$, 
concluding $T(\cdot,y)_{\#}P_X = P_{X|Y}(\cdot|y)$ a.e. $y$. 
\subsection{Proof of Proposition \ref{prop:consistency}}
\label{proof:consistency}
{Express the objective function in~(\ref{eq:new_loss}) as
\begin{align*}
    J(f,T) = \mathbb E[\mathcal J^Y_{\text{max-min}}(f(\cdot,Y),T(\cdot,Y))]
\end{align*}
where 
\begin{equation*}
  \mathcal J^y_{\text{max-min}}(g,S) = \mathbb E_{X \sim P_{X|Y=y}}[g(X)] + \mathbb E_{\overline X \sim P_X}[c(S(\overline X),\overline X)-g(S(\overline X)]
\end{equation*}
is the max-min objective function in~(\ref{eq:max-min-OT}) with the marginal distributions $Q=P_{X|Y=y}$ and $P=P_X$.  Then, according to the Prop. \ref{prop:max-min-OT}, 
\begin{equation*}
    \sup_{g \in c\text{-Concave}}\,\inf_{S \in \mathcal M(P_X)}\, \mathcal J^y_{\text{max-min}}(g,S) =  \mathcal J^y_{\text{max-min}}(\overline g_y, \overline S_y)
\end{equation*}
for a unique pair $(\overline g_y,\overline S_y)$ where $\overline S_y = \text{Id}-\nabla \overline g^c_y$ is the OT map from $P_X$ to $P_{X|Y=y}$. Theorem 2.3 of \cite{carlier2016vector}\footnote{The dual function $f$ is related to the dual function $\phi$ in~\cite{carlier2016vector} with the transformation $\phi(x,y)=\frac{1}{2}\|x\|^2-f(x,y)$.} verifies that there is a measurable concatenations of the functions  $\overline g_y$ to construct the pair $(\overline f,\overline T)$ according to
\begin{align*}
    \overline f(x,y) = \overline g_y(x),\quad \overline T(x,y) = x -  \nabla \overline g^c_y(x). 
\end{align*}
The concatenated functions solve the optimization problem~(\ref{eq:new_loss}) and $T(\cdot,y)$ serves the OT map from $P_X$ to $P_{X|Y=y}$, i.e. $\mathcal{B}_y(\pi) = \overline{T}(\cdot,y)_{\#}\pi$.}

\subsection{Proof of Proposition \ref{prop:map estimation error bound}}
\label{proof:map estimation error bound}
We present the proof under the change of variable $\phi(x,y)=\frac{1}{2}\|x\|^2 - f(x,y)$. Under this change of variable, the condition $f \in c\text{-Concave}_x$ translates to $\phi \in \text{CVX}_x$, i.e. $\phi$ is convex in the first variable $x$, and the objective function in~(\ref{eq:new_loss}) becomes 
\begin{equation}\label{eq:J-L-relationship}
    J(f,T) = \mathbb E_{X\sim P_X}[\|X\|^2] - L(\phi,T)
\end{equation}
where 
\begin{align*}
    L(\phi,T) :=  \mathbb{E}_{(X,Y)\sim P_{XY}}[\phi(X,Y)] + \mathbb{E}_{(\overline X,Y) \sim P_X \otimes P_Y}[\overline X^\top T(\overline X,Y) - \phi(T(\overline X,Y),Y)]. 
\end{align*}
Next, we decompose the optimization gap $\epsilon(f,T)= \tilde \epsilon(\phi) + \tilde \delta(\phi,T)$ where,
\begin{equation*}    
\begin{aligned}
    & J(\overline f,\overline T) - \min_S J(f,S) = 
    \max_{S} L(\phi,S) - L(\overline \phi, \overline T)
    =: \tilde \epsilon(\phi)
    \\
    & J(f,T) - \min_S J(f,S)=\max_{S}L(\phi,S) - L(\phi,T) =: \tilde \delta (\phi,T) 
\end{aligned}
\end{equation*}
and $\phi(x, y) = \frac{1}{2} \| x\|^2 - f(x,y)$ and similarly $\overline \phi(x,y)= \frac{1}{2}\|x\|^2 - \overline f(x,y)$ and we used~(\ref{eq:J-L-relationship}). 

The proposition assumption states that $\phi(x,y)=\frac{1}{2}\|x\|^2 - f(x,y)$ is $\alpha$-strongly convex in $x$. The duality between strong convexity and smoothness implies that $\phi^*(x,y)$ is $\frac{1}{\alpha}$-smooth in $x$, where $\phi^*(\cdot,y)$ is the convex conjugate  of $\phi(\cdot,y)$ for all $y$. 
As a result,
\begin{equation*}
    \phi^*(w, y) \; \leq \; \phi^*(x, y)+ \nabla_x \phi^*(x, y)^\top(w-x)  +\frac{1}{2\alpha}\|w-x\|^2 \coloneqq h(w, y).
\end{equation*}
The inequality 
$\phi^*(w, y) \; \leq \; h(w, y)$ implies $\phi(w,y) \; \geq \; h^*(w, y)$. Consequently, we have the inequality:
\begin{equation} \label{eq:convexity ineq}
       \phi(w, y) \; \geq \; -\phi^*(x, y) 
       + w^\top x + \frac{\alpha}{2} \|w - \nabla_x \phi^*(x, y)\|^2. 
\end{equation}
This inequality plays a crucial role in deriving lower bounds for the quantities 
$\tilde{\delta}(\phi, T)$
and $\tilde{\epsilon}(\phi)$. 
We first establish a bound for $\tilde \delta(\phi,T)$ using (\ref{eq:convexity ineq}):
\begin{align*}
    \tilde{\delta}(\phi,T) & = \max_{S} L(\phi,S) - L( \phi,  T) \\ &= L(\phi,\nabla_x
    \phi^*) - L( \phi,  T) \\&= \mathbb{E}_{(\overline X,Y) \sim P_X \otimes P_Y}[\phi^*(\overline{X},Y) - \overline{X}^T T(\overline{X},Y) + \phi(T(\overline{X},Y),Y)]
    \\
    & \geq \frac{\alpha}{2} \mathbb{E}_{(\overline X,Y) \sim P_X \otimes P_Y}[\|T(\overline{X},Y) - \nabla_x \phi^*({\overline{X},Y})\|^2]
\end{align*}
where the second and third equality follows from the fact that 
\begin{equation}\label{eq:conjugate-optimality}
\sup_{z\in \Re^n}\, \{z^\top x - \phi(z,y)\} = \nabla \phi_x^*(x,y)^\top x - \phi(\nabla_x\phi^*(x,y),y)=\phi^*(x,y).
\end{equation}

Next, we obtain the bound for $\tilde{\epsilon}(\phi)$ as follows: 
    \begin{align*}
        \tilde{\epsilon}(\phi) &\overset{(1)}{=}     \max_{S} L(\phi,S) - L(\overline \phi, \overline T)
        \\
        &\overset{(2)}{=}   L(\phi,\nabla \phi^*_x) - L(\overline \phi, \overline T)
        \\
        &\overset{(3)}{=}  \mathbb{E}_{( X,Y)\sim P_{XY}}[\phi(X,Y)-\overline{\phi}(X,Y)] +\mathbb{E}_{( \overline X,Y)\sim P_X \otimes P_Y}[\phi^*(\overline{X},Y)- \overline X^\top\overline T(\overline X,Y) + \overline \phi(\overline T(\overline X,Y),Y)]\\
           &\overset{(4)}{=}   \mathbb{E}_{( \overline X,Y)\sim P_X \otimes P_Y}[\phi(\overline{T}(\overline{X},Y),Y)+\phi^*(\overline{X},Y) -\overline{X}^\top\overline{T}(\overline{X},Y)]       \\
        &\overset{(5)}{\geq}   \frac{\alpha}{2}\mathbb{E}_{( \overline X,Y)\sim P_X \otimes P_Y}[\|\overline{T}(\overline{X},Y)-\nabla_x \phi^*(\overline{X},Y)\|^2]
     \end{align*}
      where we used~(\ref{eq:conjugate-optimality}) in the second and third steps,  $(\overline T(\overline X,Y),Y) \sim P_{XY}$ in the fourth step, and~(\ref{eq:convexity ineq}) in the fifth step. 
Finally, we use the inequality $(a+b)^2\leq 2a^2 + 2b^2$ to obtain:
\begin{align*}
        \mathbb{E}_{( \overline X,Y)\sim P_X \otimes P_Y}[\|T(\overline{X},Y) - \overline{T}(\overline{X},Y)\|^2] &\leq 2\mathbb{E}_{( \overline X,Y)\sim P_X \otimes P_Y}[\|T(\overline{X},Y) - \nabla_x \phi^*(\overline{X},Y)\|^2] \\
        & + \|\nabla_x \phi^*(\overline{X},Y) - \overline{T}(\overline{X},Y)\|^2]
        \leq  \frac{4}{\alpha}(\tilde{\delta}(\phi,T)+\tilde{\epsilon}(\phi)) =\frac{4}{\alpha}\epsilon(f,T)
    \end{align*}

\subsection{Proof of the Prop.~\ref{prop:error-analysis}}\label{apdx:error-analysis}
\subsubsection{Proof of the lower-bound~(\ref{eq:SIR-bound})}
The bound is based on the application of the central limit theorem to $\int g d \pi^{(SIR)} - \int g d (\mathcal B_Y(\pi))$ for any uniformly bounded function $g$ (e.g. see~\citep[Thm. 9.1.8]{cappe2009inference}). In particular, the definition of $\pi^{(SIR)}$ and $\mathcal B_y(\pi)$ imply
\begin{align*}
    \int g d \pi^{(SIR)} &= \frac{\frac{1}{N}\sum_{i=1}^N h(Y|X^i)g(X^i)}{\frac{1}{N}\sum_{i=1}^Nh(Y|X^i)} =  \frac{\frac{1}{N}\sum_{i=1}^N \bar h(Y|X^i)g(X^i)}{\frac{1}{N}\sum_{i=1}^N \bar h(Y|X^i)} \\
    \int g d \mathcal B_Y(\pi) &=  \mathbb E[g(X)|Y]
\end{align*}
where  $\bar h(y|x) = h(y|x)/\int h(y|x')d \pi(x')$ is the normalized likelihood.
Then, subtracting the two expressions and multiplication by $\sqrt{N}$ yields  
\begin{align*}
    \sqrt{N}\left(\int g d \pi^{(SIR)} - \int g d \mathcal B_Y(\pi)\right)&= \frac{\frac{1}{\sqrt{N}}\sum_{i=1}^N \bar h(Y|X^i)\left(g(X^i) - \mathbb E[g(X)|Y]\right)}{\frac{1}{N}\sum_{i=1}^N\bar h(Y|X^i)}. 
\end{align*}
Application of the central limit theorem to the numerator, and law of the large numbers to the denominator, concludes the  convergence 
\begin{align*}
    \sqrt{N}\left(\int g d \pi^{(SIR)} - \int g d \mathcal B_Y(\pi)\right)&\longrightarrow Z \sim N(0, V_h(g)))
\end{align*}
where 
\begin{equation*}
    V_h(g) = \mathbb E \left[\bar h(Y|\bar X)^2 \left(g( \bar X) -\mathbb E[g(X)|Y]\right)^2\right]
\end{equation*}
and $\bar X \sim \pi$ is an independent copy of $X$. 
As a result,
\begin{align*}
    \lim_{N \to \infty}\sqrt{N}\,\mathbb E\left[\left(\mathbb \int g d \pi^{(SIR)} - \int g d \mathcal B_Y(\pi)\right)^2\right]^{\frac{1}{2}}& = \sqrt{V_h(g)}.
\end{align*}
Finally, upon using the definition of the metric, 
\begin{align*}
    \liminf_{N\to \infty} \sqrt{N}d(\pi^{(SIR)},\mathcal B_Y(\pi)) &=   \liminf_{N \to \infty}\,\sup_{g\in \mathcal G}\sqrt{N}\,\mathbb E\left[\left(\mathbb \int g d \pi^{(SIR)} - \int g d \mathcal B_y(\pi)\right)^2\right]^{\frac{1}{2}} \\&\geq \sup_{g \in \mathcal G} \, \liminf_{N\to \infty} \sqrt{N}E\left[\left(\mathbb \int g d \pi^{(SIR)} - \int g d \mathcal B_y(\pi)\right)^2\right]^{\frac{1}{2}} \\&=\sup_{g \in \mathcal G} \sqrt{V_h(g)}.
\end{align*}
\subsubsection{Proof of the upper-bound~(\ref{eq:OT-bound})}
According to Prop.~\ref{prop:consistency} and the definition (\ref{eq:pi-OT}), we have 
$\mathcal B_Y(\pi) = \overline T(\cdot,Y)_{\#} \pi$ and $\pi^{(OT)}= \widehat T(\cdot,Y)_{\#} \pi$, respectively. As a result, 
\begin{align*}
    d(\pi^{(OT)},\mathcal B_Y(\pi)) &= d(\widehat T(\cdot,Y)_{\#} \pi, \overline T(\cdot,Y)_{\#} \pi)\\&= \sup_{g\in \mathcal G} \mathbb E\left[ \left(\int g d(\widehat T(\cdot,Y)_{\#} \pi) - \int g d(\overline T(\cdot,Y)_{\#} \pi)\right)^2  \right]^{\frac{1}{2}}
\end{align*}
By expressing
\begin{align*}
    \int g d(\widehat T(\cdot,Y)_{\#} \pi) &=\mathbb E[g(\widehat T(\bar X,Y)|Y]\quad \text{and}\\
    \int g d(\overline T(\cdot,Y)_{\#} \pi) &=\mathbb E[g(\overline T(\bar X,Y)|Y],
\end{align*}
where $\bar X \sim \pi$ is an independent copy of $X$, we arrive at the identity 
\begin{align*}
    d(\pi^{(OT)},\mathcal B_Y(\pi)) 
    &= \sup_{g\in \mathcal G} \mathbb E\left[ \left(\mathbb E[g(\widehat T(\bar X,Y)|Y] - \mathbb E[g (\overline T(\bar X,Y)|Y]\right)^2  \right]^{\frac{1}{2}}.
\end{align*}
Upon application of the Jensen's inequality and the fact that $g$ is Lipschitz with constant $1$, 
\begin{align*}
    d(\pi^{(OT)},\mathcal B_Y(\pi)) 
    &\leq  \sup_{g\in \mathcal G} \mathbb E\left[ \left(g(\widehat T(\bar X,Y)-g (\overline T(\bar X,Y)\right)^2  \right]^{\frac{1}{2}}\\
    & \leq   \mathbb E\left[ \|\widehat T(\bar X,Y)-\overline T(\bar X,Y)\|^2  \right]^{\frac{1}{2}}. 
\end{align*}
Finally, according to Prop.~\ref{prop:map estimation error bound}, the right-hand-side of this inequality is bounded by $\sqrt{\frac{4}{\alpha}\epsilon(\widehat f,\widehat T)}$, concluding the upper-bound~(\ref{eq:OT-bound}).

\section{Numerical details and additional results}\label{apdx:numerics}

Our numerical results involve simulation of three filtering algorithms:  (1)  Ensemble Kalman filter (EnKF)~\cite{evensen2003ensemble,calvello2022ensemble}, presented in Algorithm \ref{alg:enkf}; (2) Sequential importance resampling (SIR)~\cite{doucet09}, presented in Algorithm \ref{alg:sir}; (3) OT particle filter (OTPF),  presented in Algorithm~\ref{alg:otpf}. 

In order to implement the OTPF, we used the ResNet neural network architecture to represent both $f$ and $T$ as they appear in Fig.~\ref{tikz:static_struc}. We used the Adam optimizer with different learning rates for each example to solve the max-min problem~(\ref{eq:emperical_loss})  with inner-loop iteration number always equal to $10$. The number of particles, $N$, and the initial particles, is the same for all three algorithms. The EnKF algorithm is simulated with the additional regularization $\Gamma = \sigma_w^2 I$ where $\sigma_w$ is the noise level in the observation signal, and $I$ is the identity matrix.
The numerical results are produced using the following  two machines:
\begin{enumerate}
    \item MACBOOK M1 Pro with 8‑core CPU, 14‑core GPU, and 16GB unified memory
    \item MAC STUDIO M2 Max with 12‑core CPU, 30‑core GPU, and 64GB unified memory
\end{enumerate}

\begin{algorithm}[h]
\caption{Ensemble Kalman filter (EnKF)}

\begin{algorithmic}
\STATE \textbf{Input:} Initial particles $\{{X}_0^{i}\}_{i=1}^N\sim \pi_0$, observation signal $\{Y_{t}\}_{t=1}^{t_f}$, $\Gamma \succ 0$\\ dynamic model $a(x\mid x')$, observation model $h(y\mid x)$.
\FOR{$t = 1$ to $t_f$}
    \STATE \textbf{Propagation:} 
    \FOR{$i = 1$ to $N$} 
        \STATE $X_{t \mid t-1}^i \sim a(.\mid X_{t-1}^{i})$ 
        \STATE $
            {{Y}}_{t}^{i} \sim h(.\mid X_{t \mid t-1}^{i})$
        \ENDFOR 
                    \STATE \textbf{Conditioning:} 
    \STATE  $ \widehat{{X}}_{t} = \frac{1}{N} \sum_{i=1}^N X_{t \mid t-1}^i$
    \STATE  $\widehat{{Y}}_{t} = \frac{1}{N} \sum_{i=1}^N {{Y}}_{t}^{i}$
    \STATE ${C}^{xy}_{t} = \frac{1}{N} \sum_{i=1}^N (X_{t \mid t-1}^i - \widehat{{X}}_{t}) \otimes ({{Y}}_{t}^{i} - \widehat{{Y}}_{t})$
    \STATE $ {C}^{yy}_{t} = \frac{1}{N} \sum_{i=1}^N ({{Y}}_{t}^{i} - \widehat{{Y}}_{t}) \otimes ({{Y}}_{t}^{i} - \widehat{{Y}}_{t})$
    \STATE $K_{t} = {C}^{xy}_{t} ({C}^{yy}_{t} + \Gamma)^{-1}$

                      \FOR{$i = 1$ to $N$} 
        \STATE ${X}_{t}^{i} = X_{t \mid t-1}^i + K_{t}({Y}_{t} - {{Y}}_{t}^{i})$. 
     \ENDFOR
\ENDFOR
\STATE \textbf{Output:} Particles $\{{X}_t^{i}\}_{i=1}^N$ for $t = 0,\ldots, t_f$.
\end{algorithmic}
\label{alg:enkf}
\end{algorithm}

\begin{algorithm}[h] 
\caption{Sequential importance resampling (SIR)}
\begin{algorithmic}
\STATE \textbf{Input:} Initial particles $\{{X}_0^{i}\}_{i=1}^N\sim \pi_0$, observation signal $\{Y_{t}\}_{t=1}^{t_f}$,\\ dynamic model $a(x\mid x')$, observation model $h(y\mid x)$.
\FOR{$t = 1$ to $t_f$}
    \STATE \textbf{Propagation:} 
    \FOR{$i = 1$ to $N$}
            \STATE$X_{t \mid t-1}^i \sim a(.\mid {X}_{t-1}^{i}).$
        \ENDFOR
    \STATE \textbf{Conditioning:}
    \STATE $w_{t}^{i} = \frac{h(Y_{t}|X_{t \mid t-1}^i)}{\sum_{i=1}^N h(Y_{t}|X_{t \mid t-1}^i)}$
    \STATE $X_{t}^{i} \sim \sum_{i=1}^N w_{t}^{i}\delta_{X_{t \mid t-1}^i}^{i}$
\ENDFOR
\STATE \textbf{Output:} Particles $\{{X}_t^{i}\}_{i=1}^N$ for $t = 0,\ldots, t_f$.
\end{algorithmic}
\label{alg:sir}
\end{algorithm}

\begin{algorithm}[h]
\caption{OT particle filter (OTPF) 
}
\begin{algorithmic}
\STATE \textbf{Input:} Initial particles $\{{X}_0^{i}\}_{i=1}^N\sim \pi_0$, observation signal $\{Y_{t}\}_{t=1}^{t_f}$,\\ dynamic and observation models $a(x\mid x'),h(y\mid x)$\\ batch size $bs$, optimizer and learning rates for $f,T$, inner loop iterations $K^{\text{(in)}}$, outer iteration sequence $K^{\text{(out)}}_t$.
\STATE \textbf{Initialize:} initialize neural net $f,T$ according to the architecture in Fig.~\ref{tikz:static_struc}, and their weights $\theta_f,\theta_T$.
\FOR{$t = 1$ to $t_f$}
\STATE \textbf{Propagation:} 
\FOR{$i = 1$ to $N$} 
        \STATE $X_{t \mid t-1}^i \sim a(.\mid X_{t-1}^{i})$ 
        \STATE $
            {{Y}}_{t}^{i} \sim h(.\mid X_{t \mid t-1}^{i})$
        \ENDFOR 
\STATE \textbf{Conditioning:}
\STATE  Create a random permutation $\{\sigma_i\}_{i=1}^N$ 
\FOR{$k=1$ to $K^{\text{(out)}}_t$}
\STATE sample $bs$ particles from $(X_{t \mid t-1}^i,X_{t \mid t-1}^{\sigma_i},Y_{t}^i)$ 
\FOR{$j=1$ to $K^{\text{(in)}}$}
\STATE  Update $\theta_T$ to minimize \(\frac{1}{bs}\sum_{i=1}^{bs} \big[\frac{1}{2}\|T(  X_{t \mid t-1}^{\sigma_i},Y_{t}^i)- X_{t \mid t-1}^{\sigma_i}\|^2 - f(T(X_{t \mid t-1}^{\sigma_i},Y_{t}^i),Y_{t}^i)\big]\)
\ENDFOR
\STATE Update $\theta_f$ to minimize \( \frac{1}{bs}\sum_{i=1}^{bs} \big[-f( X_{t \mid t-1}^i,Y_{t}^i) + f(T(X_{t \mid t-1}^{\sigma_i},Y_{t}^i),Y_{t}^i)\big]\)
\ENDFOR
 \FOR{$i = 1$ to $N$} 
        \STATE ${{X}}_{t}^{i} = T(X_{t \mid t-1}^i,Y_{t})$. 
     \ENDFOR
\ENDFOR
\STATE \textbf{Output:} Particles $\{{X}_t^{i}\}_{i=1}^N$ for $t = 0,\ldots, t_f$.
\end{algorithmic}
\label{alg:otpf}
\end{algorithm}

\subsection{Computational time}\label{apdx:comp-time}
The raw computational time of the algorithms, across all of the main examples, is presented in Table~\ref{table:running-time}. 
It is observed that the computational time of the OT method is higher than the other two algorithms. 
These results are included for comprehensive understanding without an intensive focus on optimizing computational time, which may be achieved by careful selection of architectures for $f$ and $T$, tuning the algorithm hyper-parameters, and including an offline training stage that provides a warm start for training in the online implementation. This is subject of ongoing research. In our preliminary efforts in this direction, we were able to reduce the computational time of the OT method for the Bimodal dynamic example to 2.9151s  with only 1.3\% loss in accuracy. This is achieved by performing the training for the initial step offline for $10^4$ iterations. The resulting $f$ and $T$ are then used as warm start  for the online implementation of the filter, with training iterations reduced to $8$.

In order to further study the computational efficiency, we numerically evaluated the relationship/trade-off between computational time and filtering accuracy, for the  bimodal dynamic example in section \ref{sec:dynamic-bimodal}, for the OT and SIR methods. In the implementation of the OT method, we start with the outer-loop iteration number $K_{(1)}^{(out)}= 2^{12}=4096$, as it appears in algorithm \ref{alg:otpf}, and decrease it according to the rule $K_{(t+1)}^{(out)} = K_{(t)}^{(out)}/2$ until we reach a predefined final iteration number $K_s$. 

The result is presented in Fig.~\ref{fig:computaional_time}. The filtering accuracy is evaluated in terms of the MMD distance with respect to the exact posterior distribution. Panel (a) shows the relationship between the MMD distance and computational time, as the problem dimension $n$ increases.  Panel (b) shows the same relationship, but as the number of particles $N$ change.  The result for the OT method is presented for three different values of the final iteration number $K_s$.  
The result of panel (a) shows that, although the absolute computational time of OT is larger than SIR, it scales better with the problem dimension. The results of panel (b) shows that the accuracy of the OT can be improved by increasing the number of particles, without significant increase in the computational time, while the SIR method should use a much larger number of particles to reach the same accuracy.

\begin{table}[h]
\caption{Computational time in seconds for EnKF, SIR and OTPF.}
\label{table:running-time}
\vskip 0.05in
\begin{center}
\begin{small}
\begin{sc}
\begin{tabular}{lcccr}
\toprule
Example & EnKF & SIR & OTPF \\
\midrule
Bimodel dynamic example    & 0.0095 & 0.0108  & 34.4555 \\
Lorenz 63 example (200 steps) & 0.0542 & 0.0605 & 50.9799 \\
Lorenz 96 example (200 steps)    & 0.3029 & 0.3726 & 79.7705 \\
Static image in-painting on MNIST  & 0.4577 & 0.3880 & 529.8863 \\
Dynamic image in-painting on MNIST & 0.5373 & 0.3128 & 728.8051\\
\bottomrule
\end{tabular}
\end{sc}
\end{small}
\end{center}
\vskip -0.1in
\end{table}

\begin{figure*}[h]
    \centering
\begin{subfigure}{0.45\textwidth}
\label{fig:computaional_time_vs_n}
    \centering
    \includegraphics[width=1.0\textwidth,trim={22 0 70 60},clip]{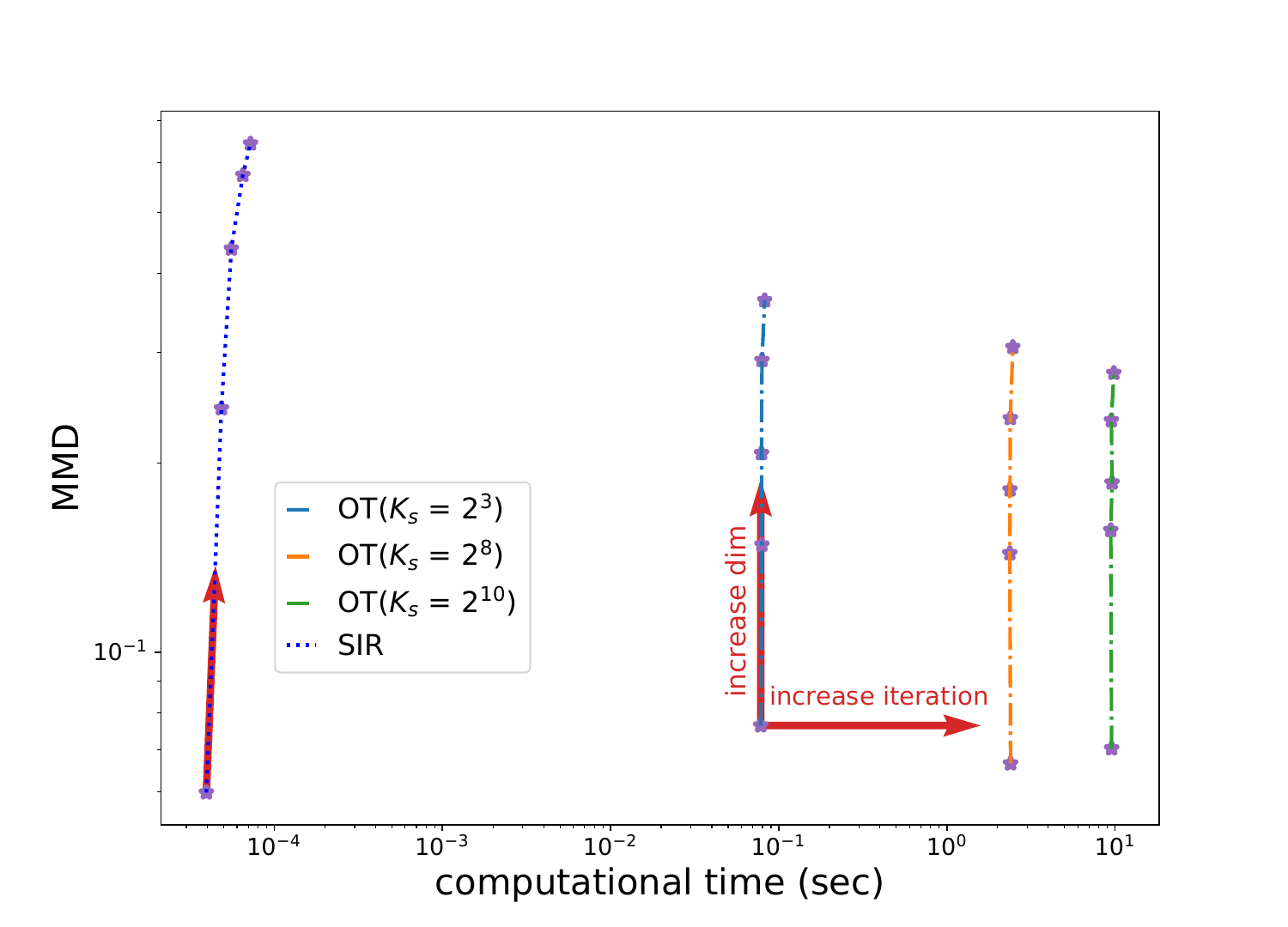}
    \caption{MMD vs computational time as dimension $n$ varies.}
\end{subfigure}
\hspace{1.0em}
\begin{subfigure}{0.45\textwidth}
\label{fig:computaional_time_vs_N}
    \centering
     \includegraphics[width=1.0\textwidth,trim={22 0 70 60},clip]{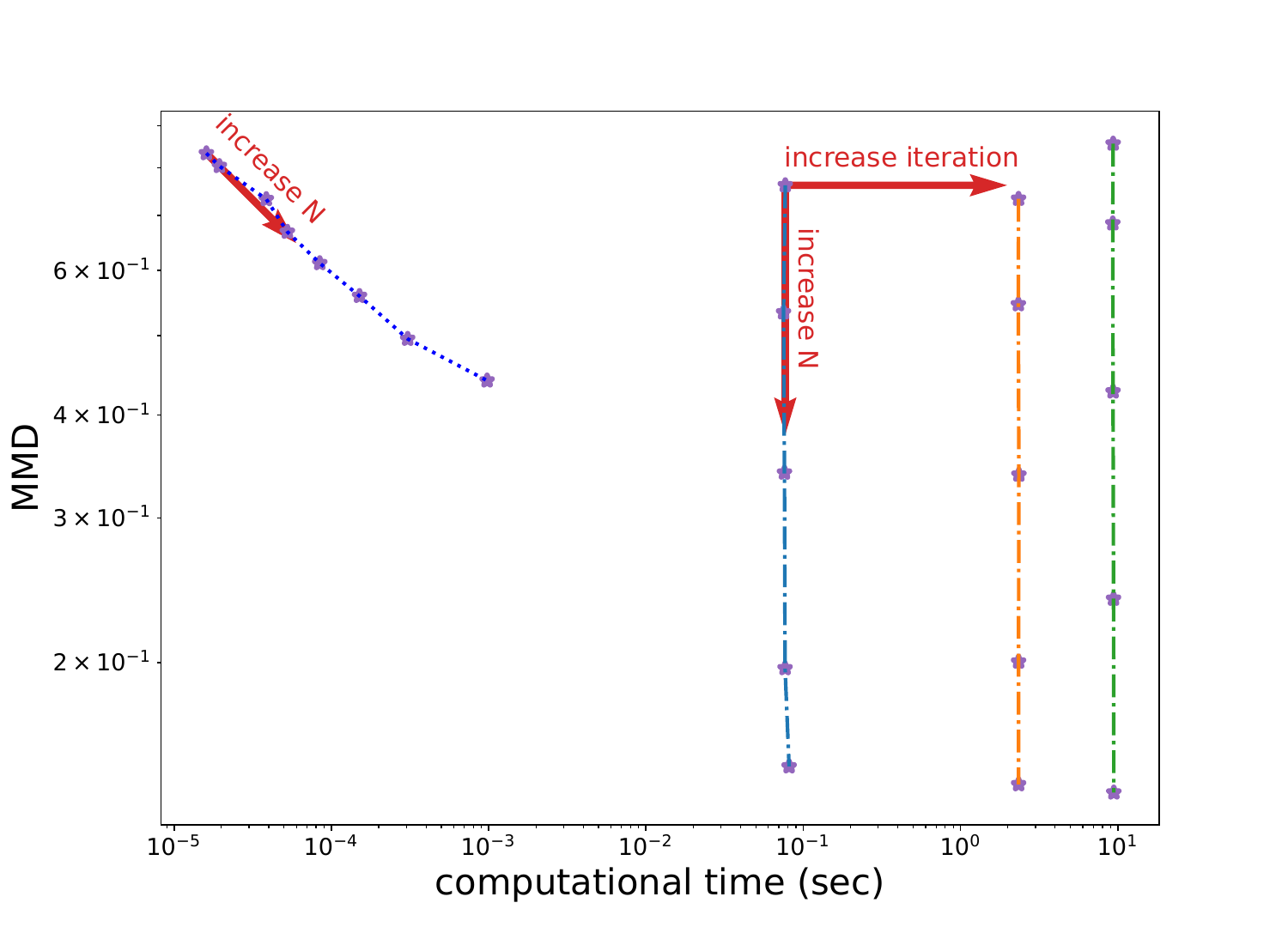} 
    \caption{MMD vs computational time as  particle size $N$ varies.}
\end{subfigure}
     \caption{Evaluation of the computational time of the OT and SIR algorithms, for the example in Sec.~\ref{sec:dynamic-bimodal}. The left panel shows the relationship between the MMD distance  and computational time, as the problem dimension $n$ increases.  The computational time of the OT approach is evaluated for three different 
 values of final number of iteration $K_s$, as explained in Sec.~\ref{apdx:comp-time}. The right panel shows the same relationship, but as the number of particles $N$ change. }
    \label{fig:computaional_time}
\end{figure*}

\subsection{Additional static example}\label{sec:app_static}
Consider the following observation model:
\begin{align}\label{example:bimodal}
    Y = X + \sigma_w W, \qquad X \sim \frac{1}{2} N(-1,\sigma^2 I_2) + \frac{1}{2}N(+1,\sigma^2 I_2)
\end{align}
where $W\sim N(0,I_2)$ and $\sigma_w=0.4$. In this example, the prior distribution is bimodal, while the observation model is linear. The numerical result for a varying number of particles is presented in Fig.~\ref{fig:bimodal}, where it is observed that the OT algorithm is also capable of pushing a bimodal to an unimodal distribution. At the same time, EnKF underperforms compared to the other two methods due to its Gaussian approximation of the prior.

In this example, and also the examples in Sec.~\ref{sec:Static_Example}, we did not use the EnKF layer in modeling $T$. We used a ResNet with a single-block of size $32$ and ReLU-activation function. The learning rate of the Adam optimizer is set to $10^{-3}$, with the total number of iterations $2\times 10^4$, and batch-size is $128$.

\begin{figure}[h]
    \centering
    \includegraphics[width=0.7\textwidth,trim={70 20 90 40},clip]{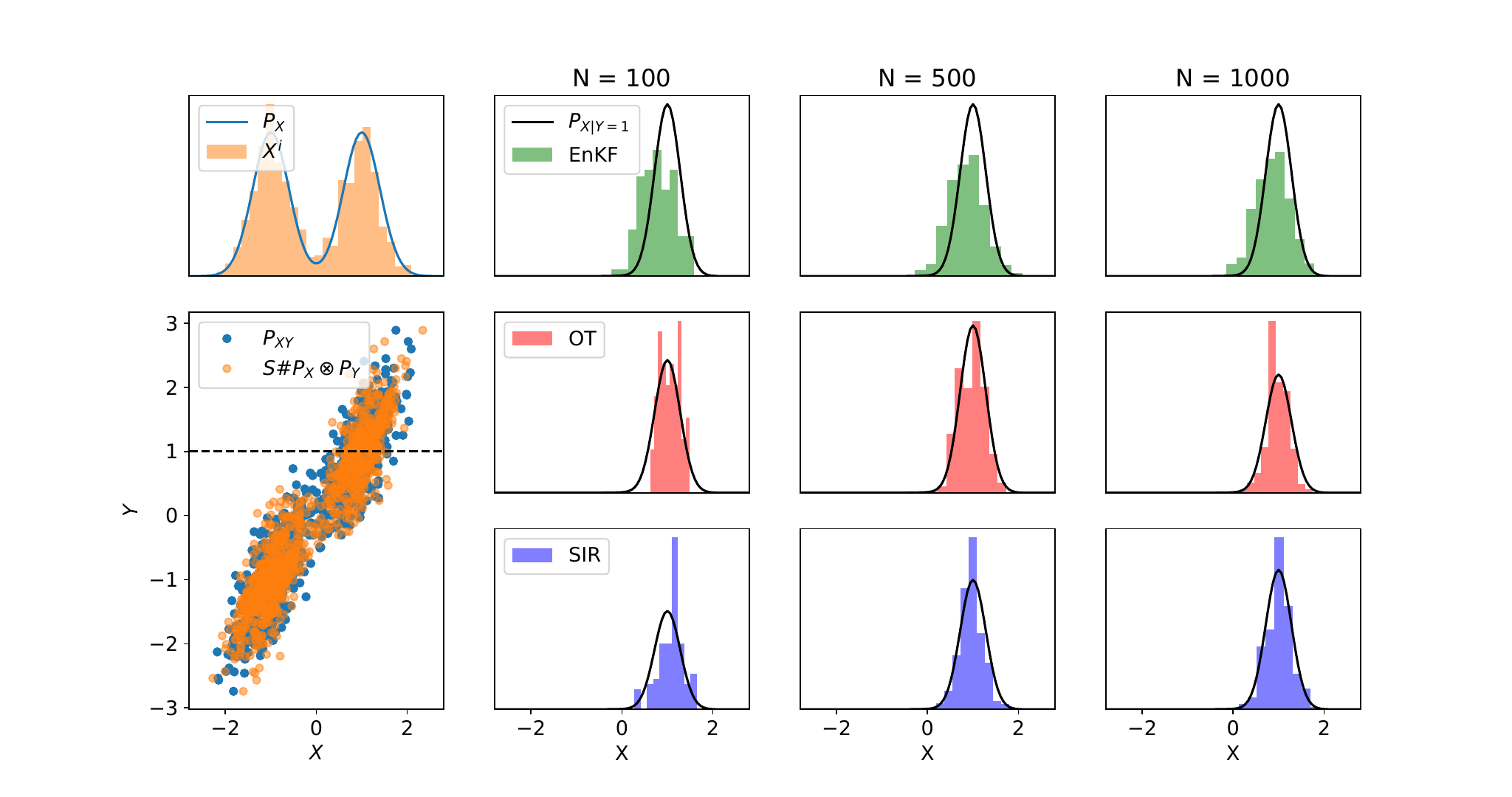}
    \caption{Numerical results for the static example in equation~(\ref{example:bimodal}). Top-left: Samples $\{X^i\}_{i=1}^N$ from the prior $P_X$; bottom-left: samples $\{(X^i,Y^i)\}_{i=1}^N$ from the joint distribution $P_{XY}$ in comparison with the transported samples $\{(T(X^{\sigma_i},Y^i),Y^i)\}_{i=1}^N$; rest of the panels: transported samples for $Y=1$ for different values of $N$ and three different algorithms.}
    \label{fig:bimodal}
\end{figure}

\subsection{Additional dynamic example}\label{sec:app_dynamic}
Consider the  following model:
\begin{subequations}\label{eq:model-examples}
\begin{align}
    X_{t} &= (1-\alpha) X_{t-1} + 2\lambda V_t,\quad X_0 \sim \mathcal{N}(0,I_n),\\
    Y_t &= h(X_t) + \lambda W_t,
\end{align}
\end{subequations}
where $\{V_t,W_t\}_{t=1}^\infty$ are i.i.d sequences of standard Gaussian random variables, $\alpha=0.1$ and $\lambda=\sqrt{0.1}$.  This is the generalization of the example~\ref{eq:model-example} where now, instead of  $h(x)=x\odot x$, the observation function is selected to be linear $h(x)=x$ or cubic $h(x)=x \odot x\odot x$. 

We parameterize $f$ and $T$ as before, except that we included the EnKF layer in $T$ and used two ResNet blocks. The learning rate for the optimizers of $f$ and $T$ are $10^{-3}$ and $2\times 10^{-3}$, respectively.  The total number of iterations is $1024$, which is divided by $2$ after each time step (of the filtering problem) until it reaches $64$. The batch-size is $64$ and the number of particles $N=1000$. 

We quantify the performance of the algorithms by computing the maximum mean discrepancy (MMD) distance between the exact posterior distribution and the approximated distribution formed by the particles. In order to give an accurate approximation of the true posterior, we use the SIR method with $N=10^5$ particles. The MMD distance between two empirical distributions $\mu = \frac{1}{N_1}\sum_{i=1}^{N_1} \delta_{U^i}$ and $\nu = \frac{1}{N_2}\sum_{i=1}^{N_2} \delta_{V^i}$ is 
\begin{equation}\label{eq:mmd}
\text{MMD}_{\text{emp}}(\mu, \nu) = \frac{1}{N_1^2} \sum_{i=1}^{N_1} \sum_{j=1}^{N_1} k(U^i, U^j) + \frac{1}{N_2^2} \sum_{i=1}^{N_2} \sum_{j=1}^{N_2} k(V^i, V^j) - \frac{2}{N_1N_2} \sum_{i=1}^{N_1} \sum_{j=1}^{N_2} k(U^i, V^j)
\end{equation}  
where $k(\cdot, \cdot)$ denotes the RBF kernel. The kernel bandwidth is selected individually for each example.

The numerical results for the linear and cubic observation models are depicted in Fig.~\ref{fig:dynamic_example_x} and~\ref{fig:dynamic_example_xxx}, respectively.
EnKF performs better than the other two algorithms in the linear Gaussian setting, while it does not perform as well with the cubic observation model. It is noted that the OT algorithm's performance can be enhanced through fine-tuning and additional training iterations.

\begin{figure*}[t]
    \centering
\begin{subfigure}{0.35\textwidth}
    \centering
    \includegraphics[width=1.0\textwidth,trim={30 0 70 60},clip]{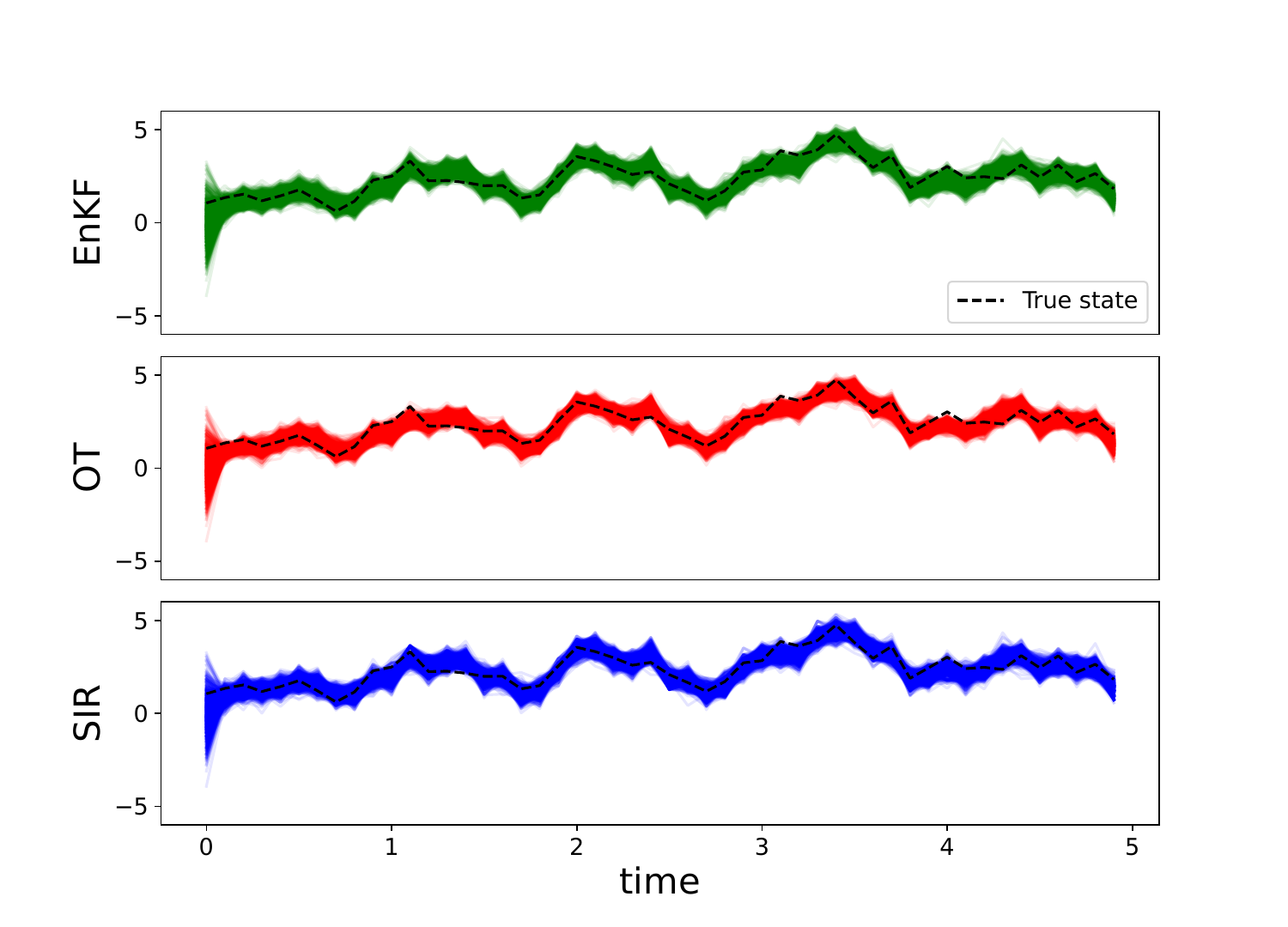}
    \caption{Particles trajectory.}
\end{subfigure}
\hspace{1.0em}
\begin{subfigure}{0.35\textwidth}
    \centering
     \includegraphics[width=1.0\textwidth,trim={30 0 70 60},clip]{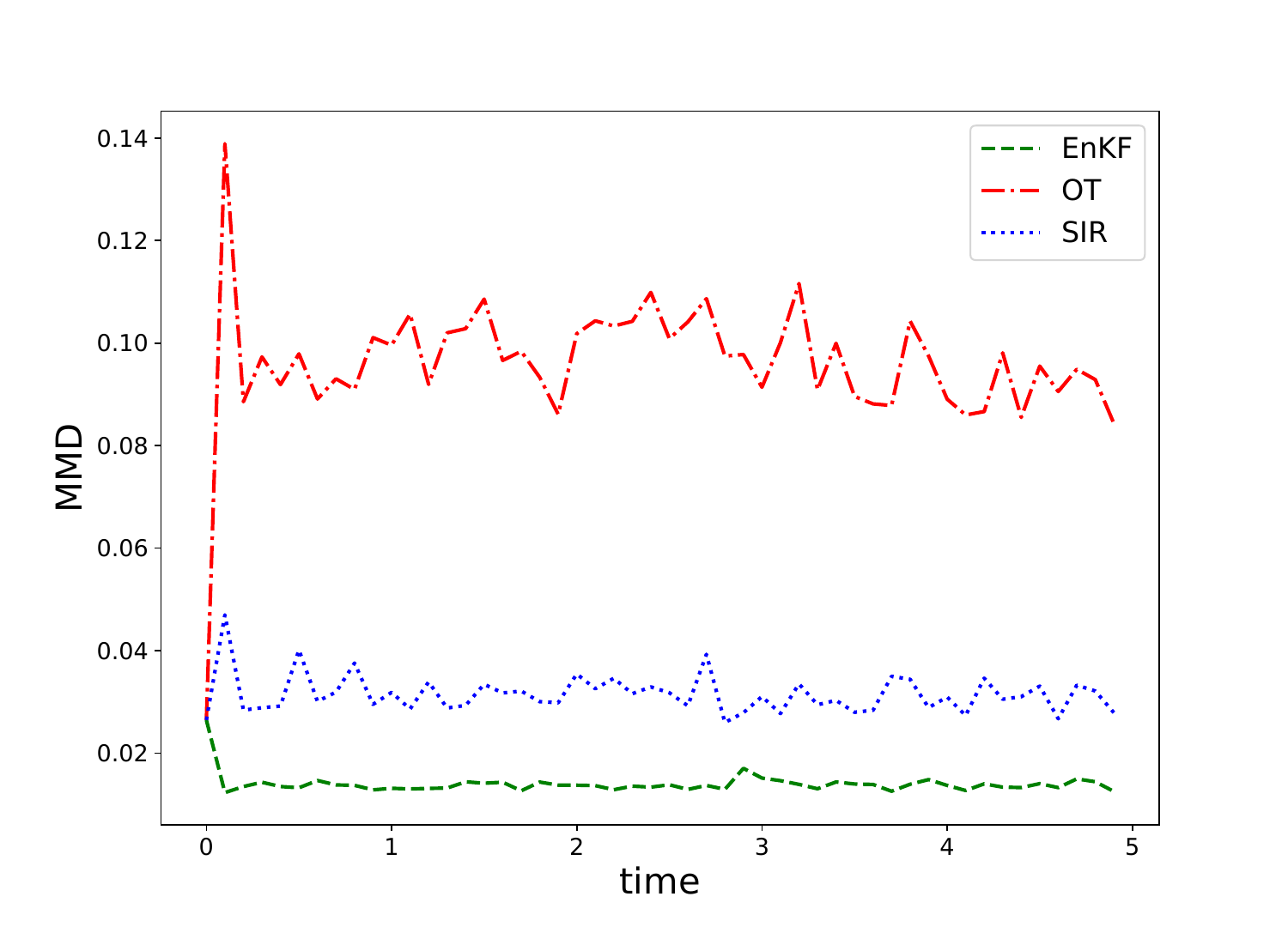} 
    \caption{MMD vs time.}
\end{subfigure}
     \caption{Numerical results for the dynamic example~(\ref{eq:model-examples}) where $h(X_t)=X_t$. The left panel shows the trajectory of the particles $\{X^1_t,\ldots,X^N_t\}$ along with the trajectory of the true state $X_t$ for EnKF, OT, and SIR algorithms, respectively. The second panel shows the MMD distance with respect to the exact conditional distribution.}
    \label{fig:dynamic_example_x}
\end{figure*}

\begin{figure*}[h]
    \centering
\begin{subfigure}{0.35\textwidth}
    \centering
    \includegraphics[width=1.0\textwidth,trim={30 0 70 60},clip]{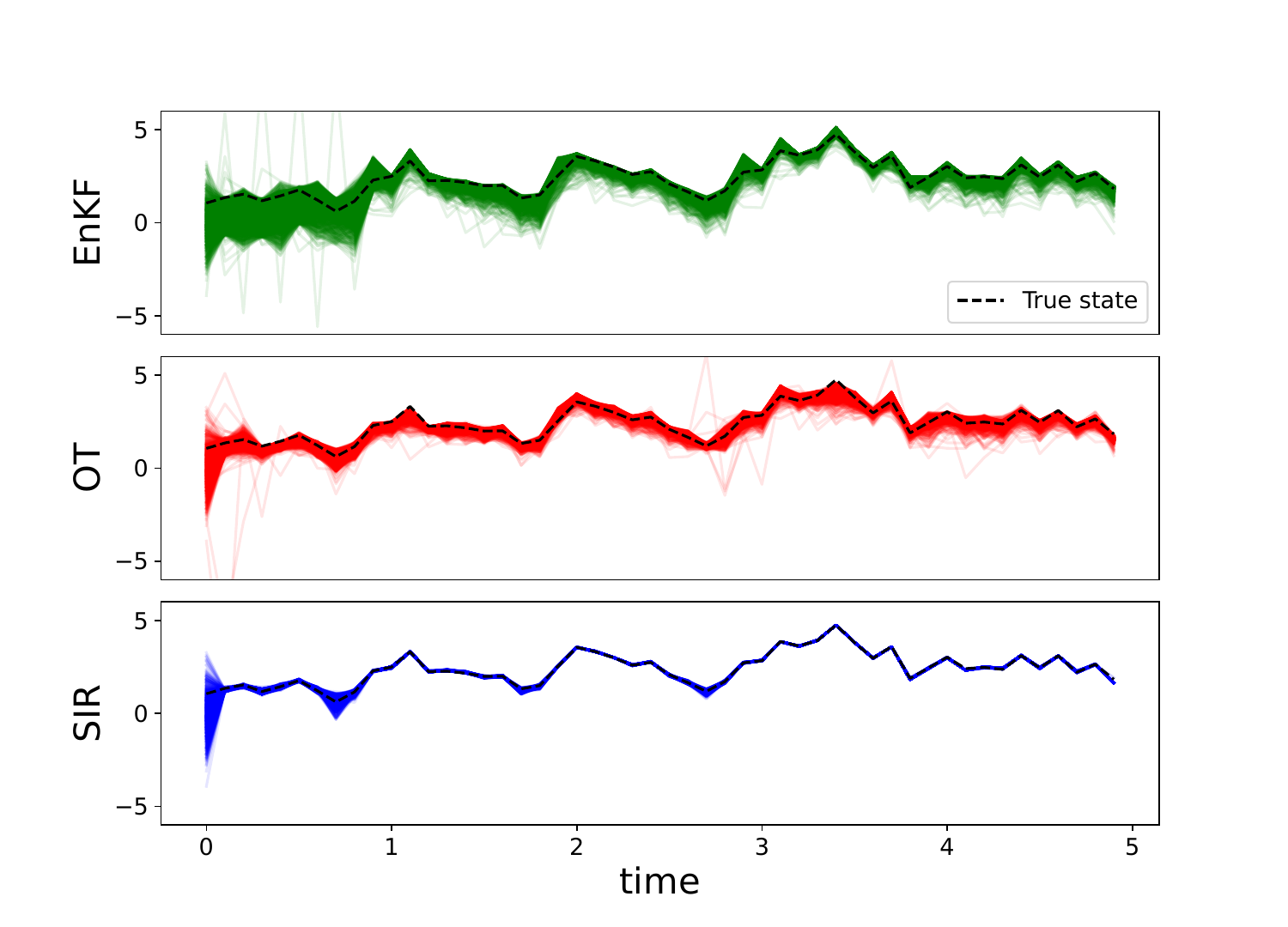}
    \caption{Particles trajectory.}
\end{subfigure}
\hspace{1.0em}
\begin{subfigure}{0.35\textwidth}
    \centering
     \includegraphics[width=1.0\textwidth,trim={30 0 70 60},clip]{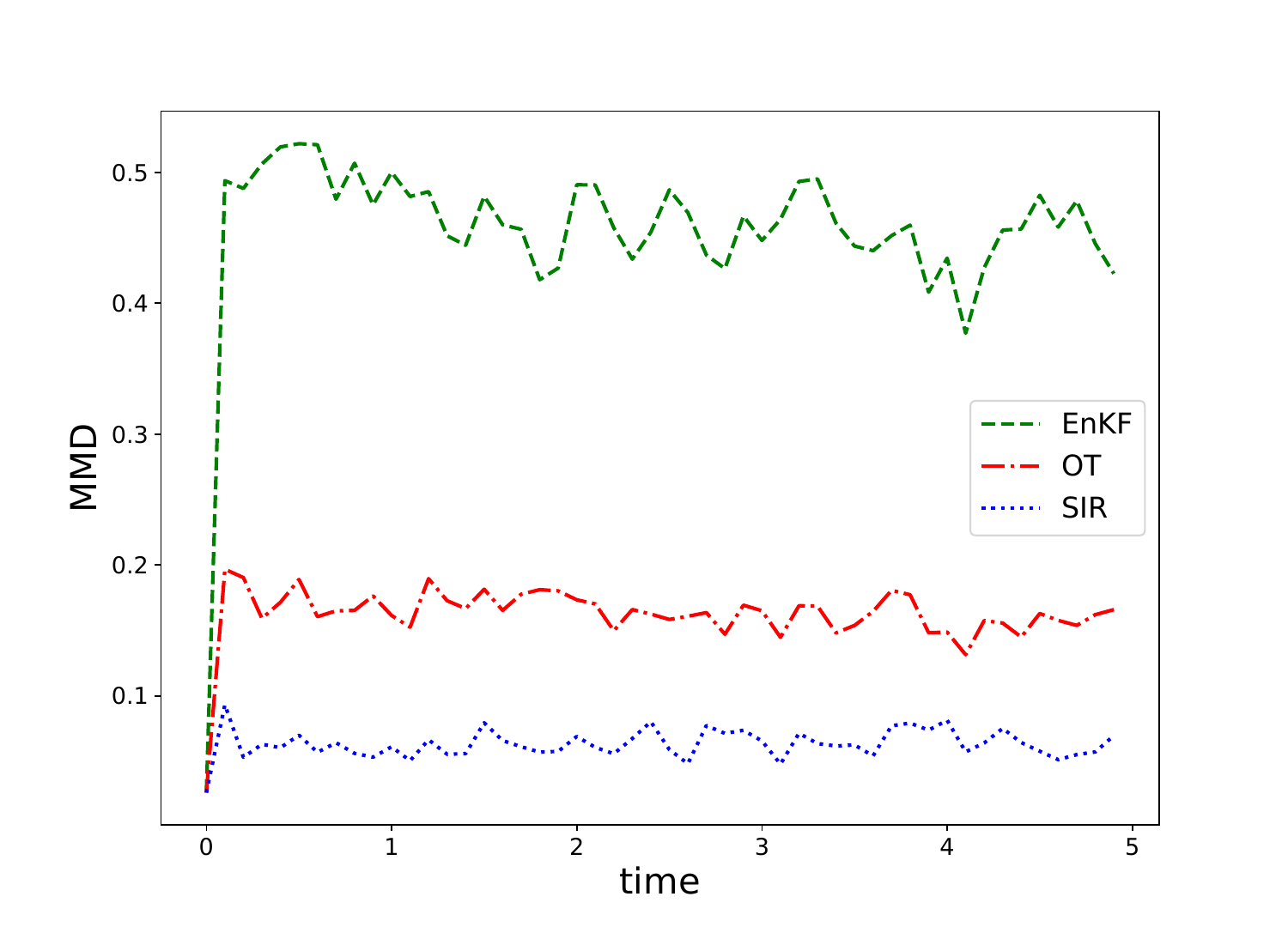} 
    \caption{MMD vs time.}
\end{subfigure}
     \caption{Numerical results for the dynamic example~(\ref{eq:model-examples}) where $h(X_t)=X_t\odot X_t\odot X_t$. The left panel shows the trajectory of the particles $\{X^1_t,\ldots,X^N_t\}$ along with the trajectory of the true state $X_t$ for EnKF, OT, and SIR algorithms, respectively. The second panel shows the MMD distance with respect to the exact conditional distribution.}
    \label{fig:dynamic_example_xxx}
\end{figure*}
\subsection{The details of the Lorenz 63 model and additional results}\label{sec:app_L63}

We consider the following Lorenz 63 model:

\begin{equation}\label{eq:L63}
\begin{split}
\begin{bmatrix}
    \dot{X}(1) \\ \dot{X}(2) \\ \dot{X}(3)
\end{bmatrix}
&= 
\begin{bmatrix}
    \sigma (X(2) - X(1)) \\
    X(1) (\rho - X(3)) - X(2) \\
    X(1)X(2) - \beta X(3)   
\end{bmatrix},\quad X_0 \sim \mathcal{N}(\mu_0,\sigma_0^2I_3),
\\
Y_t &= \begin{bmatrix}
    X_t(1)\\ X_t(3)
\end{bmatrix} 
+ \sigma_{obs}W_t,
\end{split}
\end{equation}

where \([X(1),X(2),X(3)]^\top\) are the variables representing the hidden states of the system, and \(\sigma\), \(\rho\), and \(\beta\) are the model parameters. We choose \(\sigma=10\), \(\rho=28\), \(\beta=8/3\), $\mu_0 = [25,25,25]^\top$, and $\sigma_{0}^2=10$. The observed noise $W$ is a $2$-dimensional standard Gaussian random variable with $\sigma_{obs}^2=10$.

We simulated the Lorenz 63 model without any noise in the dynamics; however, we included an artificial noise $\mathcal{N}(0,\sigma_{\textit{added}}^2 I_3),~\sigma_{\textit{added}}^2=1$ to the dynamic update step of the algorithms. In order to better study the difference between the three filters, we initialized the particles from a Gaussian distribution with mean $[0,0,0]$, while the true state is initialized with a Gaussian with mean $[25,25,25]$.

The neural network for $T$ is a two-block residual network of size $64$ with or without EnKF layer. The learning rate for the optimizers of $f$ and $T$ are $5\times 10^{-2}$ and $10^{-2}$, respectively.  The total number of iterations is $1024$, which is divided by $2$ after each time step (of the filtering problem) until it reaches $64$. Each iteration involves a random selection of a batch of samples of size $64$ from the total of $N=1000$ particles.

In the left panel in Fig.~\ref{fig:states_and_mse_L63}, we present the trajectories of the true states and particles for all three algorithms.  
In the right panel of Fig.~\ref{fig:states_and_mse_L63}, we show the mean-squared-error (MSE) of estimating the state by taking the average over $10$ independent simulations.
It is observed that both OT settings, with and without the EnKF layer, are performing better than EnKF and SIR overall. The EnKF layer helps with the performance at the initial stage; however, it requires careful tuning for the learning rate.

\begin{figure}[t!]
    \centering
    \includegraphics[width=0.7\textwidth,trim={50 20 70 20},clip]{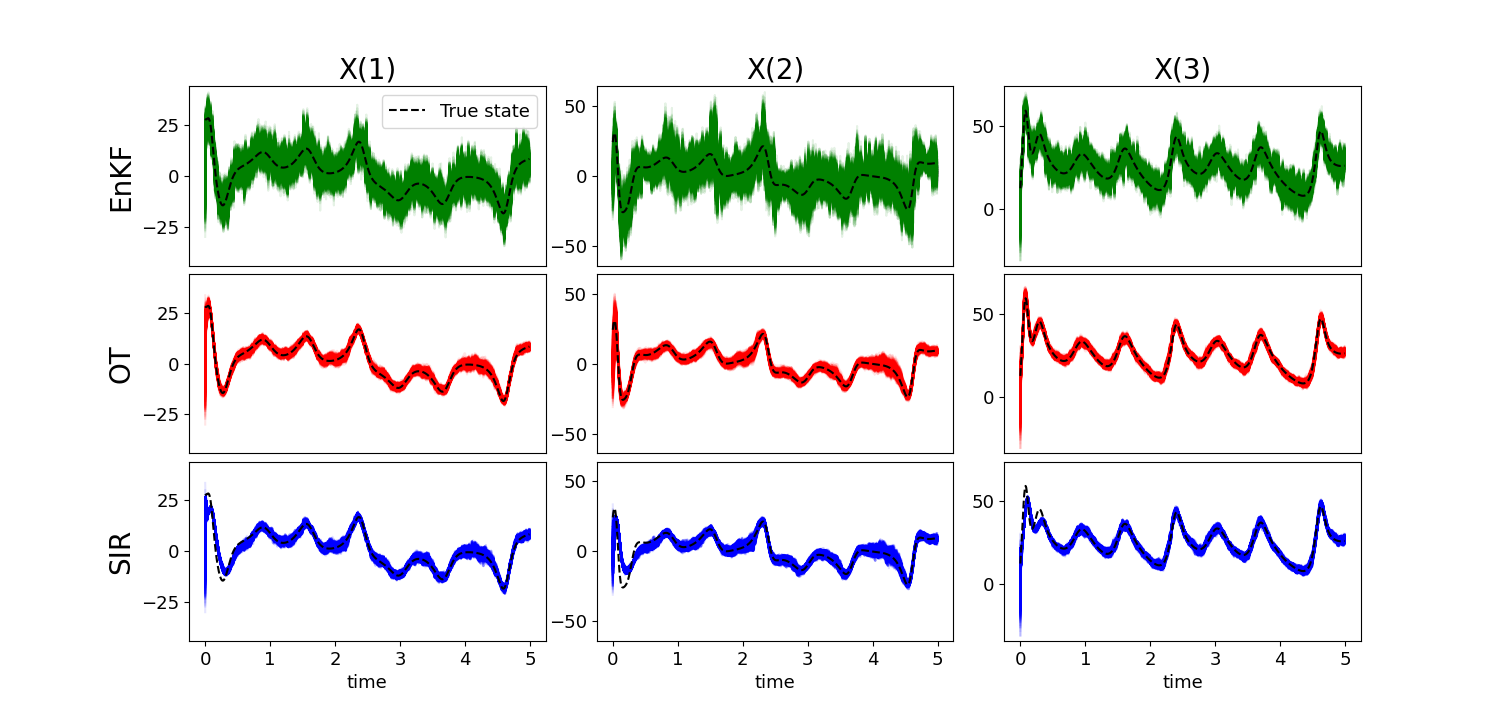}
    \caption{Numerical results for the Lorenz 63 example~(\ref{eq:L63}). The figure shows the trajectory of the particles $\{X^i\}$ along with the trajectory of the true state.Each column represents one state, and each row represents one algorithm. }
    \label{fig:states_and_mse_L63}
\end{figure}

\subsection{The Lorenz 96 model}\label{sec:app_L96}

Consider the following Lorenz-96 model:
\begin{equation}\label{eq:L96}
\begin{split}
\dot{X}(k) &= (X(k+1)-X(k-2))X(k-1)-X(k)+F + \sigma V,\quad \text{for}\quad k=1,\ldots,n \\
Y_t &= \begin{bmatrix}
    1 & 0 & 0 & 0 & 0 & 0 & 0 & 0 & 0 \\
    0 & 1 & 0 & 0 & 0 & 0 & 0 & 0 & 0 \\
    0 & 0 & 0 & 1 & 0 & 0 & 0 & 0 & 0 \\
    0 & 0 & 0 & 0 & 1 & 0 & 0 & 0 & 0 \\
    0 & 0 & 0 & 0 & 0 & 0 & 1 & 0 & 0 \\
    0 & 0 & 0 & 0 & 0 & 0 & 0 & 1 & 0 \\
\end{bmatrix} X_t + \sigma W_t
\end{split}
\end{equation}
for $n=9$ where $
X_0 \sim \mathcal{N}(\mu_0,\sigma_0^2I_n)$ and we choose the convention that $X(-1)=X(n-1)$, $X(0) = X(n)$, and $X(n+1)=X(1)$, and $F=2$ is a forcing constant. We choose the model parameters $\mu_0 = 25\cdot \mathbf{1}_n$, and $\sigma_{0}^2=10^2$. The observed noise $W$ is a $n$-dimensional standard Gaussian random variable with variance equal to $1$. 

For this experiment, we represent $T$ as a two-block residual network of size $32$ with EnKF layer. The learning rate for the optimizers of $f$ and $T$ are $10^{-3}$ and $10^{-1}$, respectively.  The total number of iterations is $1024$, which is divided by $2$ after each time step (of the filtering problem) until it reaches $64$. The batch-size is $128$, and the total of particles $N=1000$.

The numerical results are shown in Fig.~\ref{fig:states_L96}. The columns in panel (a)  represent unobserved states, and the rows represent the results for each algorithm.  
Panel (b) shows the MSE in estimating the state, averaged over $10$ independent simulations. The results indicate that the OT algorithm performs as well as the EnKF algorithm, which appears to perform well for this problem. 

\begin{figure}[ht!]
    \centering
\begin{subfigure}{0.5\textwidth}
    \centering
    \includegraphics[width=1\textwidth,trim={50 20 65 20},clip]{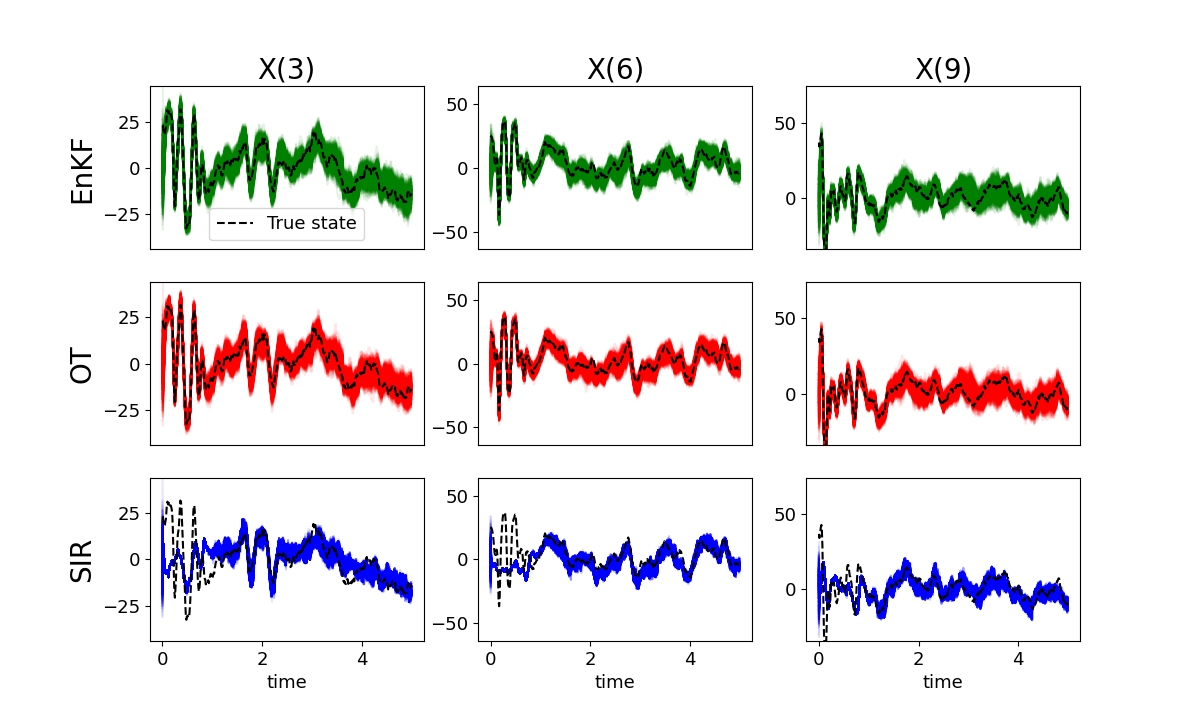}
    \caption{Particles trajectory.}
\end{subfigure}
\hspace{0.1em}
\begin{subfigure}{0.42\textwidth}
    \centering
     \includegraphics[width=1\textwidth,trim={20 20 70 20},clip]{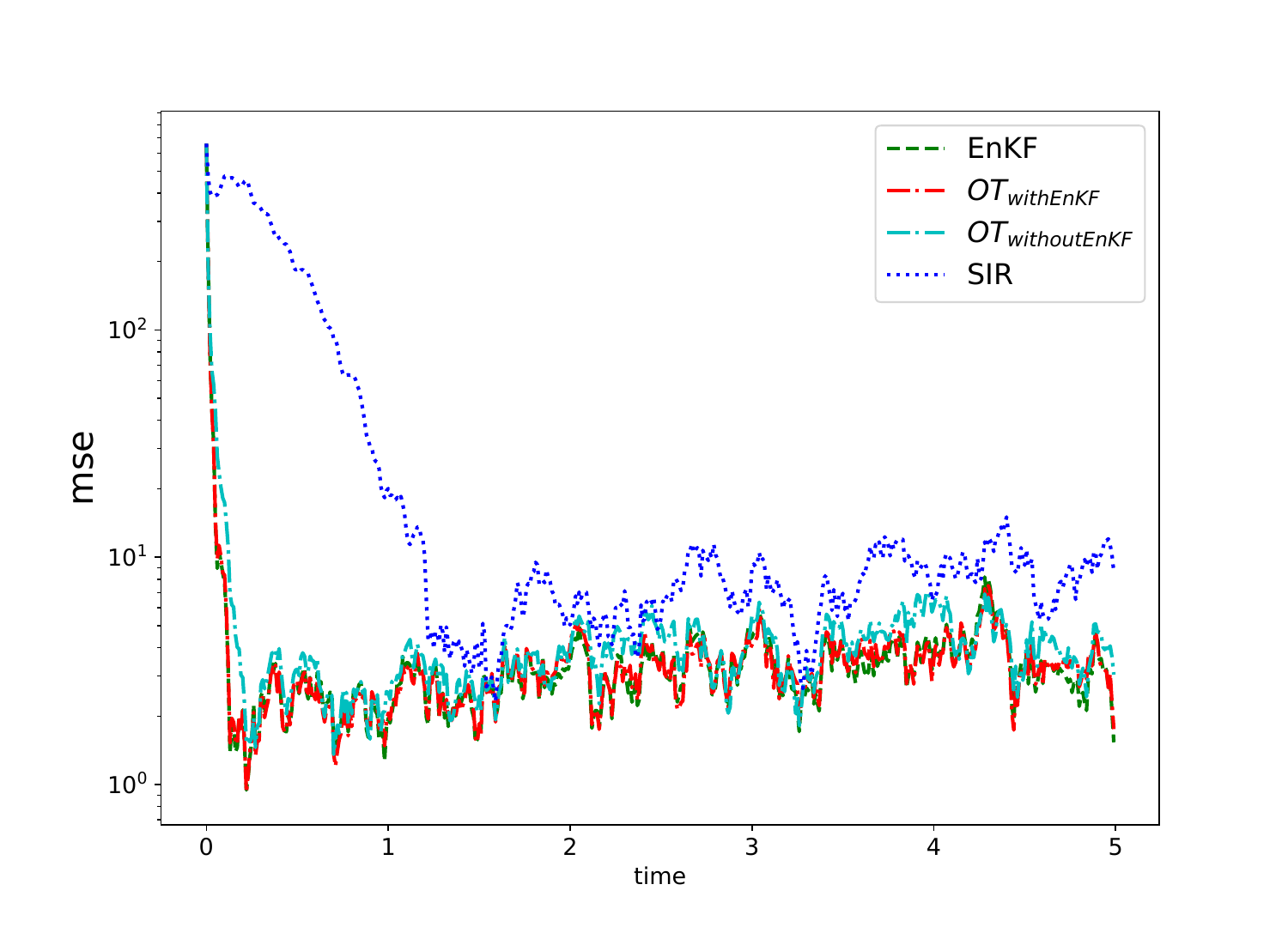}
    \caption{MSE vs time.}
\end{subfigure}

    \caption{Numerical results for the Lorenz 96 example~(\ref{eq:L96}). (a) The trajectory of the particles $\{X^i\}$ along with the trajectory of the true state; each column represents one of the unobserved state, and each row represents one algorithm; (b) Comparison of the MSE in estimating the state as a function of time.}
    \label{fig:states_L96}
\end{figure}

\subsection{Additional details and results for static image in-painting on MNIST}\label{sec:app_static_mnist}

In this example, we consider the problem of computing conditional distributions on the $100$-dimensional latent space of generative adversarial network (GAN) trained to represent the MNIST digits~\cite{goodfellow2014generative}. In particular, denoting the generator by $G : \mathbb{R}^{100}\rightarrow \mathbb{R}^{28\times 28}$, we consider the model:
\begin{align*}
    Y_t &= h(G(X),c_t) + \sigma W_t,\quad X\sim N(0,I_{100}),
\end{align*}
where the observation function $(z,c) \in \mathbb{R}^{28\times 28} \times \mathbb R^2 \mapsto h(z,c)\in \mathbb{R}^{r\times r}$ is defined as the $r\times r$ window of pixels $z[c(1):c(1)+r,c(2):c(2)+r]$. The coordinates of the corner $c_t$ move from left to right and top to bottom scanning a partial portion of the image called the {\it observed part}. In order to make the problem more challenging, we fed a noisy measurement of the corner location to the algorithms by adding a uniform random integer between $-2$ and $2$ for each axis. The observational noise set to be $\sigma=10^{-1}$, with every component of $W_t\sim N(0,1)$. While the true image does not change, we included artificial Gaussian noise $N(0,\sigma^2 I_{100})$ to the particles to avoid particle collapse.

For this example, we used the publicly available codes for a GAN model on MNIST~\footnote{\url{https://github.com/lyeoni/pytorch-mnist-GAN/tree/master}} and a classifier achieving an accuracy exceeding $98\%$ on the MNIST test dataset~\footnote{\url{https://nextjournal.com/gkoehler/pytorch-mnist}}.

The neural network for $T$ is a residual network with size $320$ without the EnKF layer. The learning rate for both $f$ and $T$ is $10^{-3}$ with $2^{12}=4096$ iterations. Batch-size is $64$ and total number of particles $N=1000$.

Additional numerical results for this example are presented in Fig.~\ref{fig:mnist_static_EnKF_and_SIR_particles_example2}. The top row shows the total observed part of the image up to that time step, and the following 16 rows show the images generated from the particles that approximate the conditional distribution. 
\begin{figure}[h]
    \centering
    \begin{subfigure}{0.485\textwidth}
    \centering
    \includegraphics[width=1\textwidth,trim={135 50 105 45},clip]{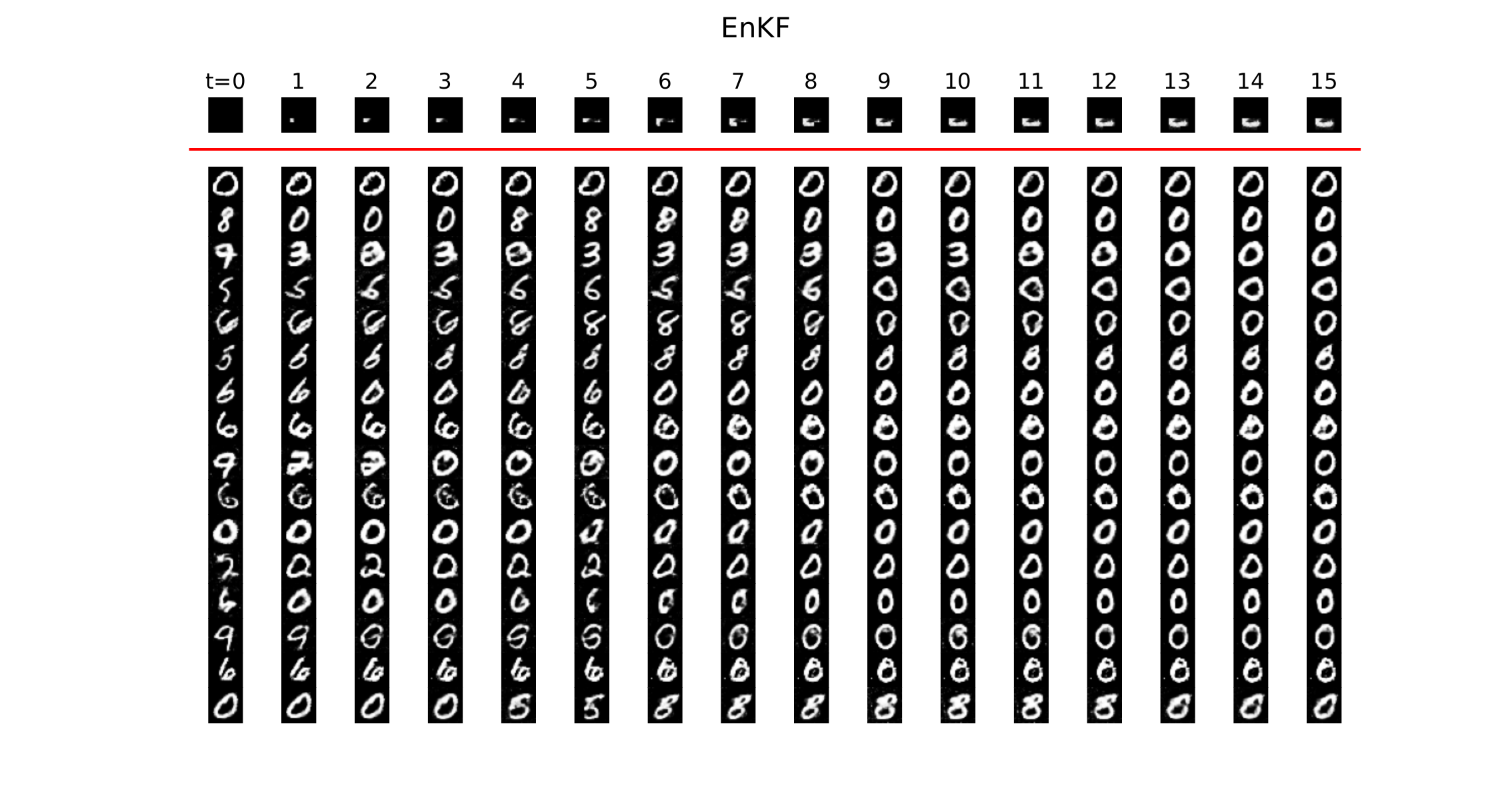}
    \caption{EnKF particles.}
\end{subfigure}
\hspace{0.8em}
\begin{subfigure}{0.485\textwidth}
    \centering
     \includegraphics[width=1\textwidth,trim={135 50 105 45},clip]{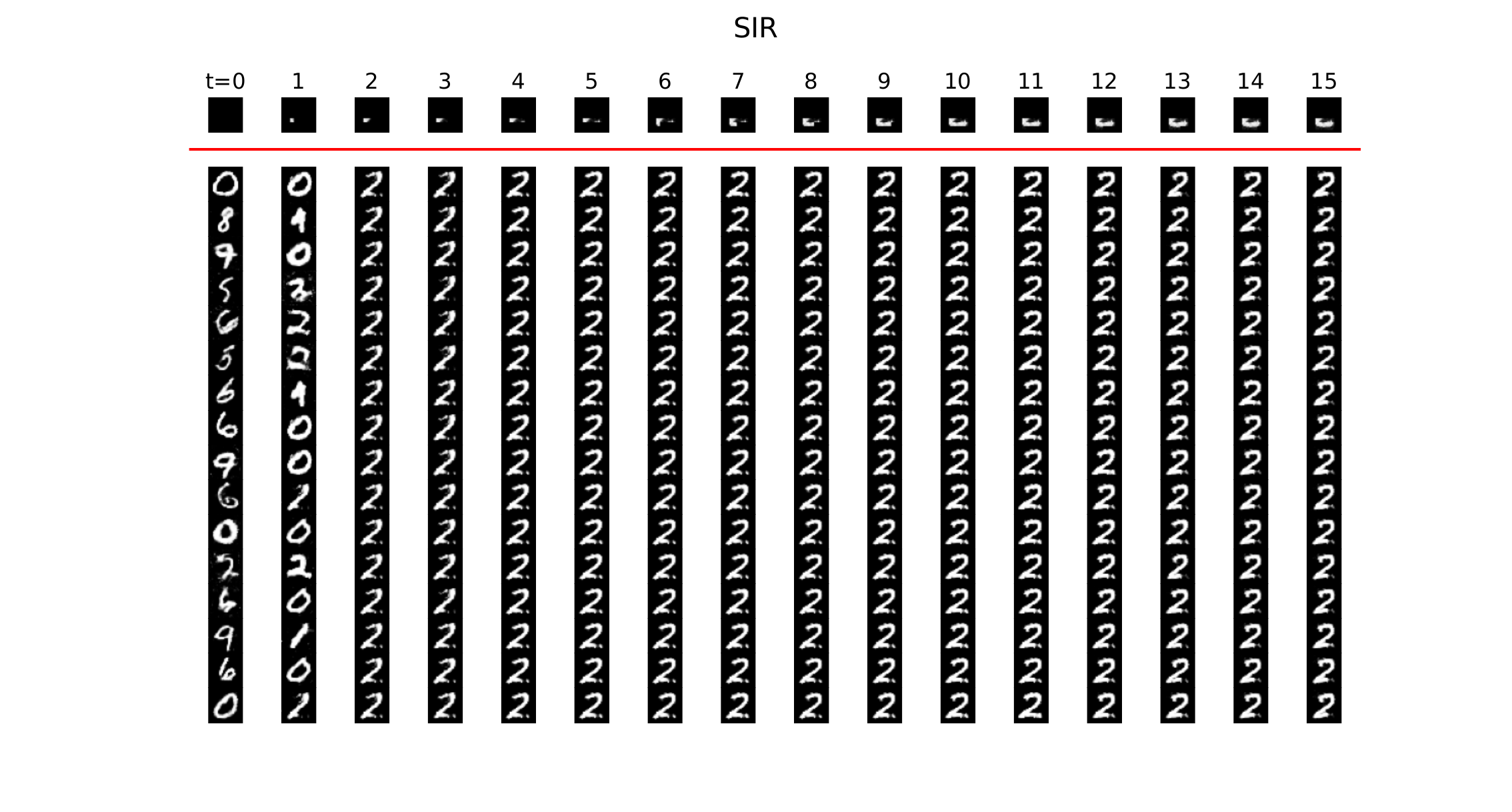}
    \caption{SIR particles.}
\end{subfigure}

    \caption{Additional numerical results for the static image in-painting on MNIST. The first row shows the cumulative total observations up to each time step. The subsequent rows under the red line show 16 particles from the EnKF and SIR methods, respectively.}
    \label{fig:mnist_static_EnKF_and_SIR_particles_example2}
\end{figure}

\subsection{Additional details and results for dynamic image in-painting on MNIST}\label{sec:app_dynamic_mnist}
We extended the previous static example by  adding a dynamic update for the latent variable $X$ as follows:
\begin{align*}
    X_{t+1} &= (1-\alpha)X_{t} + \sigma V_t,\\
     Y_{t+1} &= h(G(X_{t+1}),c_{t+1})+ \sigma_w W_{t+1},
\end{align*}
where $V_t, W_t$ are standard Gaussian with the appropriate dimension, $\sigma = \sqrt{2\alpha-\alpha^2}$, $\sigma_w=10^{-1}$, and $\alpha=0.2$. The observation function $h_t(G(X_t),c_t) = G(X_t)[28-r:28,c_t:c_t+r]$, $c_t\sim \textit{Unif}_\textit{integer}(1,28-r)$, and $r=12$.
The OT algorithm parameters are similar to the static case.

 We present the trajectory of the particles in Fig.~\ref{fig:mnist_dynamic_dist_particles_example2} (a),(b), and (c). The top row shows the true image, and the second row shows the observation at each time step. The top row shows the total observed part of the image up to that timestep, and the following 16 rows show the images generated from the particles that approximate the conditional distribution. In Fig.~\ref{fig:mnist_dynamic_dist_particles_example2}(d), we used an accurate MNIST classifier to represent the histogram of the digits generated from the particles of each algorithm.

\begin{figure}[h]
    \centering
    \begin{subfigure}{0.485\textwidth}
    \centering
    \includegraphics[width=1\textwidth,trim={105 50 90 50},clip]{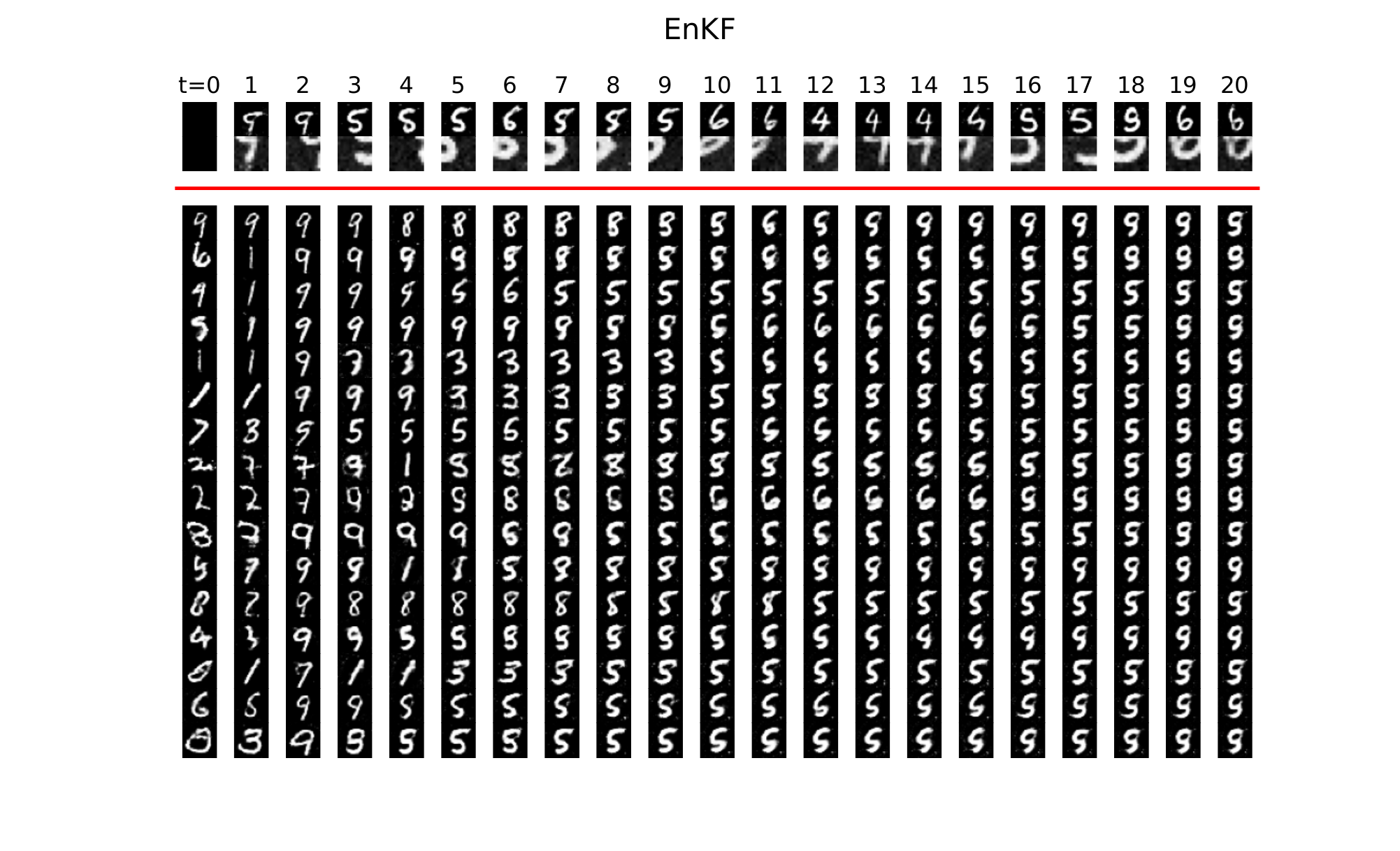}
    \caption{EnKF particles.}
    \vspace*{2mm}
\end{subfigure}
\hspace{0.8em}
\begin{subfigure}{0.485\textwidth}
    \centering
    \includegraphics[width=1\textwidth,trim={105 50 90 50},clip]{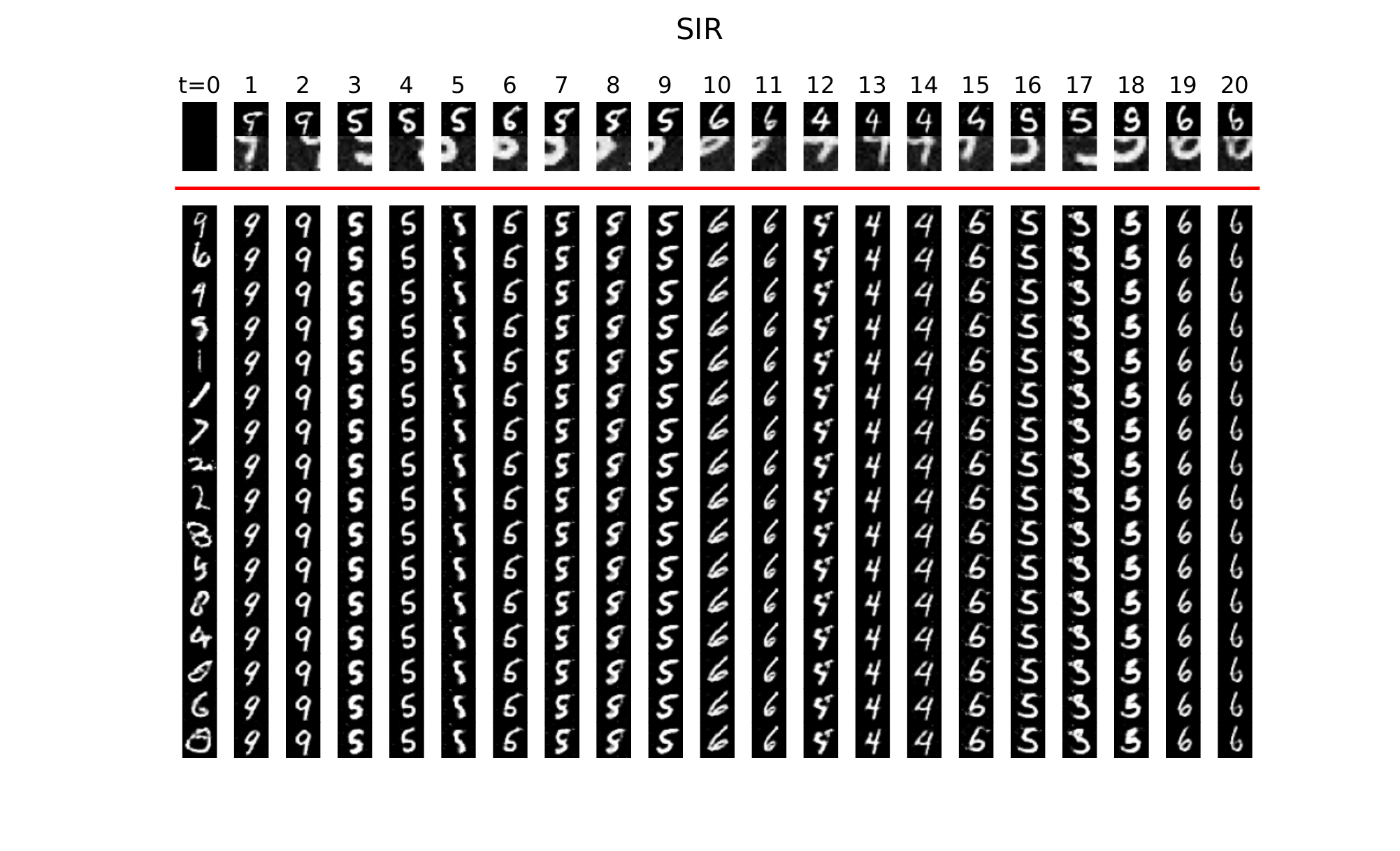}
    \caption{SIR particles.}
    \vspace*{2mm}
\end{subfigure}
\vspace{1em}
\begin{subfigure}{0.485\textwidth}
    \centering
    \vspace*{2mm}
    \includegraphics[width=1\textwidth,trim={105 60 90 50},clip]{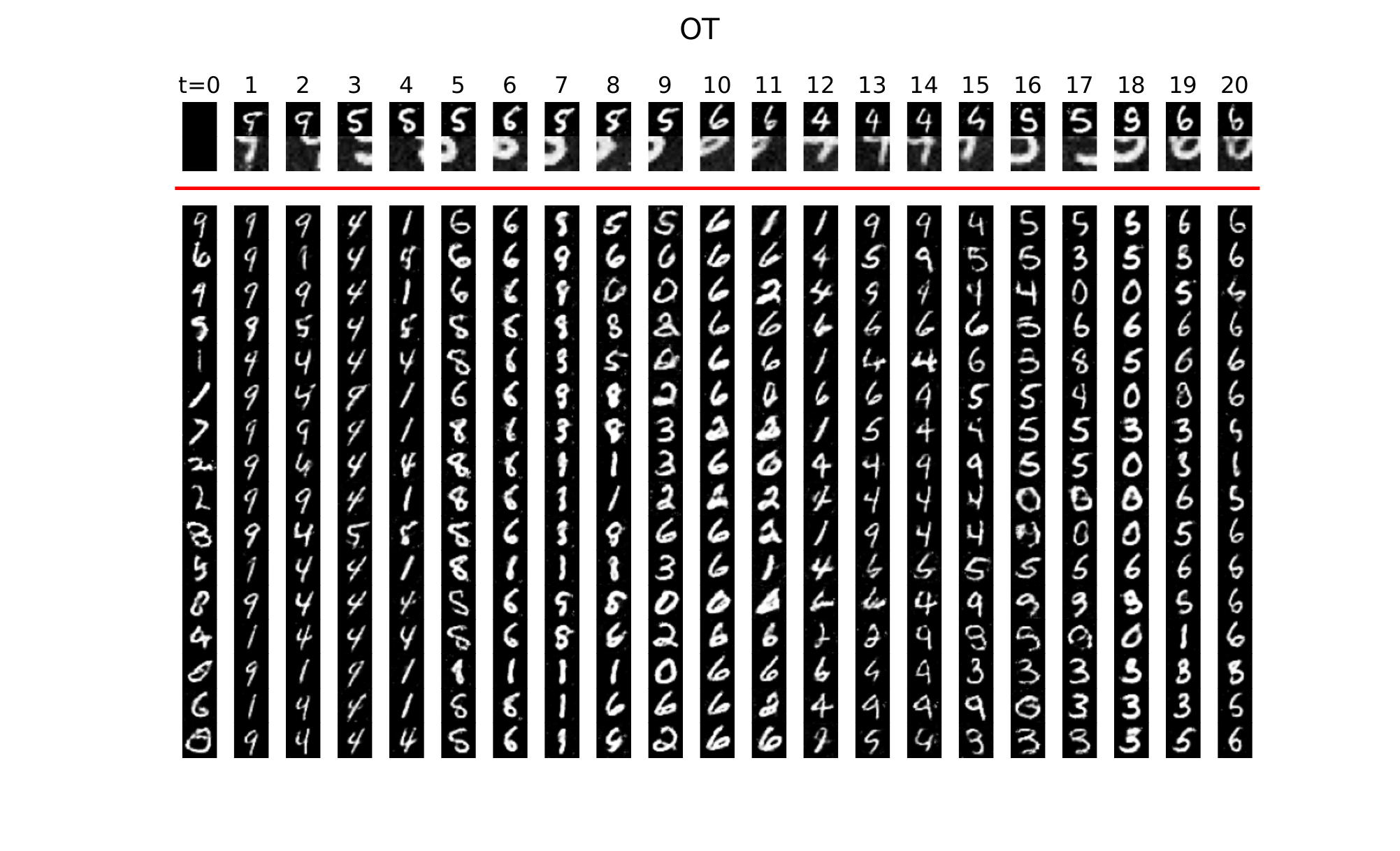}
    \caption{OT particles.}
\end{subfigure}
\hspace{0.8em}
\begin{subfigure}{0.485\textwidth}
    \centering
    \includegraphics[width=1\textwidth,trim={95 20 90 50},clip]{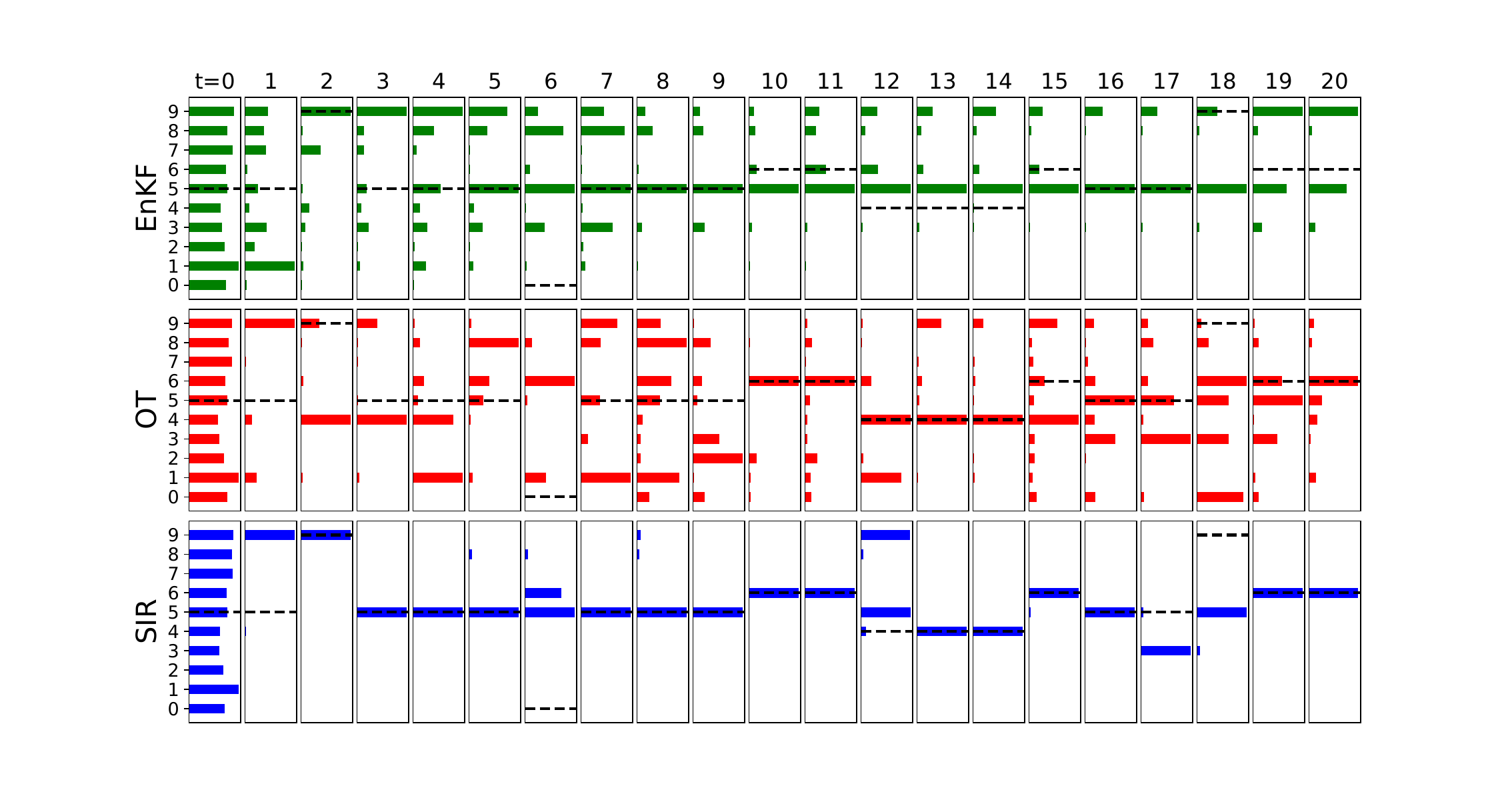}
    \caption{Particles distribution.}
\end{subfigure}
    \caption{Additional numerical results for the dynamic image in-painting on MNIST. (a-b-c) The first row shows the true image, and the second row shows the observation at each time step. The subsequent rows under the red line show 16 particles from the EnKF, SIR, and OT methods, respectively. (d) The histogram of the digits generated by the particles from the three algorithms as a function of time, evaluated using an accurate  MNIST classifier.}
    \label{fig:mnist_dynamic_dist_particles_example2}
\end{figure}

\end{document}